\documentclass{article}

\usepackage[final, nonatbib]{neurips_2019}

\usepackage[numbers]{natbib}

\usepackage[utf8]{inputenc} 
\usepackage[T1]{fontenc}    
\usepackage{hyperref}       
\usepackage{url}            
\usepackage{booktabs}       
\usepackage{amsfonts}       
\usepackage{nicefrac}       
\usepackage{microtype}      

\usepackage{amssymb}
\usepackage{nccmath}
\usepackage{bbm}
\usepackage{enumerate}
\usepackage{thmtools}
\usepackage{thm-restate}
\usepackage{graphicx}
\usepackage{color}
\usepackage{pbox}
\usepackage{standalone}
\usepackage{tikz}
\usetikzlibrary{positioning,shadows,arrows,trees,shapes,fit,decorations.pathreplacing,calc}
\usepackage{subcaption}
\usepackage{placeins}

\usepackage{algorithm}
\usepackage{algorithmic}

\allowdisplaybreaks

\declaretheorem[name=Theorem,refname={Theorem,Theorems},Refname={Theorem,Theorems}]{theorem}
\declaretheorem[name=Lemma,refname={Lemma,Lemmas},Refname={Lemma,Lemmas},sibling=theorem]{lemma}
\declaretheorem[name=Proposition,refname={Proposition,Propositions},Refname={Proposition,Propositions},sibling=theorem]{proposition}
\declaretheorem[name=Corollary,refname={Corollary,Corollaries},Refname={Corollary,Corollaries},sibling=theorem]{corollary}

\DeclareMathOperator*{\argmax}{arg max}

\newenvironment{proof}{\emph{Proof:}}{\hfill$\square$}

\newcommand{\one}{\mathbb{I}}
 
\newcommand{\Reg}{\mathfrak{R}} 
\newcommand{\PP}{\mathbb{P}}

\newcommand{\E}{\mathbb{E}}

\newcommand{\Var}{\mathbb{V}}
\newcommand{\cA}{\mathcal{A}}

\newcommand{\cF}{\mathcal{F}}

\newcommand{\cZ}{\mathcal{Z}}

\newcommand{\cS}{\mathcal{S}}
\newcommand{\cI}{\mathcal{I}}
\newcommand{\cN}{\mathcal{N}}
\newcommand{\cT}{\mathcal{T}}
\newcommand{\cGP}{\mathcal{GP}}

\newcommand{\bk}{\mathbf{k}}
\newcommand{\bK}{\mathbf{K}}
\newcommand{\bI}{\mathbf{I}}
\newcommand{\bA}{\mathbf{A}}
\newcommand{\bB}{\mathbf{B}}

\newcommand{\bZ}{\mathbf{Z}}
\newcommand{\bz}{\mathbf{z}}
\newcommand{\bN}{\mathbb{N}}

\makeatletter
\newcommand\nocaption{%
    \renewcommand\p@subfigure{}
    \renewcommand\thesubfigure{\thefigure\alph{subfigure}}
}
\makeatother

\title{Recovering Bandits}

\author{%
  Ciara Pike-Burke\thanks{This work was carried out while CPB was a PhD student at STOR-i, Lancaster University, UK.} \\
  Universitat Pompeu Fabra\\
  Barcelona,  Spain \\
  \texttt{c.pikeburke@gmail.com} \\
   \And
  Steffen Gr\"unew\"alder \\
  Lancaster University \\
  Lancaster, UK \\
  \texttt{s.grunewalder@lancaster.ac.uk} \\
}

\begin{document}

\maketitle

\begin{abstract}

We study the recovering bandits problem, a variant of the stochastic multi-armed bandit problem where the expected reward of each arm varies according to some unknown function of the time since the arm was last played. While being a natural extension of the classical bandit problem that arises in many real-world settings, this variation is accompanied by significant difficulties. In particular, methods need to plan ahead and estimate many more quantities than in the classical bandit setting. In this work, we explore the use of Gaussian processes to tackle the estimation and planing problem. 
We also discuss different regret definitions that let us quantify the performance of the methods. 
To improve computational efficiency of the methods, we provide an optimistic planning approximation.
We complement these discussions with regret bounds and empirical studies.

\end{abstract}

\section{Introduction}
The multi-armed bandit problem \citep{auer2002finite,thompson1933likelihood} is a sequential decision making problem, 
where, in each round $t$, we play an arm $J_t$ and receive a reward $Y_{t,J_t}$ generated from the unknown reward distribution of the arm. The aim is to maximize the total reward over $T$ rounds. Bandit algorithms have become ubiquitous in many settings such as web advertising and product recommendation. Consider, for example, suggesting items to a user on an internet shopping platform. This is modeled as a bandit problem where each product (or group of products) is an arm. 
Over time, the bandit algorithm will learn to suggest only good products.
In particular, once the algorithm learns that a product (eg. a television) has high reward, it will continue to suggest it.
However, if the user buys a television, the benefit of continuing to show it immediately diminishes, but may increase again as the television reaches the end of its lifetime. To improve customer experience (and profit), it would be beneficial for the recommendation algorithm to learn not to recommend the same product again immediately, but to wait an appropriate amount of time until the reward from that product has `recovered'. 
Similarly, in film and TV recommendation, a user may wish to wait before re-watching their favorite film, or conversely, may want to continue watching a series but will lose interest in it if they haven't seen it recently. It would be advantageous for the recommendation algorithm to learn the different reward dynamics and suggest content based on the time since it was last seen. The recovering bandits framework presented here extends the stochastic bandit problem to capture these phenomena.

In the recovering bandits problem, the expected reward of each arm is given by an unknown function of the number of rounds since it was last played. 
In particular, for each arm $j$, there is a function $f_j(z)$ that specifies the expected reward from playing arm $j$ when it has not been played for $z$ rounds. We take a Bayesian approach and assume 
that the $f_j$'s are sampled from a Gaussian process (GP) (see Figure~\ref{fig:recfuncts}). 
Using GPs allows us to capture a wide variety of functions and deal appropriately with uncertainty.
For any round $t$, let $Z_{j,t}$ be the number of rounds since arm $j$ was last played. This changes for both the played arm (it resets to 0) and also for the unplayed arms (it increases by 1) in every round.
Thus, the expected reward of every arm changes in every round, and this change depends on whether it was played.
This problem is therefore related to both restless and rested bandits \citep{whittle1988restless}.

In recovering bandits,
the reward of each arm depends on the entire sequence of past actions. This means that, even if the recovery functions were known, selecting the best sequence of $T$ arms is intractable (since, in particular, an MDP representation would be unacceptably large). 
One alternative is to select the action that maximizes the \emph{instantaneous} reward, 
without considering future decisions. This is still quite a challenge compared to the $K$-armed bandit problem, as instead of just learning the reward of each arm, we must learn recovery functions. In some cases, maximizing the instantaneous reward is not optimal. 
In particular, using knowledge of the reward dynamics, it is often possible to find a sequence of arms whose total reward is greater than that gained by playing the instantaneous greedy arms.
Thus, 
our interest lies in selecting sequences of arms to maximize the reward.

In this work, we present and analyze an Upper Confidence Bound (UCB) \citep{auer2002finite} and Thompson Sampling \citep{thompson1933likelihood} algorithm for recovering bandits. By exploiting properties of Gaussian processes, both of these accurately estimate the recovery functions and uncertainty, and use these to look ahead and select sequences of actions. This leads to strong theoretical and empirical performance. 
The paper proceeds as follows. 
In  Section~\ref{sec:lit} we discuss related work. We formally define our problem in Section~\ref{sec:def} and the regret in Section~\ref{sec:reg}. 
In Section~\ref{sec:rgpucb}, we present our algorithms
and bound their regret. We use  optimistic planning in Section~\ref{sec:opplanning} to improve computational complexity
and show empirical results in Section~\ref{sec:exp} before concluding. 

\section{Related Work} \label{sec:lit}
In the restless bandits problem, the reward distribution of any arm changes at any time, regardless of whether it is played. This problem has been studied by  \citep{whittle1988restless,slivkins2008adapting,garivier2008upper,raj2017taming,besbes2014stochastic} and others.
In the rested bandits problem, the reward distribution of an arm only changes when it is played. 
\citep{levine2017rotting,cortes2017discrepancy,bouneffouf2016multi,heidari2016tight,seznec2018rotting} study rested bandits problems with rewards that vary mainly with the number of plays of an arm. 

In recovering bandits, the rewards depend on the time since the arm was last played.
\citep{immorlica2018recharging} consider
 concave and increasing recovery functions and \citep{yi2017scalable} study recommendation algorithms with known step recovery functions. 
The closest work to ours is \citep{mintz2017nonstationary} where the expected reward of each arm depends on a state (which could be the time since the arm was played) via a parametric function. 
They use maximum likelihood estimation (although there are no guarantees of convergence) in a KL-UCB algorithm \citep{cappe2013kullback}. 
The expected frequentist regret of their algorithm is $O(\sum_j \nicefrac{\log(T)}{\delta_j^2})$ where $\delta_j$ depends on the random number of plays of arm $j$ and the minimum difference in the rewards of any arms at any time (which can be very small).
By the standard worst case analysis, the frequentist problem independent regret is $O^*(T^{2/3} K^{1/3})$, where we use the notation $O^*$ to suppress log factors. 
Our algorithms achieve $O^*(\sqrt{KT})$ Bayesian regret and require less knowledge of the recovery functions. 
\citep{mintz2017nonstationary} also provide an algorithm 
with no theoretical guarantees but improved experimental performance. In Section~\ref{sec:exp}, we show that our algorithms outperform this algorithm experimentally.

In Gaussian process bandits, there is a function, $f$, sampled from a GP and the aim is to minimize the (Bayesian) regret with respect to the maximum of $f$. 
The GP-UCB algorithm of \citep{srinivas2009gaussian} has Bayesian regret $O^*(\sqrt{T\gamma_T})$ where $\gamma_T$ is the maximal information gain (see Section~\ref{sec:rgpucb}).
By \citep{russo2014learning}, Thompson sampling has the same Bayesian regret. 
\citep{bogunovic2016time} consider GP bandits with a slowly drifting reward function 
and \citep{krause2011contextual} study contextual GP bandits.
These contexts and drifts do not depend on previous actions. 

It is important to note that all of the above approaches only look at instantaneous regret whereas in recovering bandits, it is more appropriate to consider lookahead regret (see Section~\ref{sec:reg}). We will also consider Bayesian regret.
Many naive approaches will not perform well in this problem. For example, treating each $(j,z)$ combination as an arm and using UCB \citep{auer2002finite} with $K|\cZ|$ arms leads to regret $O^*(\sqrt{KT|\cZ|})$ (see Appendix~\ref{app:nonparam}). 
Our algorithms exhibit only $\sqrt{\log|\cZ|}$ dependence on $|\cZ|$. Adversarial bandit algorithms will not do well in this setting either since they aim to minimize the regret with respect to the best constant arm, which is clearly suboptimal in recovering bandits.

\section{Problem Definition} \label{sec:def}
We have $K$ independent arms and play for $T$ rounds ($T$ is not known). For each arm $1 \leq j \leq K$ and round $1 \leq t \leq T$, denote by $Z_{j,t}$ the number of rounds since arm $j$ was last played, where $Z_{j,t} \in \mathcal{Z} =\{0, \dots, z_{\max}\}$ for finite $z_{\max} \in \mathbb{N}$. 
 If we play arm $J_t$ at time $t$ then, at time $t+1$,
\begin{align}
	Z_{j,t+1} = \begin{cases} 0 & \text{ if } J_t=j, \\
					\min\{z_{\max}, Z_{j,t}+1\} & \text{ if } J_t \neq j. \end{cases}  \label{eqn:zdef}
\end{align}
Hence, if arm $j$ has not been played for more than $z_{\max}$ steps, $Z_{j,t}$ will stay at $z_{\max}$. The $Z_{j,t}$ are random variables since they depend on our past actions. We will assume that $T \geq K|\cZ|$.

The expected reward for arm $j$ is modeled by an (unknown) recovery function, $f_j$. We assume that the $f_j$'s are sampled independently from a Gaussian processes with mean 0 and known kernel (see Figure~\ref{fig:recfuncts}).
Let ${\bf Z}_t = (Z_{1,t}, \dots, Z_{K,t})$ be the vector of covariates for each arm at time $t$. At round $t$, we observe ${\bf Z}_t$ and use this and past observations to select an arm $J_t$ to play. We then receive a noisy observation $Y_{J_t,t} = f_{J_t}(Z_{J_t,t}) + \epsilon_t$
where $\epsilon_t$ are i.i.d. $\cN(0,\sigma^2)$ random variables and $\sigma$ is known.

\citep{rasmussen2006gaussian} give an introduction to Gaussian Processes (GP). A Gaussian process gives a distribution over functions, when for every finite set $z_1, \dots, z_N$ of covariates, the joint distribution of $f(z_1), \dots, f(z_N)$ is Gaussian.
 A GP is defined by its mean function, $\mu(z) = \E[f(z)]$, and kernel function, $k(z,z') = \E[(f(z)-\mu(z))(f(z')-\mu(z'))]$.
If we place a GP prior on $f$ and observe ${\bf Y}_N = (Y_1, \dots, Y_N)^T$ at ${\bf z}_N = (z_1,\dots, z_N)^T$ where $Y_n = f(z_n) + \epsilon_n$ for $\epsilon_n$ iid $\cN(0,\sigma^2)$ noise, then the posterior distribution after $N$ observations is
 $\cGP(\mu(z; N), k(z,z'; N))$. Here for
${\bf k}_N(z) = (k(z_1,z), \dots, k(z_N,z))^T$ and positive semi-definite kernel matrix ${\bf K}_N = [k(z_i,z_j)]_{i,j=1}^N$, the posterior mean and covariance are,
\begin{align*}
 \mu(z;N) &= {\bf k}_N(z)^T ({\bf K}_N + \sigma^2 {\bf I})^{-1} {\bf y}_N,  \quad  k(z,z; N) = k(z,z') - {\bf k}_N(z)^T ({\bf K}_N + \sigma^2 {\bf I})^{-1}  {\bf k}_N(z'), 
\end{align*}
so $\sigma^2(z; N) = k(z,z; N)$. For $z \in \cZ$, the posterior distribution of $f(z)$ is $\cN(\mu(z; N), \sigma^2(z; N))$.
We consider the posterior distribution of $f_j$ for each arm at every round, when it has been played some (random) number of times. For each arm $j$, denote the posterior mean and variance of $f_j$ at $z$ after $n$ plays of the arm by $\mu_j(z; n)$ and $\sigma_j^2(z; n)$. Let $N_j(t)$ be the (random) number of times arm $j$ has been played up to time $t$.
We denote the posterior mean and variance of arm $j$ at round $t$ by, 
\begin{align*}
 \mu_t(j) &= \mu_j(Z_{j,t}; N_j(t-1)), \qquad  \text{ and} \qquad \sigma_t^2(j) = \sigma^2_j(Z_{j,t}; N_j(t-1)).
 \end{align*}

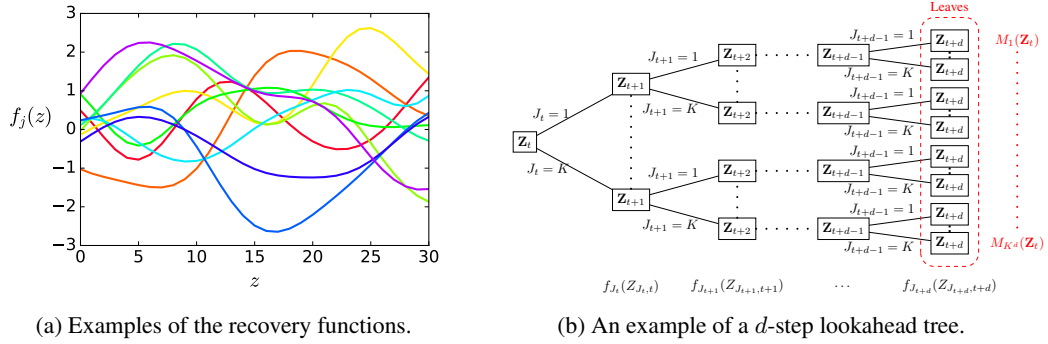
\begin{figure}
	\centering
	\begin{subfigure}{0.48\textwidth}
		\centering
 		 \includegraphics[height=4cm]{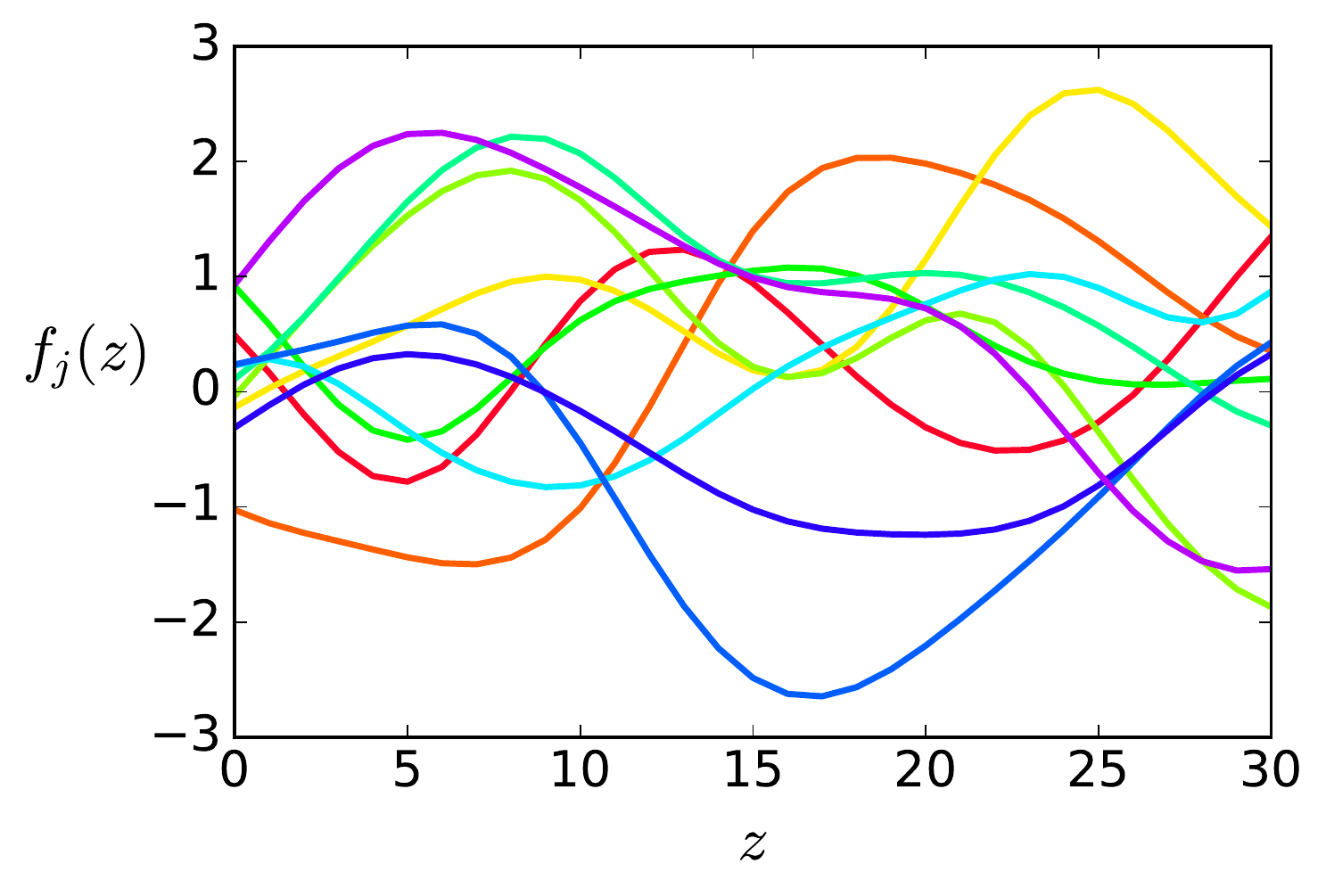}
 		\caption{Examples of the recovery functions.}
 		\label{fig:recfuncts}
	\end{subfigure}
	\quad
	\begin{subfigure}{0.48\textwidth}
		\centering
	 	\tikzset{
 dots/.style={
        line width=1.5pt,
        line cap=round,
        dash pattern=on 0pt off 5\pgflinewidth
    }
}

\begin{tikzpicture}[level distance=2.75cm,
level 1/.style={sibling distance=3cm},
level 2/.style={sibling distance=1.5cm},
level 3/.style={sibling distance=1cm},
level 4/.style={sibling distance=0.75cm},
grow=right]
\tikzstyle{every node}=[draw]

\node (Root)  {${\bf Z}_t$}
child {
    node(a) {${\bf Z}_{t+1}$} 
 	  child { node(a1) {${\bf Z}_{t+2}$}
 	  	child[dots]{ node[solid, line width=0.5pt] {${\bf Z}_{t+d-1}$}
 	  		child[solid, line width=0.5pt]{ node(a11)[solid, line width=0.5pt] {${\bf Z}_{t+d}$}
 	  		edge from parent node[left, draw=none, yshift=-0.3cm, xshift=0.5cm]{$J_{t+d-1}=K$}}
 	  		child[solid, line width=0.5pt]{node(a12)[solid, line width=0.5pt]{${\bf Z}_{t+d}$}
 	  		edge from parent node[left, draw=none, yshift=0.3cm, xshift=0.5cm]{$J_{t+d-1}=1$}}}
	edge from parent node[left,draw=none, yshift=-0.3cm, xshift=0.5cm]{$J_{t+1}=K$}} 
   	 child { node(a2) {${\bf Z}_{t+2}$}
 	  	child[dots]{ node[solid, line width=0.5pt] {${\bf Z}_{t+d-1}$}
 	  		child[solid, line width=0.5pt]{ node(a21)[solid, line width=0.5pt] {${\bf Z}_{t+d}$}
 	  		edge from parent node[left, draw=none, yshift=-0.3cm, xshift=0.5cm]{$J_{t+d-1}=K$}}
 	  		child[solid, line width=0.5pt]{node(a22)[solid, line width=0.5pt]{${\bf Z}_{t+d}$}
 	  		edge from parent node[left, draw=none, yshift=0.3cm, xshift=0.5cm]{$J_{t+d-1}=1$}}}
	edge from parent node[left,draw=none, yshift=0.3cm, xshift=0.5cm]{$J_{t+1}=1$}} 
edge from parent node[left,draw=none,align=center]{$J_t=K$}}
child {
    node(b) {${\bf Z}_{t+1}$} 
 	  child { node(b1) {${\bf Z}_{t+2}$}
 	  	child[dots]{ node[solid, line width=0.5pt] {${\bf Z}_{t+d-1}$}
 	  		child[solid, line width=0.5pt]{ node(b11)[solid, line width=0.5pt] {${\bf Z}_{t+d}$}
 	  		edge from parent node[left, draw=none, yshift=-0.3cm, xshift=0.5cm]{$J_{t+d-1}=K$}}
 	  		child[solid, line width=0.5pt]{node(b12)[solid, line width=0.5pt]{${\bf Z}_{t+d}$}
 	  		edge from parent node[left, draw=none, yshift=0.3cm, xshift=0.5cm]{$J_{t+d-1}=1$}}}
	edge from parent node[left,draw=none, yshift=-0.3cm, xshift=0.5cm]{$J_{t+1}=K$}} 
   	 child { node(b2) {${\bf Z}_{t+2}$}
 	  	child[dots]{ node(c)[solid, line width=0.5pt] {${\bf Z}_{t+d-1}$}
 	  		child[solid, line width=0.5pt]{ node(b21)[solid, line width=0.5pt] {${\bf Z}_{t+d}$}
 	  		edge from parent node[left, draw=none, yshift=-0.3cm, xshift=0.5cm]{$J_{t+d-1}=K$}}
 	  		child[solid, line width=0.5pt]{node(b22)[solid, line width=0.5pt]{${\bf Z}_{t+d}$}
 	  		edge from parent node[left, draw=none, yshift=0.3cm, xshift=0.5cm]{$J_{t+d-1}=1$}}}
	edge from parent node[left,draw=none, yshift=0.3cm, xshift=0.5cm]{$J_{t+1}=1$}} 
edge from parent node[left,draw=none,align=center]{$J_t=1$}};

\draw[dots](a)--(b);
\draw[dots](a1)--(a2);
\draw[dots](b1)--(b2);
\draw[dotted, line width=1.5pt](a11)--(a12);
\draw[dotted, line width=1.5pt](a21)--(a22);
\draw[dotted, line width =1.5pt](b11)--(b12);
\draw[dotted, line width =1.5pt](b21)--(b22);

\draw[dashed,rounded corners=10, red]($(a11) + (-0.75,-0.6)$)rectangle($(b22) +(0.75,0.6)$);

\path let \p1 = (a) in node[draw=none] at (\x1,-3.75) {$f_{J_t}(Z_{J_t,t})$};
\path let \p1 = (a1) in node[draw=none] at (\x1,-3.75) {$f_{J_{t+1}}(Z_{J_{t+1},t+1})$};
\path let \p1 = (c) in node[draw=none] at (\x1,-3.75) {$\dots$};
\path let \p1 = (a11) in node[draw=none] at (\x1,-3.75) {$f_{J_{t+d}}(Z_{J_{t+d},t+d})$};

\path let \p1 = (a11) in node[draw=none, red] at (\x1,3.5) {Leaves};

\path let \p1 = (a11) in node[draw=none, red](MK) at (12.75,\y1) {$M_{K^d}({\bf Z}_t)$};
\path let \p1 = (b22) in node[draw=none, red](M1) at (12.75,\y1) {$M_{1}({\bf Z}_t)$};
\draw[dots, red](M1)--(MK);

\end{tikzpicture}
 		\caption{An example of a $d$-step lookahead tree.}
 	 	\label{fig:tree}
	\end{subfigure}
	\caption{Illustration of recovery functions and lookahead trees.}
\end{figure}

\section{Defining the Regret} \label{sec:reg}
The regret is commonly used to measure the performance of an algorithm and is defined as the difference in the cumulative expected reward of an algorithm and an oracle. 
We will use the Bayesian regret, where the expectation is taken over the recovery curves and the actions. 
In recovering bandits, there are various choices for the oracle. We discuss some of these here. 

\paragraph{Full Horizon Regret.} 
One candidate for the oracle is the deterministic policy which knows the recovery functions and $T$, and using this selects the best sequence of $T$ arms.
This policy can be horizon dependent. Anytime algorithms, which are horizon independent, lead to policies that are stationary and do not change over time. In various settings, these stationary deterministic policies achieve the best possible regret
 \cite{puterman2005markov}. In the following, we focus on 
 the stationary deterministic (SD) oracle.
Note that it is computationally intractable to calculate this oracle in all but the easiest problems. This can be seen by formulating the problem as an MDP, with natural state-space 
of size $K^{|\cZ|}$. Techniques such as dynamic programming cannot be used unless $K$ and $|\cZ|$ are very small.

\paragraph{Instantaneous Regret.}
Another candidate for the oracle is the policy which in each round $t$, greedily plays the arm with the highest immediate reward at $\bZ_t$. These $\bZ_t$ depend on the previous actions of the oracle. Consider a policy which plays this oracle up to time $s-1$, then selects a different action at time $s$, and continues to play greedily. The cumulative reward of this policy could be vastly different to that of the oracle since they may have very different $\bZ$ values.
Therefore, defining regret in relation to this oracle may penalize us severely for early mistakes. 
Instead, one can define the regret of a policy $\pi$ with respect to an oracle which selects the best arm \emph{at the $\bZ_t$'s generated by $\pi$}. We call this the \emph{instantaneous regret}. This regret is commonly used in restless bandits and in   \citep{mintz2017nonstationary}.

\paragraph{$d$-step Lookahead Regret.}
A policy with low instantaneous regret may miss out on additional reward by not considering the impact of its actions on future $\bZ_t$'s. Looking ahead and considering the evolution of the $Z_{j,t}$'s can lead to choosing sequences of arms which are collectively better than individual greedy arms. For example, if two arms $j_1, j_2$ have similar $f_j(Z_{j,t})$ but the reward of $j_1$ doubles if we do not play it, while the reward of $j_2$ stays the same, it is better to play $j_2$ then $j_1$.
We will consider oracles which take the $\bZ_t$ generated by our algorithm and select the best sequence of $d \geq 1$ arms. We call this regret the \emph{$d$-step lookahead regret} and will use this throughout the paper.

To define this regret, we use decision trees.
 Nodes are ${\bZ}$ values and edges represent playing arms and updating $\bZ$ (see Figure~\ref{fig:tree}). Each sequence of $d$ arms is a leaf of the tree. Let $\mathcal{L}_d(\bZ)$ be the set of leaves of a $d$-step lookahead tree with root $\bZ$. For any $i \in \mathcal{L}_d(\bZ)$, denote by $M_i({\bZ})$ the expected reward at that leaf, that is the sum of the $f_j$'s along the path to $i$ at the relevant $Z_{j}$'s (see Section~\ref{sec:rgpucb}). 
The $d$-step lookahead
oracle  selects the leaf with highest $M_i(\bZ_t)$ from a given root node $\bZ_t$, denote this value by $M^*(\bZ_t)$. This leaf is the best sequence of $d$ arms from $\bZ_t$.
If we select leaf $I_t$ at time $t$, we
 play the arms to $I_t$ for $d$ steps, so select a leaf every $d$ rounds. The $d$-step lookahead regret is,
\[ \E[\Reg_T^{(d)}] =  \sum_{h=0}^{\lfloor T/d \rfloor} \E \bigg[ M^*(\bZ_{hd+1}) - M_{I_{hd+1}}(\bZ_{hd+1}) \bigg], \]
with expectation over $I_{hd+1}$ and $f_j$.
If $d=T$ or $d=1$, we get the full horizon or instantaneous regret.
We study the single play regret, $\E[\Reg_T^{(d,s)}]$, where arms can only be played once in a $d$-step lookahead, and 
the multiple play regret, $\E[\Reg_T^{(d,m)}]$, which allows multiple plays of an arm in a lookahead.
This regret is related to that in episodic reinforcement learning (ERL) \citep{jaksch2010near,osband2013more,azar2017minimax}. A key difference is that in ERL, the initial state is reset or re-sampled every $d$ steps independent of the actions taken. 
Note that the $d$-step lookahead regret can be calculated for any policy, regardless of whether the policy is designed to look ahead and select sequences of $d$ actions.

For large $d$, the total reward from the optimal $d$-step lookahead policy will be similar to that of the optimal full horizon stationary deterministic policy. Let $V_T(\pi)$ be the total reward of policy $\pi$ up to horizon $T$ and note that the optimal SD policy will be periodic by Lemma~\ref{lem:period} (Appendix~\ref{app:lookahead}). Then,
\begin{restatable}{proposition}{proplookahead}
\label{prop:lookahead}
Let $p^*$ be the period of the optimal SD policy $\pi^*$. For any $l =1, \dots, \lfloor \frac{T-z_{\max}}{p^*} \rfloor$, the optimal $(z_{\max} + lp^*)$-lookahead policy, $\pi_l^*$, satisfies,
$ V_T(\pi_l^*) \geq \big( 1-\frac{(l+1)p^* + z_{\max}}{T+p^*}\big) \frac{lp^*}{lp^* + z_{\max}} V_T(\pi^*)$.
\end{restatable}
Hence, any algorithm with low $(z_{\max} + lp^*)$-step lookahead regret will also have high total reward. In practice, we may not know the periodicity of $\pi^*$. Moreover, if $p^*$ is too large, then looking $(z_{\max} + lp^*)$ steps ahead may be computationally challenging, and prohibit learning. Hence, we may wish to consider smaller values of $d$. 
One option is to look far enough ahead that we consider a local maximum of each recovery function. For a GP kernel with lengthscale $l$ (e.g. squared exponential or Mat{\'e}rn), this requires looking $2l$ steps ahead \citep{murray2016gaussian,rasmussen2006gaussian}.
This should still give large reward while being computationally more efficient and allowing for learning.

\section{Gaussian Processes for Recovering Bandits} \label{sec:rgpucb}
In Algorithm~\ref{alg:drgpucb} we present a UCB ($d$RGP-UCB) and Thompson Sampling ($d$RGP-TS) algorithm for the $d$-step lookahead recovering bandits problem, for both the single and multiple play case.
Our algorithms use Gaussian processes to model the recovery curves, allowing for efficient estimation and facilitating the lookahead.
For each arm $j$ we place a GP prior on $f_j$ and initialize $Z_{j,1}$ (often this initial value is known, otherwise we set it to 0). 
Every $d$ steps we construct the $d$-step lookahead tree as in Figure~\ref{fig:tree}. At time $t$, 
we select a sequence of arms by choosing a leaf $I_t$ of the tree with root $\bZ_t$. 
For a leaf $i \in \mathcal{L}_d(\bZ_t)$, let $\{J_{t+\ell}\}_{\ell=0}^{d-1}$ and$\{Z_{J_{t+\ell},t+\ell}\}_{\ell=0}^{d-1}$ be the sequences of arms and $z$ values (which are updated using \eqref{eqn:zdef}) on the path to leaf $i$. Then define the total reward at $i$ 
as, 
\[ M_i(\bZ_t) = \sum_{\ell=0}^{d-1} f_{J_{t+\ell}}(Z_{J_{t+\ell},t+\ell}). \]
Since the posterior distribution of $f_j(z)$ is Gaussian, the posterior distribution of the leaves of the lookahead tree will also be Gaussian. In particular, $\forall i \in \mathcal{L}_d(\bZ_t)$, 
$M_i(\bZ_t) \sim \cN(\eta_t(i), \varsigma_t^2(i))$ where,
\begin{align}
 \eta_t(i) &=  \sum_{\ell=0}^{d-1} \mu_{t}(J_{t+\ell}), \quad  \varsigma^2_t(i) = \sum_{\ell,q =0}^{d-1} \text{cov}_{t}(f_{J_{t+\ell}}(Z_{J_{t+\ell}, t+\ell}), f_{J_{t+q}}(Z_{J_{t+q}, t+q})), \text{ and,  } \label{eqn:varm}
 \end{align}
$\text{cov}_{t}(f_{J_{t+\ell}}(Z_{J_{t+\ell}, t+\ell}), f_{J_{t+q}}(Z_{J_{t+q}, t+q}))= \one \{J_{t+\ell} = J_{t+q}\} k_{J_{t+\ell}}(Z_{J_{t+\ell}, t+\ell}, Z_{J_{t+q}, t+q}; N_{J_{t+\ell}}(t))$. 
Hence, using GPs enables us to accurately estimate the reward and uncertainty at the leaves.

\begin{algorithm}[tb]
   \caption{$d$-step lookahead UCB and Thompson Sampling}
   \label{alg:drgpucb}
\begin{algorithmic}
   \STATE {\bfseries Input:} $\alpha_t$ from \eqref{eqn:alphadef} (for UCB). 
   \STATE {\bfseries Initialization:} Define $\cT_d = \{1, d+1, 2d+1 , \dots \}$. For all arms $j \in A$, set $Z_{j,1}= 0$ (optional). 
   \FOR{$t \in \cT_d$}
   	\STATE{If $t\geq T$ {\bf break}. Else, construct the $d$-step lookahead tree. Then,
   			\begin{align*} 
   			\begin{array}{l}
   			\text{If UCB: } 
   			\\ I_t = \argmax \limits_{ i \in \mathcal{L}_d(\bZ_t)} \bigg \{\eta_t(i) + \alpha_t \varsigma_t(i) \bigg \},
   			\end{array}
   			\begin{array}{lll}
        							\text{If TS:} \hspace{-7pt}&  \text{(i)}& \hspace{-10pt} \forall j \in A, \text{ sample } \tilde f_j \text{ from the posterior at } \mathbf{Z}_{j,t}^{(d)}. \\
        							 &\text{(ii)} & \hspace{-10pt} \forall i \in \mathcal{L}_d(\bZ_t), \tilde \eta_t(i) = \sum_{l=0}^{d-1} \tilde f_{J_{t+\ell}}(Z_{J_{t+\ell}, t+\ell}) \\
        							 &\text{(iii)}& \hspace{-5pt}  I_t = \argmax_{i \in \mathcal{L}_d(\bZ_t)} \{\tilde \eta_t(i) \}
       						 \end{array}
			\end{align*}}
   	\FOR{$\ell=0, \dots, d-1$}
   		\STATE{Play $\ell$th arm to $I_t$, $J_\ell$, and get reward $Y_{J_{\ell},t+\ell}$.}
   		\STATE{ Set $Z_{J_{\ell},t+\ell+1} = 0$. For all $j \neq J_{\ell}$, set $Z_{j,t+\ell+1} = \min\{Z_{j,t+\ell}+1, z_{\max}\}$.}
	\ENDFOR
	   \STATE{Update the posterior distributions of the played arms.} 
   \ENDFOR
\end{algorithmic}
\end{algorithm}

For $d$RGP-UCB, we construct upper confidence bounds on each $M_i(\bZ_t)$ using Gaussianity. We then select the leaf $I_t$ with largest upper confidence bound at time $t$. That is,
\begin{align}
I_t &= \argmax_{1 \leq i \leq K^d} \{ \eta_t(i) + \alpha_t \varsigma_t(i) \} \qquad \text{ where } \qquad \alpha_t = \sqrt{2 \log((K|\cZ|)^d (t+d-1)^2)}. \label{eqn:alphadef}
\end{align} 
In $d$RGP-TS, we select a sequence of $d$ arms by sampling the recovery function of each arm $j$ at $\bZ_{j,t}^{(d)} = (Z_{j,t}, \dots, Z_{j,t}+d-1, 0, \dots, d-1)^T$ and then calculating the `reward' of each node using these sampled values. Denote the sampled reward of node $i$ by $\tilde \eta_t(i)$. We choose the leaf $I_t$ with highest $\tilde \eta_t(i)$. 
In both $d$RGP-UCB and $d$RGP-TS, by the lookahead property, we will only play an arm at a large $Z_{j,t}$ value if it has high reward, or high uncertainty there. We play the sequence of $d$ arms indicated by $I_t$ over the next $d$ rounds. 
We then update the posteriors and repeat the process. 

We analyze the regret in the single and multiple play cases separately.
Studying the single play case first allows us to gain more insights about the difficulty of the problem. Indeed, from our analysis we observe that the multiple play case is more difficult since we may loose information from not updating the posterior between plays of the same arm. 
All proofs are in the appendix.
The regret of our algorithms will depend on the GP kernel
through the maximal information gain. For a set $\cS$ of covariates and observations $Y_{\cS} = [f(z) + \epsilon_z]_{z \in \cS}$, we define the \emph{information gain}, $\cI(Y_{\cS} ;f)= H(Y_{\cS}) - H(Y_{\cS}|f)$
where $H(\cdot)$ is the entropy. 
As in \citep{srinivas2009gaussian}, 
we consider the \emph{maximal information gain} from $N$ samples, $\gamma_N$. If $z_t \in \cS$ is played at time $t$, then,
\begin{align}
\cI(Y_{\cS},f) &= \frac{1}{2} \sum_{t=1}^{|\cS|} \log(1+ \sigma^{-2} \sigma^2(z_t; t-1)),  \qquad \text{ and, } \qquad \gamma_N = \max_{\cS \subset \cZ^N: |\cS|=N} \cI(Y_{\cS};f).  \label{eqn:gammadef}
\end{align}
Theorem 5 of \cite{srinivas2009gaussian} gives bounds on $\gamma_T$ for some kernels. We apply these results using the fact that the dimension, $D$, of the input space is $1$. 
For any lengthscale, $\gamma_T = O(\log(T))$ for linear kernels, $\gamma_T = O(\log^2(T))$ for squared exponential kernels, and $\gamma_T = O(T^{2/(2\nu+2)} \log(T))$ for Mat\'ern($\nu$).

\subsection{Single Play Lookahead}
In the single play case, each am can only be played once in the $d$-step lookahead. This simplifies the variance in \eqref{eqn:varm} since the arms are independent. For any leaf $i$ corresponding to playing arms $J_{t}, \dots, J_{t+d-1}$ (at the corresponding $z$ values), $\varsigma_t^2(i) = \sum_{\ell=0}^{d-1} \sigma^2_t(J_{t+\ell})$.
This involves the posterior variances at time $t$. 
However, as we cannot repeat arms, if we play arm $j$ at time $t+\ell$ for $0\leq \ell \leq d-1$, it cannot have been played since time $t$, so its posterior is unchanged. Using this and \eqref{eqn:gammadef}, we relate the variance of $M_{I_t}(\bZ_t)$ to 
the information gain about the $f_j$'s.
We get the following regret bounds.

\begin{restatable}{theorem}{thmregds}
\label{thm:regrgpds}
The $d$-step single play lookahead regret of $d$RGP-UCB satisfies,
\[ \E[\Reg_T^{(d,s)}] \leq O(\sqrt{KT \gamma_T \log(TK|\cZ|)}). \]
\end{restatable}

\begin{restatable}{theorem}{thmregdsts}
\label{thm:regdsts}
The $d$-step single play lookahead regret of $d$RGP-TS satisfies,
\[ \E[\Reg_T^{(d,s)}] \leq O(\sqrt{KT \gamma_T \log(TK|\cZ|)}). \]
\end{restatable}

\subsection{Multiple Play Lookahead}
When arms can be played multiple times in the $d$-step lookahead, the problem is more difficult since 
we cannot use feedback from plays within the same lookahead to inform decisions.
It is also harder to relate $\varsigma^2_t(I_t)$ to the information gain about each $f_j$. In particular, $\varsigma^2_t(I_t)$ contains covariance terms and is defined using the posteriors at time $t$. On the other hand, $\gamma_T$ is defined in terms of the posterior variances when each arm is played. These may be different to those at time $t$ if an arm is played multiple times in the lookahead.
However, using the fact that the posterior covariance matrix is positive semi-definite, $2k_j(z_1, z_2; n) \leq \sigma_j^2(z_1; n) + \sigma_j^2(z_2; n)$, so we can bound $\varsigma_t^2(I_t) \leq 3  \sum_{\ell=0}^{d-1} \sigma_t^2(J_{t+\ell})$.
Then, the change in the posterior variance of a repeated arm 
can be bounded by the following lemma. 

\begin{restatable}{lemma}{lemsigmadiff}
\label{lem:sigmadiff}
For any $z \in \cZ$, arm $j$ and $n \in \mathbb{N}, n\geq 1$, let $Z_j^{(n)}$ be the $z$ value at the $n$th play of arm $j$. Then, 
$\sigma_{j}^2(z; n-1) - \sigma_{j}^2(z; n) \leq \sigma^{-2}{\sigma_j^2(Z_j^{(n)}; n-1) }$.
\end{restatable}

We get the following regret bounds for $d$RGP-UCB and $d$RGP-TS.  Due to not updating the posterior between repeated plays of an arm, they increase by a factor of $\sqrt{d}$ compared to the single play case. 
Thus, although by Proposition~\ref{prop:lookahead} larger $d$ leads to higher reward, it makes the learning problem harder.

\begin{restatable}{theorem}{thmregdm}
\label{thm:regdm}
The $d$-step multiple play lookahead regret of $d$RGP-UCB satisfies,
\[ \E[\Reg_T^{(d,m)}] \leq O\left(\sqrt{KT\gamma_T \log((K|\cZ|)^d T)}\right). \]
\end{restatable}

\begin{restatable}{theorem}{thmregtsdm}
\label{thm:regtsdm}
The $d$-step multiple play lookahead regret of $d$RGP-TS satisfies,
\[ \E[\Reg_T^{(d,m)}] \leq O\left(\sqrt{KT\gamma_T \log((K|\cZ|)^d T)}\right). \]
\end{restatable}

\subsection{Instantaneous Algorithm}
If we set $d=1$ in Algorithm~\ref{alg:drgpucb}, we obtain algorithms for minimizing the instantaneous regret. In this case, $\cT=\{1,\dots,T\}$ and there are $K$ leaves of the $1$-step lookahead tree, so each $M_i(\bZ_t)$ corresponds to one arm. One arm is selected and played each time step
so $\eta_t(i) = \mu_t(j)$, $\varsigma_t^2(i) = \sigma_t^2(j)$ for some $j$. 
 For the UCB, we define $\alpha_t$ as in \eqref{eqn:alphadef} with $d=1$.
We get the following regret,

\begin{restatable}{corollary}{corregrgpucb}
\label{cor:regrgpucb}
The instantaneous regret of $1$RGP-UCB  and $1$RGP-TS up to horizon $T$ satisfy
\[ \E[\Reg_T^{(1)}] \leq O(\sqrt{K T \gamma_T \log(TK |\cZ|)}). \]
\end{restatable}

The instantaneous regret of both algorithms is $O^*(\sqrt{KT\gamma_T})$.
Hence, we reduced the dependency on $|\cZ|$ from $\sqrt{|\cZ|}$ to $\sqrt{\log|\cZ|}$ compared to a naive application of UCB (see Appendix~\ref{app:nonparam}). 
The single play lookahead regret is of the same order as this instantaneous regret. This shows that, in the single play case, since we still update the posterior after every play of an arm, we do not loose any information by looking ahead.

\section{Improving Computational Efficiency via Optimistic Planning} \label{sec:opplanning}
For large values of $K$ and $d$, Algorithm~\ref{alg:drgpucb} may not be computationally efficient since it searches $K^d$ leaves.
We can improve this by optimistic planning \citep{hren2008optimistic,munos2014bandits}. This was developed by \cite{hren2008optimistic} for deterministic MDPs with discount factors and rewards in $[0,1]$. We adapt this to undiscounted rewards in $[\min_{j,z} \tilde f_j(z), \max_{j,z} \tilde f_j(z)]$. 
We focus on this in the multiple play Thompson sampling algorithm.

As in Algorithm~\ref{alg:drgpucb}, at time $t$, we sample $\tilde f_j(z)$ from the posterior of $f_j$ at $\bZ_{j,t}^{(d)} = (Z_{j,t}, \dots, Z_{j,t}+d, 0, \dots d)^T$ for all arms $j$. 
Then, instead of searching the entire tree to find the leaf with largest total $\tilde f_j(z)$, 
we use optimistic planning (OP) to iteratively build the tree.
We start from an initial tree of one node, $i_0=\bZ_t$. At step $n$ of the OP procedure, let $\mathcal{T}_n$ be the expanded tree and let $\mathcal{S}_n$ be the nodes not in $\mathcal{T}_n$ but whose parents are in $\mathcal{T}_n$. We select a node in $\mathcal{S}_n$ and move it to $\mathcal{T}_n$, adding its children to $\mathcal{S}_n$. If we select a node $i_n$ of depth $d$, we stop and output node $i_n$. 
Otherwise we continue until 
 $n=N$ for a predefined budget $N$. Let $d_N$ be the maximal depth of nodes in $\cT_N$. We output the node at depth $d_N$ with  largest upper bound on the value of its continuation (i.e. with largest $b_N(i)$ in \eqref{eqn:bdef}).

Nodes are selected using upper bounds on the total value of a continuation of the path to the node. For node $i \in \cS_n \cup \cT_n$, let $u(i)$ be the sum of the $\tilde f_j(z)$'s on the path to $i$, and $l(i)$ the depth of $i$. The value, $v(i)$, of node $i$ is the reward to $i$, $u(i)$, plus the maximal reward of 
a path from $i$ to depth $d$. We upper bound $v(i)$ by,
\begin{align}
b_n(i) = u(i) + \Psi(\mathbf{z}(i), d- l(i)) \qquad \text{ for } i \in \mathcal{S}_n \cup \mathcal{T}_n. \label{eqn:bdef}  
\end{align}
where $\mathbf{z}(i)$ is the vector of $z_j$'s at node $i$, and the function $\Psi(\mathbf{z}(i),d-l(i))$ is an upper bound on the maximal reward from node $i$ to a leaf. 
Let $g_j(z,l) = \max\{\tilde f_j(z), \dots, \tilde f_j(z+l), \tilde f_j(0), \dots, \tilde f_j(l)\}$ be the maximal reward that can be gained from playing arm $j$ in the next $1 \leq l \leq d$ steps from $z \in \bZ_{j,t}^{(d)}$. Then, $\Psi(\mathbf{z}(i),d-l(i)) =(d- l(i)) \max_{1 \leq j \leq K} g_j(z_j(i), d-l(i))$, 
and $\Psi(\mathbf{z}(i), 0) =0$ for any $\mathbf{z}(i)$.

We can often bound the error from this procedure. Let $v^*= \max_{i\in\mathcal{L}_d(\bZ)} v(i)$ be the value of the maximal node. 
A node $i$ is $\epsilon$-optimal if $v^*-v(i) \leq \epsilon$, and let $p_l(\epsilon)$ be the proportion of $\epsilon$-optimal nodes at depth $l$. Let $\Delta =\max_{j,z} \tilde f_j(z)- \min\{ \min_{j,z} \tilde f_j(z), 0\}$ and define $\Psi^*(l) = \max_{\mathbf{z} \in \cZ} \Psi(\mathbf{z},l)$ for any $l=0,1,\dots, d$. 
We get the following bound (whose proof is in Appendix~\ref{app:opplan}).

\begin{restatable}{proposition}{lemopplanning}
\label{lem:opplanning}
For the optimistic planning procedure with budget $N$, if the procedure stops at step $n<N$ because a node $i_n$ of depth $d$ is selected, then $v^*-v(i_n)=0$. Otherwise, if there 
exist $\lambda \in (\frac{1}{K},1]$ and $1\leq d_0\leq d$ such that $\forall l\geq d_0$, $p_l((d-l) \Delta) \leq \lambda^l$, then for $N>n_0=\frac{K^{d_0+1}-1}{K-1}$,
\begin{align}
v^*- v(i_N)\leq\bigg(d- \frac{\log(N-n_0)}{\log(\lambda K)} - \frac{\log(\lambda K -1)}{\log(\lambda K)} +1 \bigg) \Delta . \label{eqn:opplan}
\end{align}
\end{restatable}

Hence, if we stop the procedure at $n<N$, the node $i_n$ of depth $d$ we return will be optimal.
In many cases, especially when the proportion of $\epsilon$-optimal nodes, $\lambda$, is small, this will occur. 
Otherwise, the 
error will depend on $\lambda$, and the budget, $N$. By \eqref{eqn:opplan}, for $N \approx (\lambda K)^d$, the error will be near zero.

\section{Experimental Results} \label{sec:exp}

\begin{figure*}[t]
  \centering
    \begin{subfigure}{0.48\textwidth}
    \includegraphics[width=\textwidth]{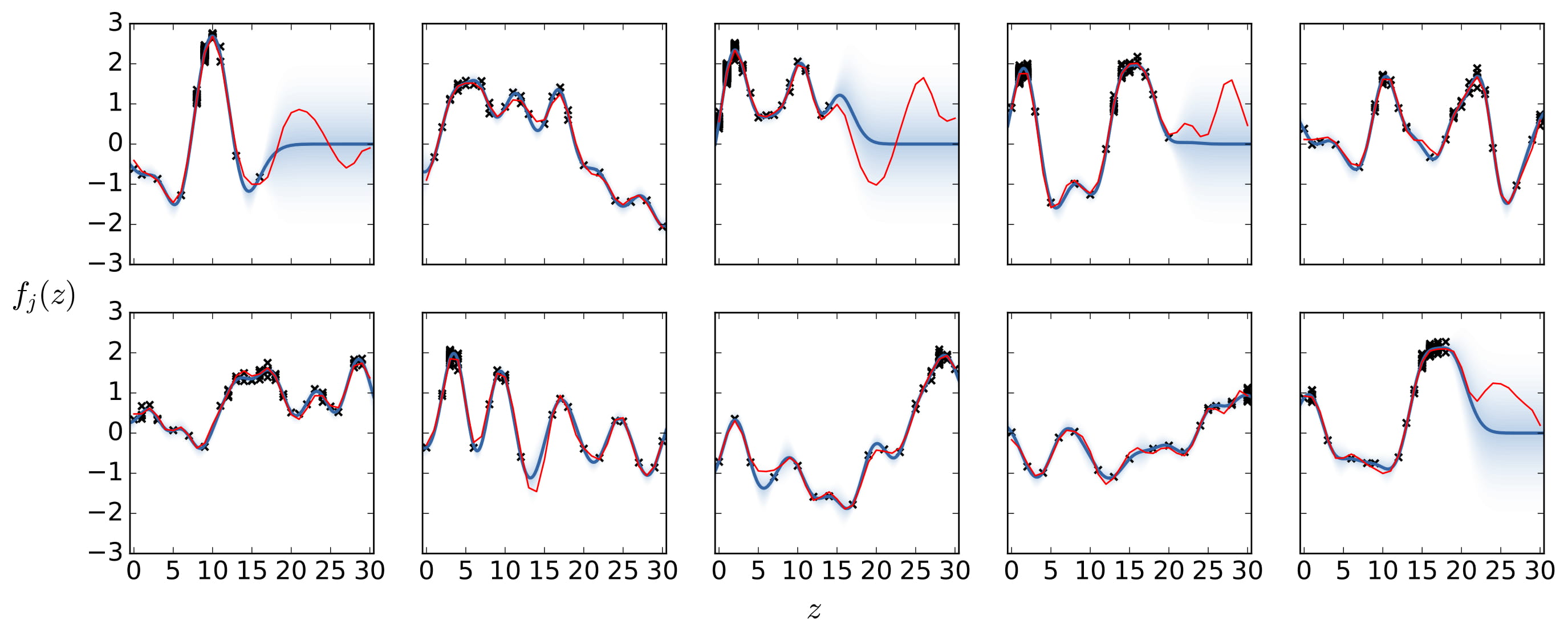}
    \caption{Instantaneous: $d=1$}
    \end{subfigure}
    \quad
        \begin{subfigure}{0.48\textwidth}
    \includegraphics[width=\textwidth]{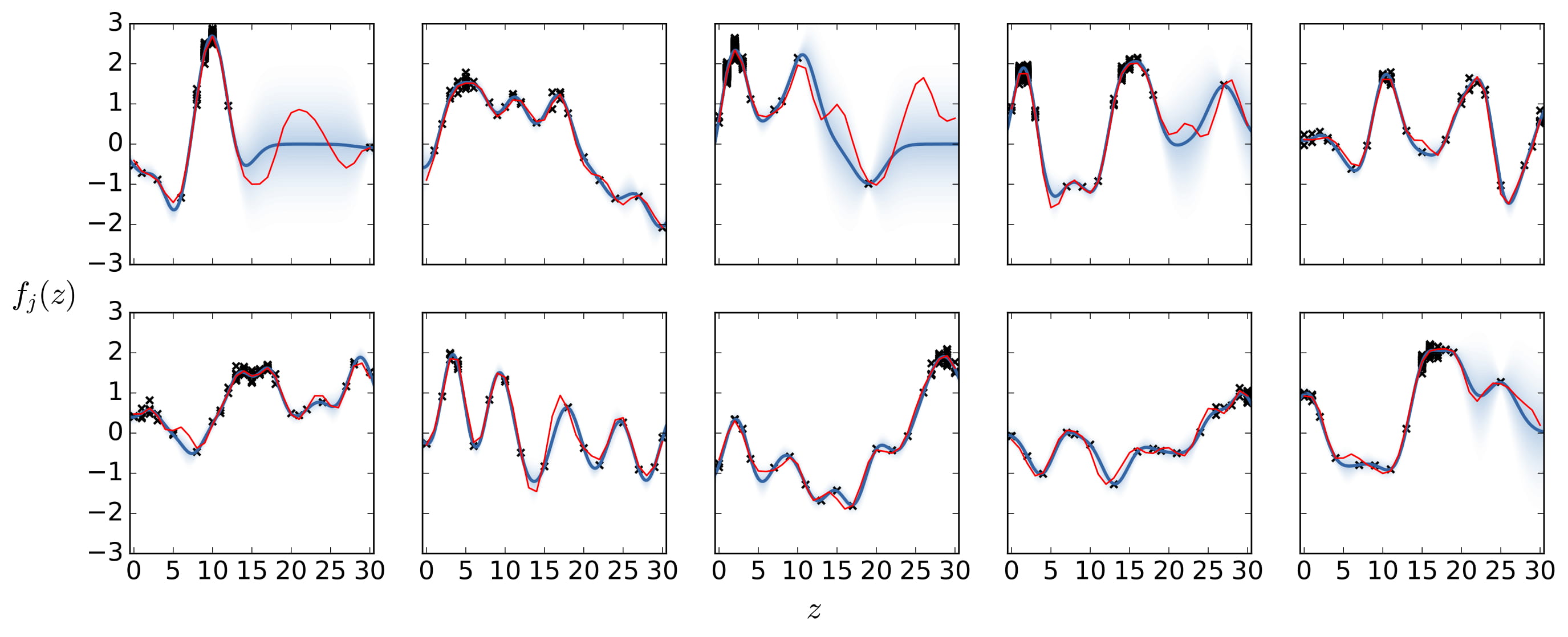}
        \caption{Lookahead: $d=3$}
    \end{subfigure}
    \caption{The posterior mean (blue) of RGP-UCB with density shaded in blue for a squared exponential kernel ($l=2$). The true recovery curve is in red and the crosses are the observed samples.}
      \label{fig:post}
\end{figure*}

\begin{figure*}
  \centering
    \begin{subfigure}{0.32\textwidth}
    \centering
    \includegraphics[width=\textwidth]{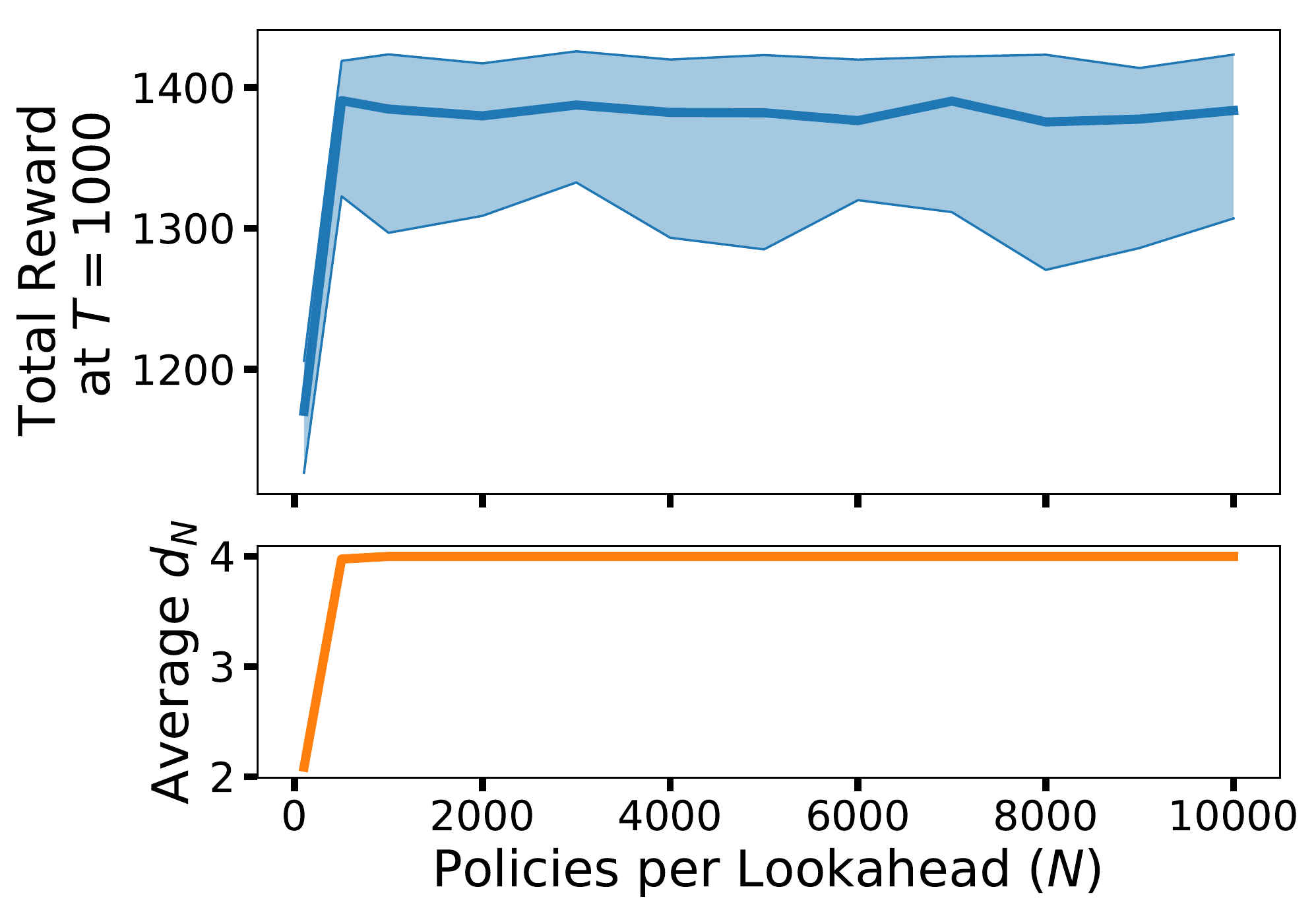}
    \caption{$d=4$, $K=10$} \label{fig:oppland4k10}
    \end{subfigure}
        \begin{subfigure}{0.32\textwidth}
        \centering
    \includegraphics[width=\textwidth]{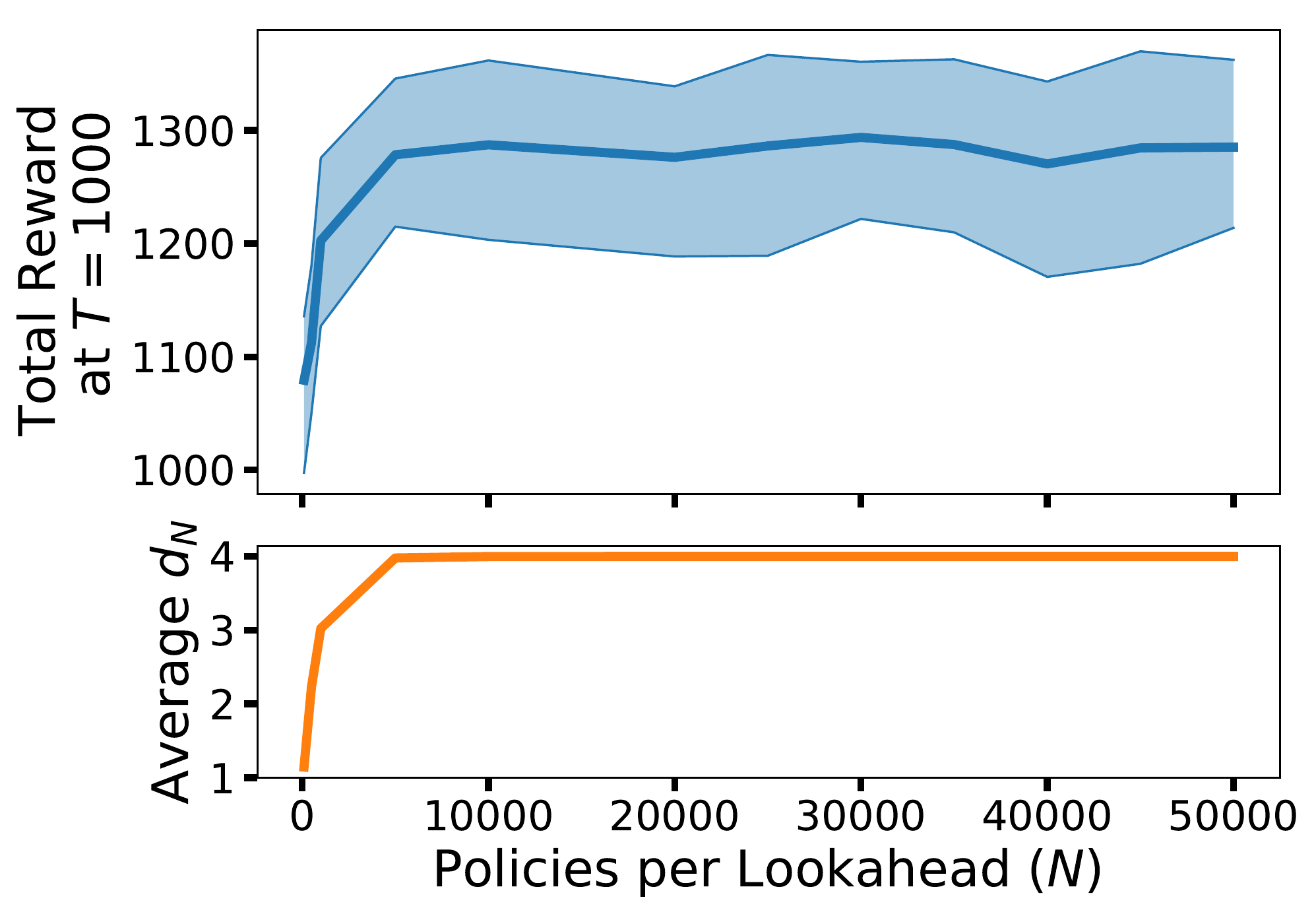}
        \caption{$d=4$, $K=30$} \label{fig:oppland4k30}
    \end{subfigure}
       \begin{subfigure}{0.32\textwidth}
       \centering
    \includegraphics[width=\textwidth]{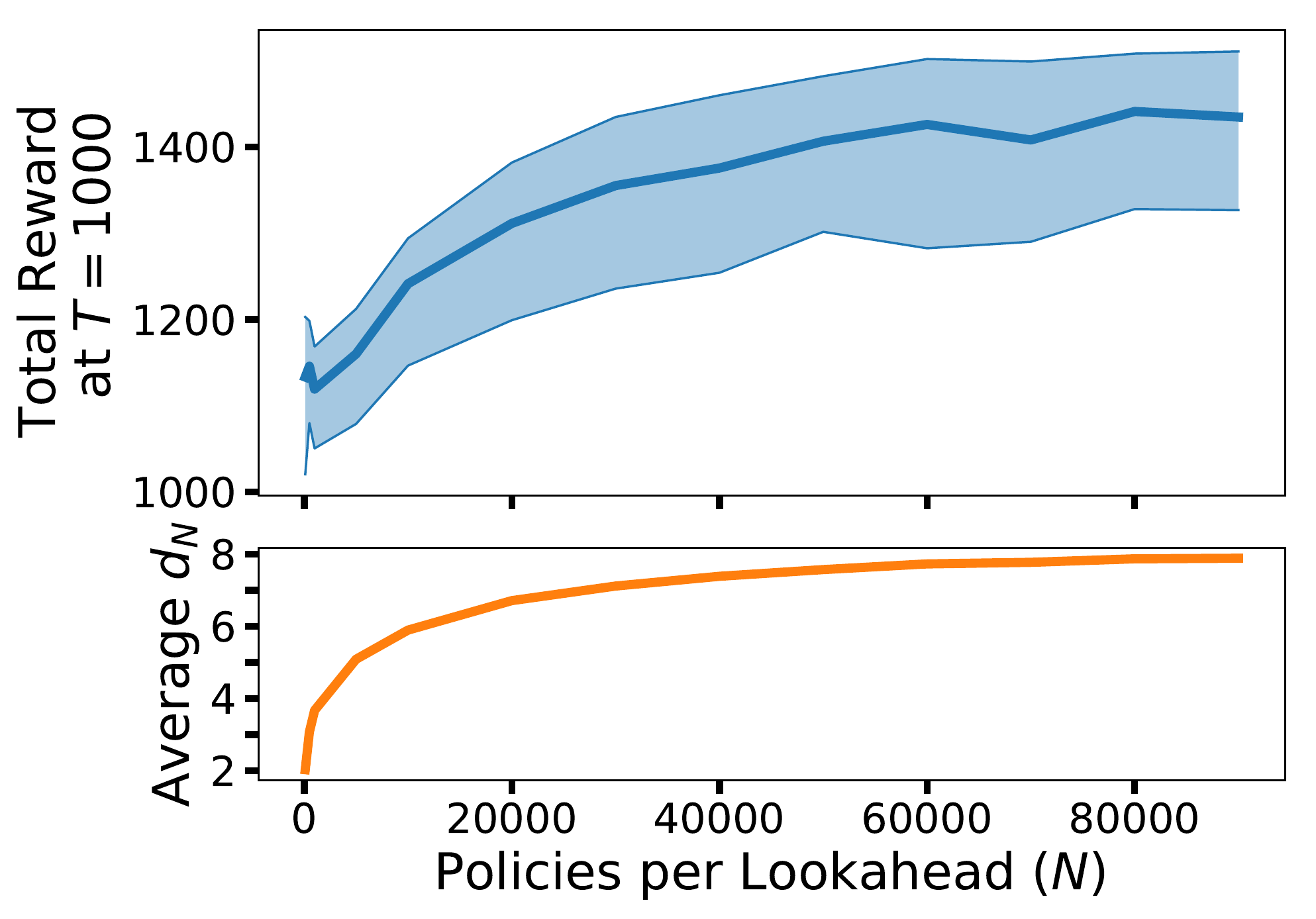}
        \caption{$d=8$, $K=10$} \label{fig:oppland10k10}
    \end{subfigure}
    \caption{The total reward and final depth of the lookahead tree, $d_N$, as the budget, $N$, increases.}
      \label{fig:opplan}
\end{figure*}

We tested our algorithms in experiments with $z_{\max}=30$, noise standard deviation $\sigma = 0.1$, and horizon $T=1000$. 
We used GPy \citep{gpy2014} to fit the GPs.
We first aimed to check that our algorithms were playing arms at good $z$ values (i.e. play arm $j$ when $f_j(z)$ is high). We used $K=10$ and sampled the recovery functions from a squared exponential kernel and ran the algorithms once. 
Figure~\ref{fig:post} shows that, for lengthscale $l=2$, $1$RGP-UCB and $3$RGP-UCB accurately estimate the recovery functions and learn to play each arm in the regions of $\cZ$ where the reward is high. Although, as expected, $3$RGP-UCB has more samples on the peaks, it is reassuring that the instantaneous algorithm
also selects good $z$'s. The same is true for $d$RGP-TS and different values of $d$ and $l$ (see Appendix~\ref{app:tsplot}). 

Next, we tested the performance of using optimistic planning (OP) in $d$RGP-TS. We averaged all results over 100 replications and used a squared exponential kernel with $l=4$. In the first case, $K=10$ and $d=4$, so 
direct tree search may have been possible. Figure~\ref{fig:oppland4k10} shows that, when $N$, the budget of policies the OP procedure can evaluate per lookahead, increases above $500$, the total reward plateaus
and the average depth of the returned policy is $4$.
By Proposition~\ref{lem:opplanning}, this means that we found an optimal leaf of the lookahead tree while evaluating significantly fewer policies. 
We then increased the number of arms to $K=30$. Here, searching the whole lookahead tree would be computationally challenging. Figure~\ref{fig:oppland4k30} shows that we found the optimal policy after about 5,000 policies (since here $d_N=d$), which is less than $0.1\%$ of the total number of policies. In Figure~\ref{fig:oppland10k10}, we increased the depth of the lookahead to $d=8$. 
In this case, we had to search more policies to find optimal leaves. However, this
was still less than $0.1\%$ of the total number of policies. From Figure~\ref{fig:oppland10k10}, we also see that when $d_N < d$, increasing $N$ leads to higher total reward.

Lastly, we compared our algorithms to RogueUCB-Tuned \citep{mintz2017nonstationary} and UCB-Z, the basic UCB algorithm with $K|Z|$ arms (see Appendix~\ref{app:nonparam}), in two parametric settings. Details of the implementation of RogueUCB-Tuned are given in Appendix~\ref{app:rogueucb}.
As in \citep{mintz2017nonstationary}, we only considered $d=1$. We used squared exponential kernels in $1$RGP-UCB and $1$RGP-TS with lengthscale $l=5$ (results for other lengthscales are in Appendix~\ref{app:thetas}).
The recovery functions were logistic, $f(z) = \theta_0(1+ \exp\{-\theta_1(z-\theta_2)\})^{-1}$ which increases in $z$,
and modified gamma, $f(z) = \theta_0C\exp\{-\theta_1z\}z^{\theta_2}$ (with normalizer $C$), which
 increases until a point and then decreases. The values of $\theta$ were sampled uniformly and are in Appendix~\ref{app:thetas}. We averaged the results over 500 replications. The cumulative regret (and confidence regions) is shown in Figure~\ref{fig:regloggam} and the cumulative reward (and confidence bounds) in Table~\ref{tab:rew_loggamma}. Our algorithms achieve lower regret and higher reward than RogueUCB-Tuned.
UCB-Z does badly since the time required to play each (arm,$z$) pair during initialization is large.

\subsection{Practical Considerations}
There are several issues to consider when applying our algorithms in a practical recommendation scenario. The first is the choice of $z_{\max}$. Throughout, we assumed that this is known and constant across arms. Our work can be easily extended to the case where there is a different value, $z_{\max,j}$, for each arm $j$, by defining $z_{\max} = \max_j z_{\max,j}$ and extending $f_j$ to $z_{\max}$ by setting $f_j(z)=f_j(z_{\max,j})$ for all $z=z_{\max,j}+1, \dots, z_{\max}$. A similar approach can be used if we only know an upper bound on $z_{\max}$. 
Additionally, in practice, the recovery curves may not be sampled from Gaussian processes, or the kernels may not be known.
As demonstrated experimentally, our algorithms can still perform well in this case. Indeed,
kernels can be chosen to (approximately) represent a wide variety of recovery curves, ranging from uncorrelated rewards to constant functions
In practice, we can also use adaptive algorithms for selecting a kernel function out of a large class of kernel functions (see e.g. Chapter 5 of [24] for details).

\begin{table*}[t]
\caption{Total reward at $T=1000$ for single step experiments with parametric functions}
\label{tab:rew_loggamma}
\begin{center}
\begin{tabular}{c c c c c}
\toprule
Setting & 1RGP-UCB ($l=5$) & 1RGP-TS ($l=5$) & RogueUCB-Tuned & UCB-Z \\
\midrule
Logistic & 461.7  (454.3,468.9) &  462.6  (455.7,469.3)  & 446.2  (438.2,453.5) & 242.6 (229.6,256.0) \\ 
\rule{0pt}{2ex}    Gamma & 145.6  (139.6, 151.7)  & 156.5  (149.6,163.0) & 132.7 (111.0,144.5) & 116.8 (108.4,125.5)   \\ 
 \bottomrule
\end{tabular}
\end{center}
\end{table*}

\vspace{-10pt}
\begin{figure}[t]
    \centering
    \begin{subfigure}{0.48\columnwidth}
    	\centering
        \includegraphics[width=.7\textwidth]{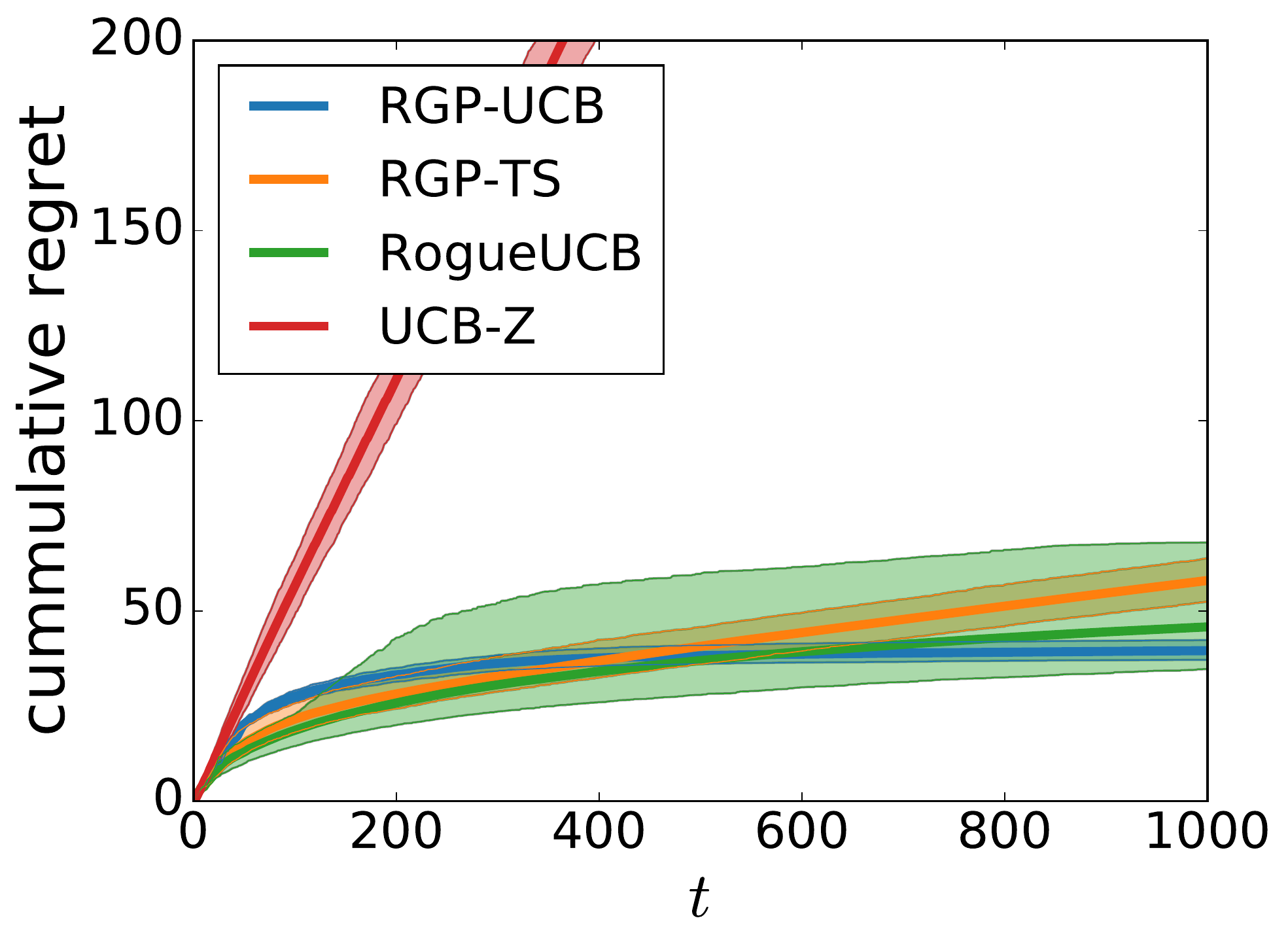}
        \caption{Logistic setup}
        \label{fig:reg_logistic}
    \end{subfigure}
    \begin{subfigure}{0.48\columnwidth}
    	\centering
       	\includegraphics[width=.7\textwidth]{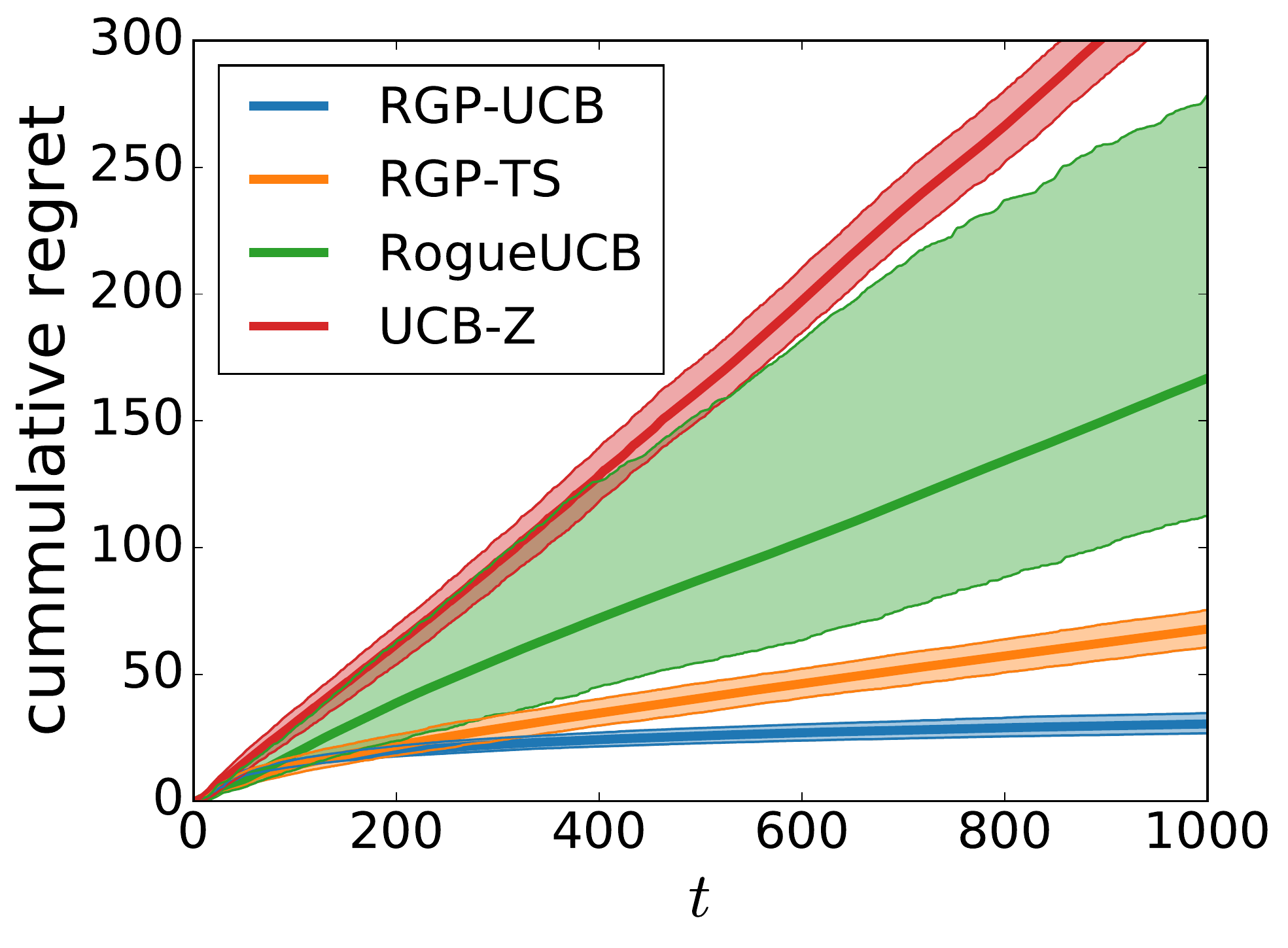}
        \caption{Gamma setup}
        \label{fig:reg_gamma}
    \end{subfigure}
    \caption{Cumulative instantaneous regret with parametric recovery functions.} \label{fig:regloggam}
\end{figure}

\section{Conclusion}
In this work, we used Gaussian processes to model the recovering bandits problem and incorporated this into UCB and Thompson sampling algorithms. These algorithms use properties of GPs to look ahead and find good sequences of arms. They achieve $d$-step lookahead Bayesian regret of $O^*(\sqrt{KdT})$ for linear and squared exponential kernels, and perform well experimentally. We also improved computational efficiency of the algorithm using optimistic planning. Future work 
includes 
considering the frequentist setting, analyzing online methods for choosing $z_{\max}$ and the kernel, and investigating the use of GPs in other non-stationary bandit settings.

\subsubsection*{Acknowledgments}
CPB would like to thank the EPSRC funded EP/L015692/1 STOR-i centre for doctoral training and Sparx. We would also like to thank the reviewers for their helpful comments.

{\small
\bibliographystyle{abbrv}
\bibliography{recband_refs}
}

\clearpage
\onecolumn
\appendix

{\center \huge Supplementary Material for Recovering Bandits}

\section{Preliminaries} \label{app:gen}
Define the filtration $\{\cF_t\}_{t=0}^\infty$ as $\cF_0=\emptyset$ and 
\begin{align}
\cF_t = \sigma(J_1, \dots, J_t, Y_1, \dots, Y_t, \bZ_1, \dots, \bZ_t) \label{eqn:Fdef}
\end{align}
 where $\bZ_t=[Z_{1,t}, \dots, Z_{K,t}]$. It is important to note that $\mu_t(j), \sigma_t(j), J_t$ and $\bZ_t$ are  $\cF_{t-1}$ measurable.

Recall that in both $d$RGP-UCB and $d$RGP-TS, we select a sequence of arms to play at time $t$ by building a $d$-step lookahead tree with root $\bZ_t$ and selecting the leaf node $i$ with highest upper confidence bound on $M_i$, the cumulative reward from playing all arms in that sequence,
\[ M_i(\bZ_t)= \sum_{\ell=0}^{d-1} f_{J_{t+\ell}}(Z_{J_{t+\ell},t+\ell}) \]
where $\{J_{t+\ell}\}_{\ell=0}^{d-1}$ are the sequence of arms played on the path to leaf $i$ and $\{Z_{J_{t+\ell},t+\ell}\}_{\ell=0}^{d-1}$ the corresponding $z$ values. Denote the posterior mean and variance of $M_i(\bZ_t)$ at time $t$ as $\eta_t(i)$ and $\varsigma_t(i)$, then, conditional on the history $\cF_{t-1}$, $M_i(\bZ_t)\sim \cN (\eta_t(i) , \varsigma^2_t(i))$. When each arm can be played multiple times, there are interaction terms in the variance of the $M_i(\bZ_t)$'s and thus we suffer some additional cost for not updating after every play. For each leaf node $i$, we can calculate
\[ \varsigma_t^2(i) = \sum_{\ell=0}^{d-1} \sigma_t^2(J_{t+\ell}) + \sum_{\ell \neq q; \ell,q=0}^{d-1} \text{cov}_{t}(f_{J_{t+\ell}}(Z_{J_{t+\ell}, t+\ell}), f_{J_{t+q}}(Z_{J_{t+q}, t+q})) \]
where $\text{cov}_{t}(f_{J_{t+\ell}}(Z_{J_{t+\ell}, t+\ell}), f_{J_{t+q}}(Z_{J_{t+q},t+q}))$ is 0 if $J_{t+\ell} \neq J_{t+q}$ and $k_{J_{t+\ell}}(Z_{J_{t+\ell}, t+\ell}, Z_{J_{t+q}t+q}; N_{J_{t+\ell}}(t-1))$ for $J_{t+\ell} = J_{t+q}$. Note that throughout, we assume that the variances and covariances are calculated at the $Z_{j,t}$'s where the arms are played, ie. $ \sigma_t^2(J_{t+\ell}) = \sigma^2_{J_{t+\ell}} (Z_{J_{t+\ell},t+\ell}; N_{J_{t+\ell}}(t-1))$.

Before providing the proofs of the regret bounds, we need the following lemmas,

\begin{lemma} \label{lem:boundsigma}
 \[  \sum_{t=1}^T \sum_{j=1}^K \sigma_t^2(J_t) \one\{J_t=j\} \leq C_1 K\gamma_T. \]
 where $C_1 = 1/\log(1+ \sigma^{-2})$.
\end{lemma}
\begin{proof}
Using the results of Lemma 5.4 of \cite{srinivas2009gaussian} and the fact that the maximal information gain is increasing in the number of data points, it follows that
\begin{align*}
\sum_{t=1}^T \sum_{j=1}^K \sigma_t^2(J_t) \one\{J_t=j\} &= \sum_{j=1}^K \sum_{n=1}^{N_j(T)} \sigma_j^2(Z_j^{(n)}; n-1)
\\ & \hspace{-40pt}  \leq  T\sum_{j=1}^K C_1 I(\mathbf{y}_{j,N_j(T)}; \mathbf{f}_{j,N_j(T)})  \leq C_1 \sum_{j=1}^K \gamma_{N_j(T)} \leq C_1K\gamma_T. 
\end{align*}
\end{proof} 

The following lemmas bound the amount of information we loose by only updating the posterior every $d$ steps in the case where we can play each arm multiple times in a $d$-step lookahead. The result proves Lemma~\ref{lem:sigmadiff} in the main text.
\begin{lemma}
For any $z \in \cZ$ arm $j$ and $n \in \mathbb{N}, n\geq 1$, let $Z^{(n)}$ be the $z$ value when arm $j$ is played for the $n$th time. Then, 
\[ \sigma_{j}^2(z; n-1) - \sigma_{j}^2(z; n) = \frac{k^2_j(Z_j^{(n)},z; n-1)}{\sigma^2_j(Z_j^{(n)}; n-1) + \sigma^2} \leq \frac{\sigma_j^2(Z_j^{(n)}; n-1) }{ \sigma^2} \]
\end{lemma}
\begin{proof}
For convenience, we drop the $j$ notation and let $\bk_{n}(z) = [k(Z^{(1)}, z), \dots, k(Z^{(n)},z)]^T$ and $\bK_n = [k(Z^{(i)},Z^{(j)})]_{i,j=1}^n$. Then,
\begin{align}
 & \hspace{-10pt} \sigma^2(z; n-1) - \sigma^2(z; n)  \nonumber
\\  &= k(z,z) - {\bf k}_{n-1}(z)^T ({\bf K}_{n-1} + \sigma^2 {\bf I})^{-1}{\bf k}_{n-1}(z) - k(z,z) + {\bf k}_{n}(z)^T ({\bf K}_{n} + \sigma^2 {\bf I})^{-1}{\bf k}_{n}(z) \nonumber 
 \\ &=  {\bf k}_{n}(z)^T ({\bf K}_{n} + \sigma^2 {\bf I})^{-1}{\bf k}_{n}(z)  -  {\bf k}_{n-1}(z)^T ({\bf K}_{n-1} + \sigma^2 {\bf I})^{-1}{\bf k}_{n-1}(z) \label{eqn:kerndecomp}
\end{align}
We write,
\begin{align*}
 {\bf k}_n(z) = \begin{bmatrix} {\bf k}_{n-1}(z) \\ k(Z^{(n)},z) \end{bmatrix} \qquad {\bf K}_n + \sigma^2 {\bf I} = \begin{pmatrix} {\bf K}_{n-1} + \sigma^2 {\bf I} & {\bf k}_{n-1}(z) \\ {\bf k}_{n-1}(z)^T & k(Z^{(n)},Z^{(n)}) + \sigma^2 \end{pmatrix} = \begin{pmatrix} \bA & \bB \\ \bB^T & C \end{pmatrix}.
 \end{align*}
 Then, by the block matrix inversion formula,
 \begin{align*}
 (\bK_n + \sigma^2 \bI)^{-1} = \begin{pmatrix} \bA^{-1} + \bA^{-1}\bB(C-\bB^T\bA^{-1}\bB)^{-1}\bB^T \bA^{-1} & -\bA^{-1}\bB(C-\bB^T\bA^{-1}\bB)^{-1} 
 							\\ -(C-\bB^T \bA^{-1} \bB)^{-1}\bB^T \bA^{-1} & (C-\bB^T\bA^{-1}\bB)^{-1} \end{pmatrix}
 \end{align*}
Hence,
 \begin{align*}
 &\hspace{-15pt} \bk_n(z)^T (\bK_n + \sigma^2 \bI)^{-1} \bk_n(z) =  [\bk_{n-1}(z)^T, k(Z^{(n)},z)] (\bK_n + \sigma^2 \bI)^{-1} \begin{bmatrix} \bk_{n-1}(z) \\ k(Z^{(n)},z) \end{bmatrix} 
 \\ &= \bk_{n-1}(z)^T (\bA^{-1} + \bA^{-1}\bB(C-\bB^T \bA^{-1} \bB)^{-1}\bB^T \bA^{-1})\bk_{n-1}(z) 
 	\\ & \qquad - k(Z^{(n)},z)(C-\bB^T \bA^{-1} \bB)^{-1}\bB^T \bA^{-1}\bk_{n-1}(z)
 	\\ & \qquad - \bk_{n-1}(z)^T \bA^{-1} \bB(C-\bB^T \bA^{-1} \bB)^{-1} k(Z^{(n)},z) 
 	\\ & \qquad + k(Z^{(n)},z) (C-\bB^T \bA^{-1} \bB)^{-1} k(Z^{(n)},z)
\\ &= \bk_{n-1}(z)^T \bA^{-1} \bk_{n-1}(z) 
	\\ & \qquad + \bk_{n-1}(z)^T (\bA^{-1}\bB(C-\bB^T \bA^{-1} \bB)^{-1}(\bB^T \bA^{-1} \bk_{n-1}(z) - k(Z^{(n)},z))
	\\ & \qquad + (k(Z^{(n)},z) - \bk_{n-1}(z)^T \bA^{-1} \bB)(C-\bB^T\bA^{-1} \bB)^{-1} k(Z^{(n)},z)
\\ & = \bk_{n-1}(z)^T \bA^{-1} \bk_{n-1}(z) 
	\\ & \qquad + (k(Z^{(n)},z) - \bk_{n-1}(z)^T \bA^{-1}\bB) (C-\bB^T\bA^{-1}\bB)^{-1}(k(Z^{(n)},z) - (\bk_{n-1}(z)^T\bA^{-1}\bB)^T)
 \end{align*}
 Then, substituting back $\bA=\bK_{n-1}+ \sigma^2 \bI, \bB=\bk_{n-1}(z), C=k(Z^{(n)},z_{(n)}) + \sigma^2$ gives,
 \begin{align*}
\bk_n(z)^T (\bK_n + \sigma^2 \bI)^{-1} \bk_n(z)  =& \bk_{n-1}(z)^T (\bK_{n-1} + \sigma^2 \bI)^{-1}\bk_{n-1}(z) 
	\\ & \quad + (k(Z^{(n)},z) - \bk_{n-1}(z)^T(\bK_{n-1}+ \sigma^2 \bI)^{-1}\bk_{n-1}(z)) 
		\\ & \qquad (k(Z^{(n)},Z^{(n)}) - \bk_{n-1}(z_n)^T(\bK_{n-1}+ \sigma^2 \bI)^{-1}\bk_{n-1}(z) + \sigma^2)^{-1} 
		\\ & \qquad (k(Z^{(n)},z) - (\bk_{n-1}(z)^T(\bK_{n-1}+ \sigma^2 \bI)^{-1}\bk_{n-1}(z))^T) 
	\\ &\hspace{-90pt}  =  \bk_{n-1}(z)^T (\bK_{n-1} + \sigma^2 \bI)^{-1}\bk_{n-1}(z)  + \frac{k^2(Z^{(n)},z; n-1)}{\sigma^2(Z^{(n)}; n-1) + \sigma^2}
\end{align*}
Hence, substituting into \eqref{eqn:kerndecomp} gives,
\[  \sigma^2(z; n-1) - \sigma^2(z; n)  =  \frac{k^2(Z^{(n)},z; n-1)}{\sigma^2(Z^{(n)}; n-1) + \sigma^2}. \]
Then, since the covariance matrix is positive semi-definite, for any $z,z'$ and $m \in \mathbbm{N}$, $k(z,z'; m) \leq \sqrt{\sigma^2(z; m) \sigma^2(z'; m)}$ and so
\begin{align*}
 \sigma^2(z; n-1) - \sigma^2(z; n)  \leq  \frac{\sigma^2(Z^{(n)}; n-1) \sigma^2(z; n-1)}{\sigma^2(Z^{(n)}; n-1) + \sigma^2} \leq \frac{\sigma^2(Z^{(n)}; n-1) }{ \sigma^2}
\end{align*}
since for any $z \in \cZ$ and $m \in \mathbbm{N}$, $0 \leq \sigma^2(z;m) \leq 1$. This concludes the proof.
\end{proof}

We then use this result in the following lemma,
\begin{lemma} \label{lem:varbound}
For any leaf node $i$ of the $d$-step look ahead tree constructed at time $t$,
\[ \varsigma_t^2(i) \leq  3 \sum_{j=1}^K \bigg( \sum_{m=N_j(t)+1}^{N_j(t+d)} \frac{N_j(t+d)-m+1}{\sigma^2} \sigma_{j}^2(z^{(m)}; m-1) \bigg)= \zeta_t^2 \]
and $\zeta_t$ is $\cF_{t-1}$ measurable.
\end{lemma}
\begin{proof}
First note that since the posterior covariance matrix of $f_j$ is positive semi-definite, for any $z_1, z_2$ and number of samples, $n-1$, $k_j(z_1,z_2; n-1) \leq 1/2(\sigma_j^2(z_1; n-1) + \sigma_j^2(z_2; n-1))$. Hence,
\[ \varsigma_t(i) \leq 3 \sum_{\ell=0}^{d-1} \sigma_t^2(J_{t+\ell}). \]
Now consider arm $j$ and assume it appears $s\leq d$ times in the $d$-step look ahead policy selected at time $t$. Then, the contribution of arm $j$ (which for ease of notation we assume has been played $n-1$ times previously) to $\varsigma_t^2(i)$ is given below where we use the notation $\sigma_j^2(z^{(i)}; n-1)$ to denote the posterior variance at the $i$th $z$ of arm $j$ given $n-1$ observations of arm $j$. 
\begin{align*}
& \hspace{-10pt} \sum_{m=n}^{n+s-1} \sigma_{j}^2 (Z_j^{(m)}; n-1) = \sigma_{j}^2(z^{(n)}; n-1) + \dots + \sigma_{j}^2(z^{(n+s-1)}; n-1)
\\ & = \sigma_{j}^2(z^{(n)}; n-1) + \sigma_{j}^2(z^{(n+1)}; n) + \left(\sigma_{j}^2(z^{(n+1)}; n-1) - \sigma_{j}^2(z^{(n+1)}; n) \right) + \dots 
	\\ & \qquad+ \sigma_{j}^2(z^{(n+s-1)}; n+s-2) + \left(\sigma_{j}^2(z^{(n+s-1)}; n+s-3) - \sigma_{j}^2(z^{(n+s-1)}; n+s-2)\right)
	\\ & \hspace{135pt} + \dots  +\left(\sigma_{j}^2(z^{(n+s-1)}; n-1) - \sigma_{j}^2(z^{(n+s-1)}; n)\right)
\\ & \leq \sigma_{j}^2(z^{(n)}; n-1) + \sigma_{j}^2(z^{(n+1)}; n) + \frac{\sigma_{j}^2(z^{(n)}; n-1)}{\sigma^2} + \dots 
	\\ & \qquad + \sigma_{j}^2(z^{(n+s-1)}; n+s-2) + \dots + \frac{\sigma_{j}^2(z^{(n+1)}; n)}{\sigma^2} + \frac{\sigma_{j}^2(z^{(n)}; n-1) }{\sigma^2}
\\ & = \sum_{q=0}^{s-1} (1 + \frac{s-q-1}{\sigma^2}) \sigma_{j}^2 (z^{(n+q)}; n+q -1) 
\\ & \leq \sum_{q=0}^{s-1}\frac{s-q}{\sigma^2}\sigma_{j}^2 (z^{(n+q)}; n+q -1) 
\end{align*}
which follows by recursively applying Lemma~\ref{lem:sigmadiff}.
Then, summing over all arms $j$ gives,
\begin{align*}
\varsigma^2_t(i) &\leq 3\sum_{j=1}^K \bigg( \sum_{m=N_j(t)+1}^{N_j(t+d)} \sigma_{j}^2(z^{(m)}; N_j(t)) \bigg)
\\ & \leq 3 \sum_{j=1}^K \bigg( \sum_{m=N_j(t)+1}^{N_j(t+d)} \frac{N_j(t+d)-m+1}{\sigma^2} \sigma_{j}^2(z^{(m)}; m-1) \bigg)
\end{align*}
Then, we note that $\zeta_t$ is $\cF_{t-1}$ measurable since for a given leaf node $i$ of the tree constructed at time $t$, the sequence of arms played to get to node $i$ is known so $N_j(t+d)$ will be known and also the sequence of $Z_j^{(m)}$'s where arm $j$ is played will also be known. Since the posterior variance of arm $j$ after $m$ plays depends only on the number of plays and the covariates (not the observed rewards), $\sigma_j^2(z^{(m)}; m-1)$ is $\cF_{t-1}$ measurable for $m=N_j(t)+1, \dots, N_j(t+d)$.
\end{proof}

We also need the following result on the expectation of the maximum.
\begin{lemma} \label{lem:maxgp}
Let $X_1, \dots X_n$ be Gaussian random variables such that $\max_{1 \leq i \leq n} \Var(X_i)\leq \zeta^2$. Then,
\[ \E \bigg[\max_{1 \leq i \leq n} X_i \bigg] \leq \zeta \sqrt{2 \log(n)}.\]
\end{lemma}
\begin{proof}
See for example, Lemma 2.2 in \cite{devroye2001combinatorial}.
\end{proof}

\section{Theoretical Results for $d$RGP-UCB} \label{app:ucb}
 We first prove the following lemma.
 \begin{lemma} \label{lem:concM}
For any leaf node $i$, initial node $z$ and constant $a>0$,
\[  \int_{a}^\infty  \PP(M_{i}(z) - \eta_t(i)  \geq x | \cF_{t-1}) dx \leq \sqrt{2\pi} \varsigma_t(i) \exp\bigg \{-\frac{a^2}{2 \varsigma_t^2(i)} \bigg\}. \]
\end{lemma}
 \begin{proof}
 The proof follows using the normality of the posterior of $M_i(z)$ (so at time $t$, $M_i(\bZ_t) \sim \cN( \eta_t(i), \varsigma_t(i)^2)$). 
 \begin{align*}
  \int_{a}^\infty  \PP(M_{i}(z) - \eta_t(i)  \geq x | \cF_{t-1}) dx & \leq  \int_{a}^\infty  \exp \bigg\{-\frac{x^2}{2 \varsigma_t^2(i)} \bigg\} dx
\\ & = \sqrt{2\pi} \varsigma_t(i) \int_{a}^\infty \frac{1}{\sqrt{2\pi} \varsigma_t(i)}  \exp\bigg \{-\frac{x^2}{2 \varsigma_t^2(i)} \bigg\} dx
\\ & \leq \sqrt{2\pi} \varsigma_t(i) \exp\bigg\{-\frac{a^2}{2 \varsigma_t^2(i)} \bigg\}.
 \end{align*}
 Where we have used that if $X \sim \cN(\mu,\sigma^2)$, $\PP(X-\mu \geq b) \leq \exp\{-\frac{b^2}{2\sigma^22}\}$ for any $b>0$, and the last inequality follows through integration of the pmf of a $\cN(0, \varsigma_t(i))$ random variable.
 \end{proof}
 
Then, define $M_{I_t^*}(\bZ_t)$ to be the sum of the $f_j(z)$'s to leaf ${I_t}^*$ of the optimal $d$ step look ahead policy from time $t$ chosen using the unknown $f_j(z)$'s. Let $r_t$ be the per step regret at time $t$. We now bound the expected regret from time steps $t, t+1, \dots, t+d-1$ where we have played arms according to the choice of $I_t$ by our algorithm. Let $r_s$ be the contribution to the regret at time $s$, that is $r_s = f_{J_t^*}(Z_{J_t^*,t}) - f_{J_t}(Z_{J_t,t})$. Then, let
 \[ \alpha_t = \sqrt{2\log((K|\cZ|)^d(t+d-1)^2)}. \]

We will use the following lemma,
\begin{lemma} \label{lem:dreg}
Assume we start a $d$-step look ahead policy at time $t$, selecting leaf node $I_t$, then
\begin{align*}
\sum_{s=t}^{t+d-1} \E[r_s |\cF_{t-1}] &\leq  \frac{ \sqrt{2d \pi}}{(t+d-1)^2} + \alpha_t \varsigma_t(I_t).  
\end{align*}
 \end{lemma}
 \begin{proof}
From \eqref{eqn:alphadef}, the upper confidence bound of node $i$ at time $t$ is given by,
\[ \eta_t(i) + \alpha_t \varsigma_t(i), \]
and since we play node $I_t$, this has the highest upper confidence bound. Then, we use the following decomposition of the regret,
 \begin{align*}
 \sum_{s=t}^{t+d-1} \E[r_s|\cF_{t-1}] &= \E[M_{I_t^*}(\bZ_t) - M_{I_t}(\bZ_t)| \cF_{t-1}]
 \\ & = \E[M_{I_t^*}(\bZ_t) - (\eta_t(I_t^*) +\alpha_t \varsigma_t(I_t^*)) +  (\eta_t(I_t^*) +\alpha_t \varsigma_t(I_t^*))  - M_{I_t}(\bZ_t)| \cF_{t-1}]
 \\ & \leq \E[M_{I_t^*}(\bZ_t) - (\eta_t(I_t^*) +\alpha_t \varsigma_t(I_t^*)) +  (\eta_t(I_t) +\alpha_t \varsigma_t(I_t))  - M_{I_t}(\bZ_t)| \cF_{t-1}]
 \\ & = \E[M_{I_t^*}(\bZ_t) - \eta_t(I_t^*) - \alpha_t \varsigma_t(I_t^*) | \cF_{t-1}] + \E[\eta_t(I_t) +\alpha_t \varsigma_t(I_t)  - M_{I_t}(\bZ_t)| \cF_{t-1}]
 \end{align*}
 For the first term, note that for any random variable $X$, $\E[X] \leq \E[X \one\{X>0\}] = \int_0^\infty \PP(X\geq x) dx$. Then, by Lemma~\ref{lem:concM} and using the fact that $\varsigma_t^2(i) \leq \sum_{\ell=0}^{d-1} k(z_{\ell}, z_{\ell}) \leq d$, it follows that,
 \begin{align*}
\E[M_{I_t^*}(\bZ_t) - \eta_t(I_t^*) - \alpha_t \varsigma_t(I_t^*) | \cF_{t-1}] &\leq \int_0^\infty \PP(M_{I_t^*}(\bZ_t) - \eta_t(I_t^*) - \alpha_t \varsigma_t(I_t^*) \geq x | \cF_{t-1}) dx
\\ & \leq  \int_0^\infty \sum_{i=1}^{K^d} \sum_{z \in \cZ^d} \PP(M_{i}(z) - \eta_t(i) - \alpha_t \varsigma_t(i) \geq x | \cF_{t-1}) dx
\\ & = \sum_{i=1}^{K^d} \sum_{z \in \cZ^d}  \int_{\alpha_t \varsigma_t(i)}^\infty  \PP(M_{i}(z) - \eta_t(i)  \geq x | \cF_{t-1}) dx
\\ & =  \sum_{i=1}^{K^d} \sum_{z \in \cZ^d} \sqrt{2\pi} \varsigma_t(i) \exp\bigg\{-\frac{(\alpha_t \varsigma_t(i))^2}{2 \varsigma_t^2(i)} \bigg\}
\\ & \leq  \sum_{i=1}^{K^d} \sum_{z \in \cZ^d} \sqrt{2d\pi}  \frac{1}{(t+d-1)^2 (K|\cZ|)^d}
\\ & = \frac{ \sqrt{2d \pi}}{(t+d-1)^2},
 \end{align*}
where the last inequality follows from the definition of $\alpha_t$.

For the second term, recall that $\eta_t(i) = \E[M_i(\bZ_t)| \cF_{t-1}]$ and $I_t$ is $\cF_{t-1}$ measurable. Hence,
\[ \E[\eta_t(I_t) +\alpha_t \varsigma_t(I_t)  - M_{I_t}(\bZ_t)| \cF_{t-1}] = \eta_t(I_t) +\alpha_t \varsigma_t(I_t) -\eta_t(I_t) = \alpha_t \varsigma_t(I_t). \]
Combining both terms gives the result.
 
 \end{proof}

We now prove the regret bounds for $d$RGP-UCB in the repeating and non-repeating cases.
\subsection{Non-Repeating} \label{app:drgps}
\thmregds*
\begin{proof}
For ease of notation define $\Reg_T$ as the $d$-step lookahead regret with single plays that we are interested in (i.e. $\Reg_T=\Reg_T^{(d,s)}$) and note that,
 \[  \E[\Reg_T] \leq \sum_{h=0}^{\lfloor T/d \rfloor} \E \bigg[ \sum_{s=hd+1}^{(h+1)d} \E[r_s|\cF_{hd}] \bigg]. \]
 Then, using Lemma~\ref{lem:dreg}, and the fact that since we cannot repeat plays, $\sigma_t(J_{t+\ell}) = \sigma_{t+\ell}(J_{t+\ell})$ for any $\ell=0,\dots, d-1$,
 \begin{align*}
& \E[\Reg_T] \leq \sum_{h=0}^{\lfloor T/d \rfloor} \E \bigg[ \sum_{s=hd+1}^{(h+1)d} \E[r_s|\cF_{hd}]\bigg] 
\\ & \leq \sum_{h=0}^{\lfloor T/d \rfloor} \E \bigg[ \frac{\sqrt{2d\pi}}{(h+1)^2d^2} +  \alpha_{hd+1}\sqrt{\varsigma^2_{hd+1}(I_{hd+1})} \bigg]
 \\ & \leq \frac{\sqrt{2 \pi}}{d} \sum_{h=1}^{\lfloor T/d \rfloor+1} \frac{1}{h^2} +  \sum_{h=0}^{\lfloor T/d \rfloor}\sqrt{2\log((K|\cZ|)^d (h+1)^2 d^2)} \E \bigg[\sqrt{\sum_{\ell=0}^{d-1} \sigma_{hd+1}(J_{hd +1+\ell})}\bigg]
 \\ & \leq \frac{\pi^{5/2}}{\sqrt{2}3d} + \sqrt{4\log((K|\cZ|)^d(T+d))} \sqrt{\lfloor T/d \rfloor +1} \E \bigg[ \sqrt{ \sum_{t=1}^T \sigma_t^2(J_t)}\bigg]
 \\ & \leq \frac{\pi^{5/2}}{\sqrt{2}3d} + \sqrt{4\log((K|\cZ|)^d(T+d))} \sqrt{\lfloor T/d \rfloor +1} \E \bigg[   \sqrt{\sum_{j=1}^K \ \sum_{t=1}^T \sigma_t^2(j) \one\{J_t=j\}} \bigg]
  \\ & \leq \frac{\pi^{5/2}}{\sqrt{2}3d} + \sqrt{4\log((K|\cZ|)^d(T+d))} \sqrt{\lfloor T/d \rfloor +1}\sqrt{C_1 K \gamma_T}
\end{align*}
where $C_1= 1/\log(1+\sigma^{-2})$ and the last line follows by Lemma~\ref{lem:boundsigma}. This gives the result.
 
\end{proof}

\subsection{Repeating} \label{app:drgpm}
\thmregdm*
\begin{proof}
For ease of notation define $\Reg_T$ as the $d$-step lookahead regret with multiple plays that we are interested in (i.e. $\Reg_T=\Reg_T^{(d,m)}$) and note that,
 \[  \E[\Reg_T] = \sum_{h=0}^{\lfloor T/d \rfloor} \E \bigg[ \sum_{s=hd+1}^{(h+1)d} \E[r_s|\cF_{hd}] \bigg]. \]
Then, note that from Lemma~\ref{lem:varbound}, it follows that
\[ \varsigma^2_t(i) \leq  3 \sum_{j=1}^K \bigg( \sum_{m=N_j(t)+1}^{N_j(t+d)} \frac{N_j(t+d)-m+1}{\sigma^2} \sigma_{j}^2(z^{(m)}; m-1) \bigg)  \leq \frac{3d}{\sigma^2} \sum_{j=1}^K \sum_{m=N_j(t)+1}^{N_j(t+d)} \sigma_{j}^2(z^{(m)}; m-1). \]
Hence, by lemma~\ref{lem:dreg} and summing over all time points where we start a $d$-step look ahead policy, it follows that,
\begin{align*}
& \E[\Reg_T] = \sum_{h=0}^{\lceil T/d \rceil-1} \E \bigg[ \sum_{s=hd+1}^{(h+1)d} \E[r_s|\cF_{hd}]\bigg] 
\\ & \leq \sum_{h=0}^{\lfloor T/d \rfloor} \E \bigg[ \frac{\sqrt{2d\pi}}{(h+1)^2d^2} +  \alpha_{hd+1}\sqrt{\varsigma^2_{hd+1}(I_{hd+1})} \bigg]
 \\ & \leq \frac{\sqrt{2 \pi}}{d} \sum_{h=1}^{\lfloor T/d \rfloor+1} \frac{1}{h^2} +  \sum_{h=0}^{\lfloor T/d \rfloor}\sqrt{2\log((K|\cZ|)^d (h+1)^2 d^2)} \E \bigg[\sqrt{\frac{3d}{\sigma^2} \sum_{j=1}^K \sum_{m=N_j(dh)+1}^{N_j(d(h+1))} \sigma_{j}^2(z^{(m)}; m-1) }  \bigg] 
\\ & \leq \frac{\pi^{5/2}}{\sqrt{2}3d} + \sqrt{\frac{12d}{\sigma^2}\log((K|\cZ|)^d(T+d))} \sqrt{\lfloor T/d \rfloor +1}  \E \bigg[\sqrt{ \sum_{h=0}^{\lfloor T/d \rfloor} \sum_{j=1}^K \sum_{m=N_j(dh)+1}^{N_j(d(h+1))} \sigma_{j}^2(z^{(m)}; m-1) }  \bigg] 
 \end{align*}
 Then, from Lemma~\ref{lem:boundsigma} and the fact that $\gamma_n$ is increasing in $n$,
 \begin{align*}
\sqrt{ \sum_{h=0}^{\lfloor T/d \rfloor}\sum_{j=1}^K  \sum_{m=N_j(dh)+1}^{N_j(d(h+1))} \sigma_{j}^2(z^{(m)}; m-1)} &= \sqrt{\sum_{j=1}^K \sum_{m=1}^{N_j(T)} \sigma_{j}^2(z^{(m)}; m-1) }
 \\ & \leq \sqrt{\sum_{j=1}^K C_1 \gamma_{N_j(T)}} \leq \sqrt{C_1 K \gamma_T}
 \end{align*}
 
 for $C_1 = (1+ \log(\sigma^{-2}))^{-1}$. Hence,
\begin{align*}
\E[\Reg_T] \leq \frac{\pi^{5/2}}{\sqrt{2}3d}  + \sqrt{\frac{12d}{\sigma^2}\log((K|\cZ|)^d(T+d))} \sqrt{ T/d +1}  \sqrt{C_1K\gamma_T}
\end{align*}
and so the result follows.
\end{proof}

\section{Theoretical Results for $d$RGP-TS} \label{app:TS}
The regret bounds for the Thompson sampling approach($d$RGP-TS) follow in a similar manner to those for $d$RGP-UCB using the techniques of \cite{russo2014learning}. Specifically, using \cite{russo2014learning}, we get the following result which is equivalent to Lemma~\ref{lem:dreg}, and which can then be used to get the regret bound much in the same way as Theorem~\ref{thm:regrgpds} and Theorem~\ref{thm:regdm}.

\begin{lemma} \label{lem:dregts}
Assume we start a $d$-step look ahead policy at time $t$, selecting leaf node $I_t$, then
\begin{align*}
\sum_{s=t}^{t+d-1} \E[r_s |\cF_{t-1}] &\leq  \frac{ \sqrt{2d \pi}}{(t+d-1)^2} + \alpha_t \varsigma_t(I_t).  
\end{align*}
 \end{lemma}
 \begin{proof}
As in \cite{russo2014learning} we relate the Bayesian regret of Thompson sampling to the upper confidence bounds used in our upper confidence bound approach. Specifically, by Proposition 1 in \cite{russo2014learning},
\begin{align*}
\sum_{s=t}^{t+d-1} \E[r_s |\cF_{t-1}]  &=  \E[M_{I_t^*}(\bZ_t) - M_{I_t}(\bZ_t)| \cF_{t-1}]
\\ & = \E[M_{I_t^*}(\bZ_t) - \eta_t(I_t^*) - \alpha_t \varsigma_t(I_t^*) | \cF_{t-1}] + \E[\eta_t(I_t) +\alpha_t \varsigma_t(I_t)  - M_{I_t}(\bZ_t)| \cF_{t-1}]
\end{align*}
The same argument as Lemma~\ref{lem:dreg} then gives the result.

 \end{proof}
 
 \subsection{Non-Repeating} \label{app:drgptss}
 \thmregdsts*
\begin{proof}
Given Lemma~\ref{lem:dregts}, the proof follows in the same manner as the proof of Theorem~\ref{thm:regrgpds}.
\end{proof}
 
\subsection{Repeating} \label{app:drgptsm}
\thmregtsdm*
\begin{proof}
Again, the proof follows by the same argument as Theorem~\ref{thm:regdm} using Lemma~\ref{lem:dregts}.
\end{proof}

\section{Optimality of the Lookahead Oracle} \label{app:lookahead}
For any policy $\pi$, let $V_T(\pi)$ denote the expected cumulative reward from playing policy $\pi$ up to horizon $T$. We say a policy $\pi$ is periodic with period $p \in \bN$ from some initial $\bz_1$ if there is some $t_0>0$ such that for all $t>t_0$, $\bz_t^{\pi} = \bz_{t+p}^{\pi}$ and $j_t^{\pi} = j_{t+p^*}^{\pi}$, where $j_t^{\pi}$ is the action taken at time $t$ by policy $\pi$ and $\bz^{\pi}_t$ is the vector of $z$ values obtained at time $t$ from playing according to policy $\pi$ for $t-1$ steps.
For a periodic policy $\pi$ and initial $\bz_1$, we will assume that $p \geq t_0$.

\begin{restatable}{lemma}{lemperiod}
\label{lem:period}
If $\pi^*$ is an optimal stationary deterministic policy then if $T>|\cZ|^K$, then $\pi^*$ must be periodic with some period $p^* \leq |\cZ|^K$.
\end{restatable}
\begin{proof}
The proof follows by noting that $\pi^*$ must be a deterministic mapping from $\bz$ to actions since a stationary policy does not depend on the time step. In particular, $\pi^*: \cZ^K \to \{1, \dots, K\}$ with $\pi^*(\bz) = j$ for some $1 \leq j \leq K$, and
each $\bz$ corresponds to only one action.
We now argue for a contradiction. Assume that $\pi^*$ is not periodic. Then since $T>|\cZ|^K$, there must exist some $\bz$ which is arises twice, so there exists some $t$ and $0<p^* \leq |\cZ|^K$ such that $\bz_t = \bz_{t+p^*} = \bz$. Since $\pi^*$ is a deterministic mapping, the same action must be taken in both cases, which will lead to the same next value of $\bz'=\bz_{t+1}=\bz_{t+p^* +1}$, since the evolution of the $\bz$ is deterministic conditional on actions. Repeatedly applying same argument, we see that $\pi^*$ will take the same sequence of $p^*$ actions from $\bz$ in both cases before returning to $\bz$ (if the horizon is long enough). Hence $\pi^*$ must be periodic, contradicting the assumption.
\end{proof}

\proplookahead*
\begin{proof}
Define a vector $\bz=(z_1, \dots, z_K)$ as feasible for the recovering bandits problem starting from $\bz_0$ with $K$ arms and a fixed value of $z_{\max}$, if it is possible to play a sequence of arms up to any time $t\geq 1$ such that $\bz_t=\bz$.
We begin by observing that it is possible to get from any feasible $\bz=(z_1, \dots, z_K)$ to any other feasible $\bz' = (z_1', \dots, z_K')$ in at most $z_{\max}$ steps. For this, we need the following properties of $\bz$ that are consequences of the update procedure in equation~\eqref{eqn:zdef}. Equation~\eqref{eqn:zdef} guarantees that there must be exactly one element of $\bz$ equal to 0, and if$z_i, z_j \neq z_{\max}$, then $z_i \neq z_j$ for $i \neq j$. For the target vector $\bz'$, let $n$ be the number of elements with value $z_{\max}$. The remaining $K-n$ entries must all be unique and one must be 0, denote the index of this $i_0$. In the following $z_{\max}$ steps, we play each arm corresponding to $z_i \neq z_{\max}$ at step $z_{\max} -z_i$ and play $i_0$ in the intervening steps, and at step $z_{\max}$. It is clear to see that this procedure will go from $\bz$ to $\bz'$ in $z_{\max}$ steps.

 Let $v^*$ be the reward achieved in $p^*$ steps of the optimal policy $\pi^*$. By the above argument, from any initial state of the lookahead $\bZ_t$, it is possible to get to any other (feasible) $\bz$ in at most $z_{\max}$ steps. In particular, it is possible to get to one of the elements $\bz^{(1)}, \dots, \bz^{(p^*)}$ of the optimal periodic policy in $z_{\max}$ steps. Hence, the policy that chooses the quickest route to the optimal periodic policy and then plays that policy for $lp^*$ steps is a valid $(z_{\max} + lp^*)$-lookahead policy. This policy will achieve reward of at least $lv^*$ over this period. Consequently, the optimal $(z_{\max} + lp^*)$-lookahead policy, $\pi^*_l$ will achieve reward of at least $lv^*$ every $(z_{\max} + lp^*)$ steps. We select a lookahead policy every $(z_{\max} + lp^*)$ steps, therefore the total reward of $\pi^*_l$ must be at least 
 $\lfloor \frac{T}{lp^*+z_{\max}} \rfloor lv^*$. The total reward of $\pi^*$ is less than $\lceil \frac{T}{p^*} \rceil v^*$. Therefore,
 \[ \frac{V_T(\pi^*_l)}{V_T(\pi^*)} \geq \frac{\lfloor\frac{T}{lp^*+z_{\max}}\rfloor lv^*}{\lceil \frac{T}{p^*} \rceil v^*} \geq \frac{\frac{T}{lp^*+z_{\max}}-1 }{\frac{T}{p^*} +1}l=\bigg( 1-\frac{(l+1)p^* + z_{\max}}{T+p^*}\bigg) \frac{lp^*}{lp^*+z_{\max}}. \]
 This gives the result.
\end{proof}

\section{Theoretical Guarantees on Optimistic Planning Procedure} \label{app:opplan}
\lemopplanning*
\begin{proof}
Since our $\tilde f_j(z)$'s are samples from a Gaussian posterior, they can be negative. Hence it will be convenient to work with a transformation that guarantees positivity. To this end, let $\delta = - \min_{j,z} \tilde f_j(z)$ if $\min_{j,z} \tilde f_j(z)<0$ and $\delta=0$ if $\min_{j,z} \tilde f_j(z) \geq 0$ and for any arm $j$ and covariate $z$, define,
\[ \tilde f'_j(z) = \tilde f_j(z) + \delta \geq 0.  \]
Then we define the corresponding $v$, $b$ and $u$ values of any node $i \in \mathcal{S}_n$ at step $n$ and $\Psi$ functions as,
\begin{align*}
 v'(i) = v(i) + d \delta \qquad b_n'(i) = b_n(i) + d \delta \qquad  u'(i) = u(i) + l(i) \delta 
 \\  \Psi'(\mathbf{z}(i), d-l(i)) =  \Psi(\mathbf{z}(i), d-l(i)) + (d-l(i)) \delta \qquad \Psi'^*(l) = \Psi^*(l) + l\delta, 
 \end{align*}
where $l(i)$ is the depth of node $i$. Note that node $i^*$ maximizing $v(i)$ will also maximize $v'(i)$ and that if at step $n$ we select a node maximizing $b_n(i)$ this will also be the node maximizing $b'_n(i)$ and so $v(i_1) \geq v(i_2) \iff v'(i_1) \geq v'(i_2)$ and $b(i_1) \geq b(i_2) \iff b'(i_1) \geq b'(i_2)$ for all nodes $i_1,i_2$. Furthermore, it holds that $v'(i) \geq u'(i)$ and that $b'(i)$ is an upper bound on $v'(i)$ for all nodes $i$ and in particular $b'(i) = u'(i) + \Psi'(\mathbf{z}(i), d-l(i)) $. 

We begin with the case where the algorithm is stopped after some number $n$ of nodes have been expanded because the selected node is of depth $d$. Let $i^*_1, \dots, i^*_d$ be the nodes on the path to the optimal node $i^*$ and let $j$ be the maximal depth of this path in $\mathcal{T}_n \cup \mathcal{S}_n$. If $i_n$ is the node at depth $d$ selected to be expanded at time $n$, then,
\[ 0 \leq v^*-v(i_n) = v'(i^*_j) - v'(i_n) \leq b'(i_j^*) - v'(i_n) \leq b'(i_n) - v'(i_n) = \Psi'(\mathbf{z}(i_n), d-d) = 0, \]
since we select node $i_n$ at time $n$ so it must have the largest $b_n(i)$ and $b'_n(i)$ value.
This proves the first statement.

For the other case, define the set 
\[ \Gamma = \bigcup_{l=0}^d \{ \text{ node } i \text{ of depth } l \text{ such that } v^*-v(i) \leq \Psi'^*(d-l) \}, \]
and note that if $v^*-v(i) \leq \Psi'^*(d-l) $ then also $v'^*-v'(i) \leq \Psi'^*(d-l) $.
As in \cite{hren2008optimistic}, we will show that all nodes expanded by our algorithm are in $\Gamma$. For this, let node $i$ of depth $l$ be chosen to be expanded at time $n$. This means it has the largest $b_n(i)$ (and $b_n'(i)$) value of all nodes in $\mathcal{S}_n$.
We also now need to define the $b$ value of a node in $\cT_n$ as $b_n(i) = \max_{j \in C(i)} b_n(j)$ where $C(i)$ is the set of all children of node $i$, and we define $b'_n(i)$ correspondingly. This definition together with the previous remark means that for any $j \in \mathcal{T}_n$, $b'_n(i) \geq b'_n(j)$. Then for some $1 \leq j \leq d$, $i^*_j \in \mathcal{T}_n$, so it follows that $b'_n(i^*_j) \leq b'_n(i_n)$. But, the best value of any continuation of a path to the optimal node is simply $v^*$ and so by definition of the $b$ values $b'_n(i_j^*) \geq v'(i_j^*) = v'^*$. Hence, since $v'(i) \geq u'(i)$ and $\Psi'(\mathbf{z}(i),d-l) \leq \Psi'^*(d-l)$,,
\[ v'(i) \geq u'(i) = b'_n(i) - \Psi'(\mathbf{z}(i),d-l) \geq b'_n(i_j^*) -  \Psi'(\mathbf{z}(i),d-l)  \geq v'^* -  \Psi'(\mathbf{z}(i),d-l)  \geq v'^* - \Psi'^*(d-l), \]
it follows that $i \in \Gamma$. 
Then, we bound from below the maximal depth at which a node is chosen to be expanded. Let $n_0$ be the number of policies in $\Gamma$ up to depth $d_0$ and let $d_N$ be the maximal depth of any node expanded before the algorithm is stopped at time $N$. By the assumption in the proposition, the proportion of $(d-l)\Delta$-optimal nodes at depth $l$ is bounded by $\lambda^l$. Then, $\Psi'^*(d-l) = \Psi(d-l) + (d-l)\delta  \leq (d-l) \max_{j,z} \tilde f_j(z) - (d-l) \min_{j,z} \tilde f_j(z) = (d-l)\Delta$ by definition of $\Psi$ and so $p_l(\Psi'^*(d-l)) \leq p_l((d-l)\Delta) \leq \lambda^l$. Hence,
\[ N \leq n_0 + \sum_{l=d_0}^{d_N} \lambda^l K^l = n_0 + \sum_{l=d_0}^{d_N} A^l \leq n_0 + A^{d_0+1} \frac{A^{d_N-d_0}-1}{A-1} \]
for $A=\lambda K>1$. Rearranging gives,
\begin{align*}
d_N &\geq d_0 + \log_A \bigg( \frac{(N-n_0)(A-1)}{A^{d_0+1}} +1  \bigg) \geq  d_0 + \log_A \bigg( \frac{(N-n_0)(A-1)}{A^{d_0+1}}  \bigg) \\ &\geq  \frac{\log(N-n_0)}{\log(K\lambda)} -1 + \frac{\log(\lambda K-1)}{\log(\lambda K)}
\end{align*}
Let $i_N$ be the node the algorithm outputs at step $N$ when the computational resources have been exceeded and note that this is the node in $\cT_N$ with largest depth (i.e. $l(i_N) = d_N$) that has the largest $b_N$ (or $b_N'$) value. Since $i_N \in \cT_N$, there is some step $n \leq N$ when node $i_N$ was expanded.
Then, let $j$ be the maximal depth of nodes on the path $i_1^*, \dots, i_d^*$ in $\mathcal{S}_n$. It then follows that
\[ v'^* - v'(i_N) \leq b'_n(i^*_j) - v'(i_N) \leq b'_n(i_N) - v(i_N) \leq \Psi'(\mathbf{z}(i_N), d-l(i_N)) \leq \Psi'^*(d-d_N). \]
Hence,
\begin{align*}
v^*-v(i_N)  &= v'^* - v'(i_N) \leq \Psi'^*(d-d_N) = \Psi^*(d-d_N) + (d-d_N)\delta 
\\ &\leq (d-d_N)( \max_{j,z} \tilde f_j(z) - \min_{j,z}\tilde f_j(z)) \leq \bigg( d- \frac{\log(N-n_0)}{\log(K \lambda)} - \frac{\log(\lambda K-1)}{\log(\lambda K)} +1 \bigg) \Delta
\end{align*}
which gives the result.

\end{proof}

\section{Regret Bounds for Non-Parametric Approach} \label{app:nonparam}
We use an algorithm which has no information about the recovery structure as a baseline. For this, we model each (arm, $z$) pair as an arm. This reduces the problem to a standard multi-armed bandit problem with $K|\cZ|$ arms, where only some arms are available each round. 

Let $\mu_{j,z}$ denote the expected reward of arm $j$ when $z_j=z$. We can then create estimates $\bar{Y}_{j,z,t}$ of the reward of each arm from the $N_{j,z}(t)$ samples of arm $j$ with $Z_j=z$ we receive up to time $t$. These estimates can be used to define an upper confidence bound style algorithm over the `arms' $\{(j,z)\}_{j=1, z=0}^{K,Z_{\max}}$. We define confidence bound based on UCB1 \citep{auer2002finite} and \cite{russo2014learning}
\[  U({j,z,t}) = \bar{Y}_{z,j,t} + \sqrt{\frac{\sigma^2(2+6\log(T))}{N_{j,z}(t)}}. \]
where $\sigma$ is the standard error of the noise. After playing each $j,z$ combinations once, we proceed to play the arm with largest $U(j,Z_{j,t},t)$ at time $t$. We now bound the Bayesian regret of this algorithm to horizon $T$.
\begin{restatable}{theorem}{regrucb}
\label{thm:regrucb}
The instantaneous regret up to time $T$ of the UCB1 algorithm with $K|\cZ|$ arms can be bounded by
\[ \E[ \Reg_T^{(1)}] \leq  O(\sqrt{K|\cZ|T \log(T)} + K|\cZ|^2) \]
\end{restatable}
\begin{proof}
We first consider the initialization phase. For this, note that in order to play arm $j$ at $Z_j=z$, we need to wait $z$ rounds from when it was last played. This means that the total number of plays required to play each arm at each $z$ value can be bounded by
 $t_0 = K|\cZ|(|\cZ|+1)$ (since in the worst case, for arm $j$, we need to wait, 1 round, then 2 rounds, up to $|\cZ|$ rounds).
 We can bound the per-step regret from this initialization period using Lemma~\ref{lem:maxgp}.
For any $1\leq t \leq t_0$, 
\[ \E[ f_{J_t^*}(Z_{J_t^*,t}) - f_{J_t}(Z_{J_t,t})] \leq \E[ \max_{1 \leq t \leq t_0} \{ f_{J_t^*}(Z_{J_t^*,t}) - f_{J_t}(Z_{J_t,t})\} ]\leq 2 \sqrt{2 \log(t_0)} \]
since the distribution of the difference of two zero mean Gaussian random variables with variances $\sigma_1^2, \sigma_2^2 \leq 1$ is also a Gaussian random variable with mean 0 and variance $\sigma^2_1+\sigma^2_2 \leq 2$ here. Then, we can use a similar technique to \cite{russo2014learning} to bound the cumulative regret in the remaining $t_0 +1 \leq t \leq T$ steps but using  Lemma~\ref{lem:maxgp} again to bound the maximal difference in $f_j$'s.
\begin{align*}
\E[\Reg_T] &= \sum_{t=t_0}^T \E[f_{J_t^*}(Z_{J_t^*,t}) - f_{J_t}(Z_{J_t,t}) \one \{\forall j,z; f_j(z) \in [L(j,z,t), U(j,z,t)] \} ] 
	\\ & \hspace{50pt} +  \sum_{t=t_0}^T \E[f_{J_t^*}(Z_{J_t^*,t}) - f_{J_t}(Z_{J_t,t}) \one \{\exists j,z; f_j(z) \notin [L(j,z,t), U(j,z,t)] \} ]
\\& \leq \sum_{t=t_0}^T \E[ U(J_t^*,Z_{J_t^*,t},t) - L(J_t,Z_{J_t,t},t)] + 2\sqrt{2\log(T)} T \PP(\exists j,z;  f_j(z) \notin [L(j,z,t), U(j,z,t)] )
\\ & \leq \sum_{t=t_0}^T  \E[ U(J_t,Z_{J_t,t},t) - L(J_t,Z_{J_t,t},t)] + 2\sqrt{2\log(T)} T\sum_{j=1}^K \sum_{z \in \cZ}  \PP( f_j(z) \notin [L(j,z,t), U(j,z,t)] )
\end{align*}
Since $\epsilon_t \sim \cN(0,\sigma^2)$, by Lemma 1 in \cite{russo2014learning},
\[ 2\sqrt{2\log(T)} T\sum_{j=1}^K \sum_{z \in \cZ} \PP( f_j(z) \notin [L(j,z,t), U(j,z,t)] ) \leq \frac{ 2\sqrt{2\log(T)} T |\cZ| K}{T} \leq 2 K|\cZ| \sqrt{2\log(T)}. \]
Then, for the first term, by the same argument as \cite{russo2014learning},
\begin{align*}
\sum_{t=t_0}^T  \E[ U(J_t,Z_{J_t,t},t) - L(J_t,Z_{J_t,t},t)] &\leq \sum_{t=t_0}^T \sum_{j=1}^K \sum_{z \in \cZ} \E[ U(j,z,t) - L(j,z,t)\one\{J_t=j,Z_{J_t,t}=z\} ] 
\\ & \leq 2\sqrt{\sigma^2(2 + 6 \log(T))}  \sum_{t=t_0}^T \sum_{j=1}^K \sum_{z \in \cZ} \E\bigg[ \frac{1}{\sqrt{2N_{j,z}(t)}} \one\{J_t=j,Z_{J_t,t}=z\} \bigg]  
\\ &\leq 2\sqrt{\sigma^2(2 + 6 \log(T))}  \sum_{j=1}^K \sum_{z \in \cZ}  \E\bigg[ \sum_{l=0}^{N_{j,z}(T)-1}\frac{1}{\sqrt{l+1}} \bigg] 
\\ &\leq 2\sqrt{\sigma^2(2 + 6 \log(T))}  \sum_{j=1}^K \sum_{z \in \cZ}  \E\bigg[\sqrt{N_{j,z}(T)} \bigg] 
\\ &\leq 2\sqrt{\sigma^2(2 + 6 \log(T))}\sqrt{K |\cZ|T}
\end{align*}
where the last line follows by Cauchy-Schwartz. This concludes the proof.
\end{proof}

\section{Further Experimental Results}

\subsection{Posterior Distributions and Covariates} \label{app:tsplot}
\subsubsection{$d$RGP-UCB}
In this section, we plot the posterior (blue) of $d$RGP-UCB. with density given by the blue region in the instantaneous case. The red curve is the true recovery curve and the crosses are our observed samples for various values of $d$ and different kernels. Note that as the kernel gets smoother, the algorithm places more samples in the good regions. This is to be expected as for smoother kernels, there is less need to explore as many sub-optimal regions. Also, as $d$ increases more samples are at the peaks of the recovery curves.
\begin{figure}[h]
	\centering
	\begin{subfigure}{0.8\textwidth}
		  \centering
		 \includegraphics[width=0.9\textwidth]{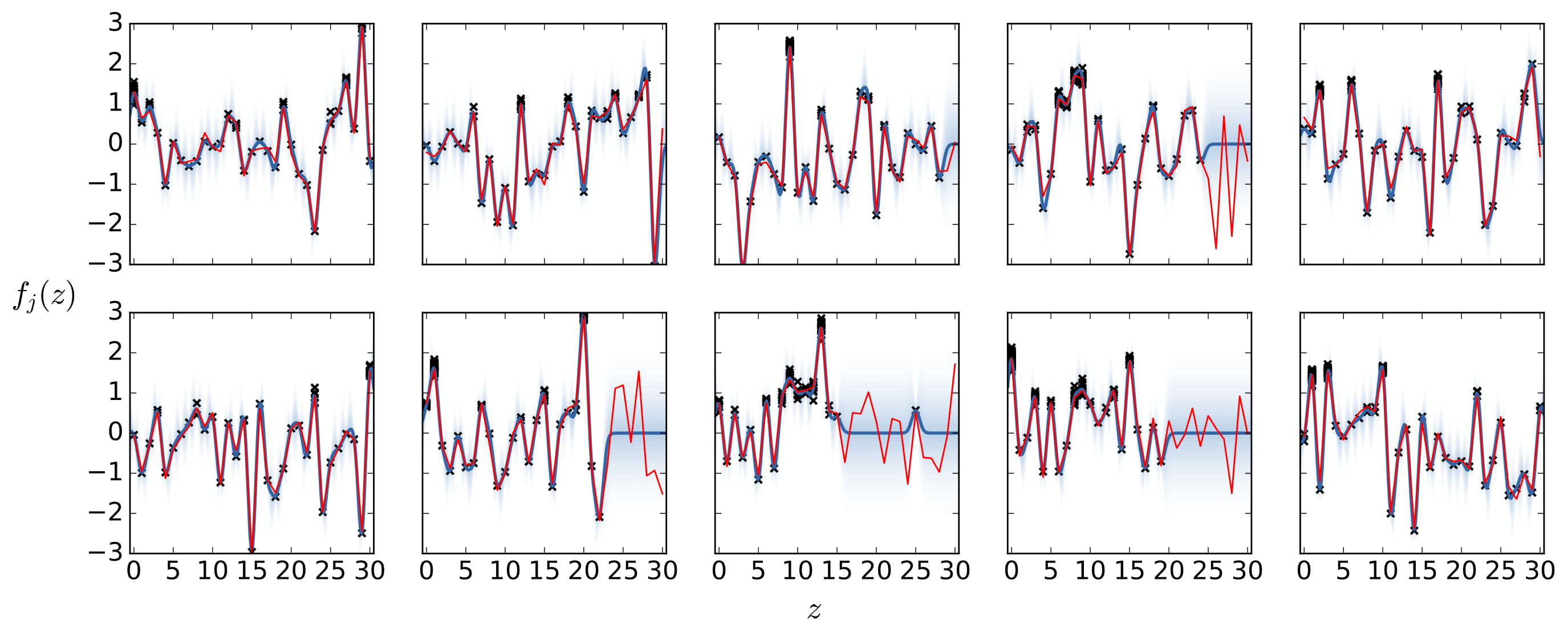}
		 \caption{$d=1$}
	\end{subfigure}
	\par\bigskip
	\begin{subfigure}{0.8\textwidth}
		  \centering
		 \includegraphics[width=0.9\textwidth]{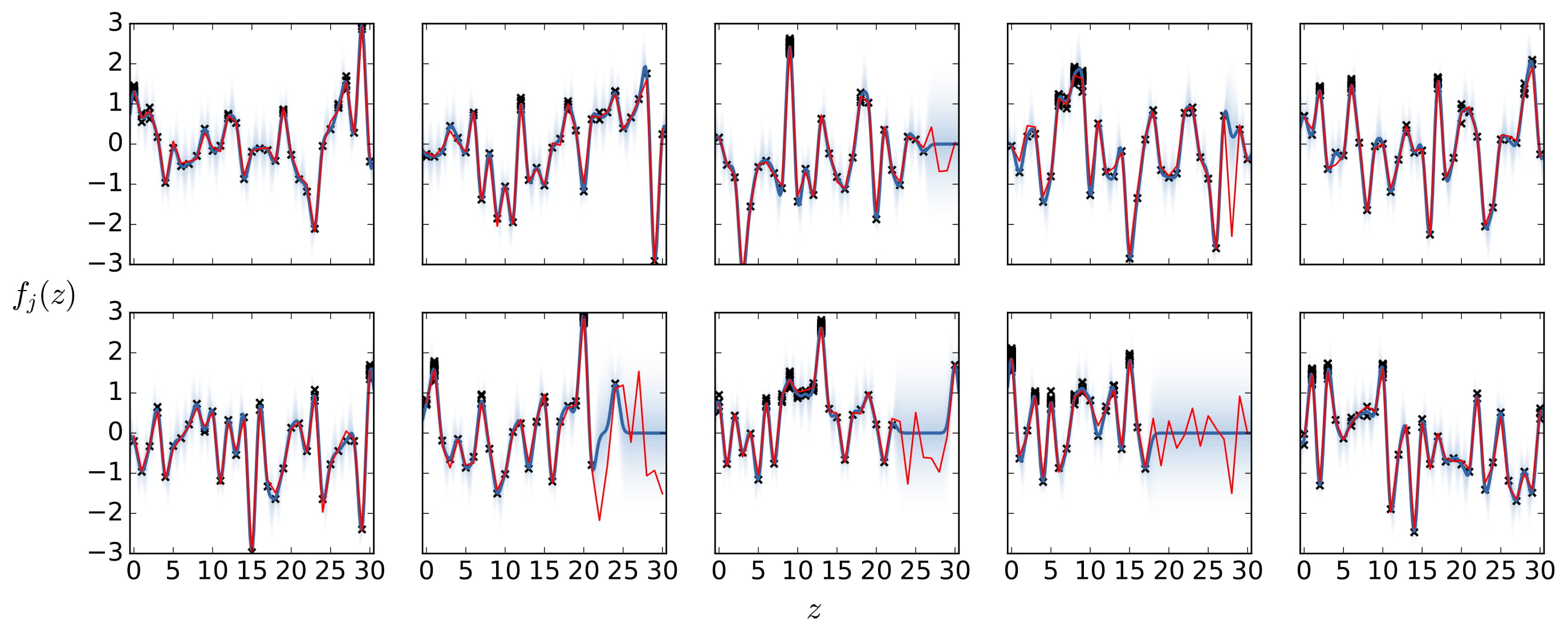}
		 \caption{$d=2$}
	\end{subfigure}
	\par\bigskip
	\begin{subfigure}{0.8\textwidth}
		  \centering
		 \includegraphics[width=0.9\textwidth]{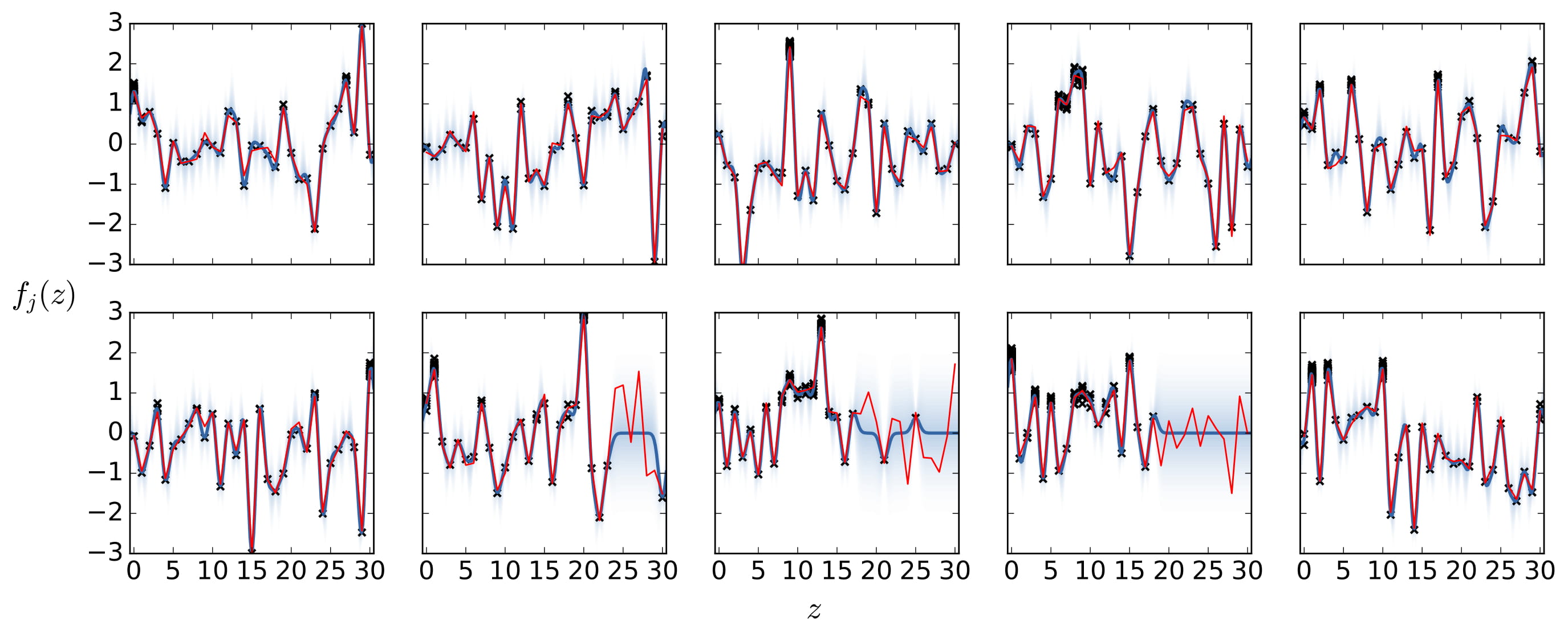} 
		 \caption{$d=3$}
	\end{subfigure}
      \caption{$d$RGP-UCB with squared exponential kernel with $l=0.5$}
      \label{fig:gpucbpost_l05}
\end{figure}

\begin{figure}[h]
	\centering
	\begin{subfigure}{0.8\textwidth}
		  \centering
		 \includegraphics[width=0.9\textwidth]{figures/posteriors/posterior_gp_ucb_l2_oct19}
		 \caption{$d=1$}
	\end{subfigure}
	\par\bigskip
	\begin{subfigure}{0.8\textwidth}
		  \centering
		 \includegraphics[width=0.9\textwidth]{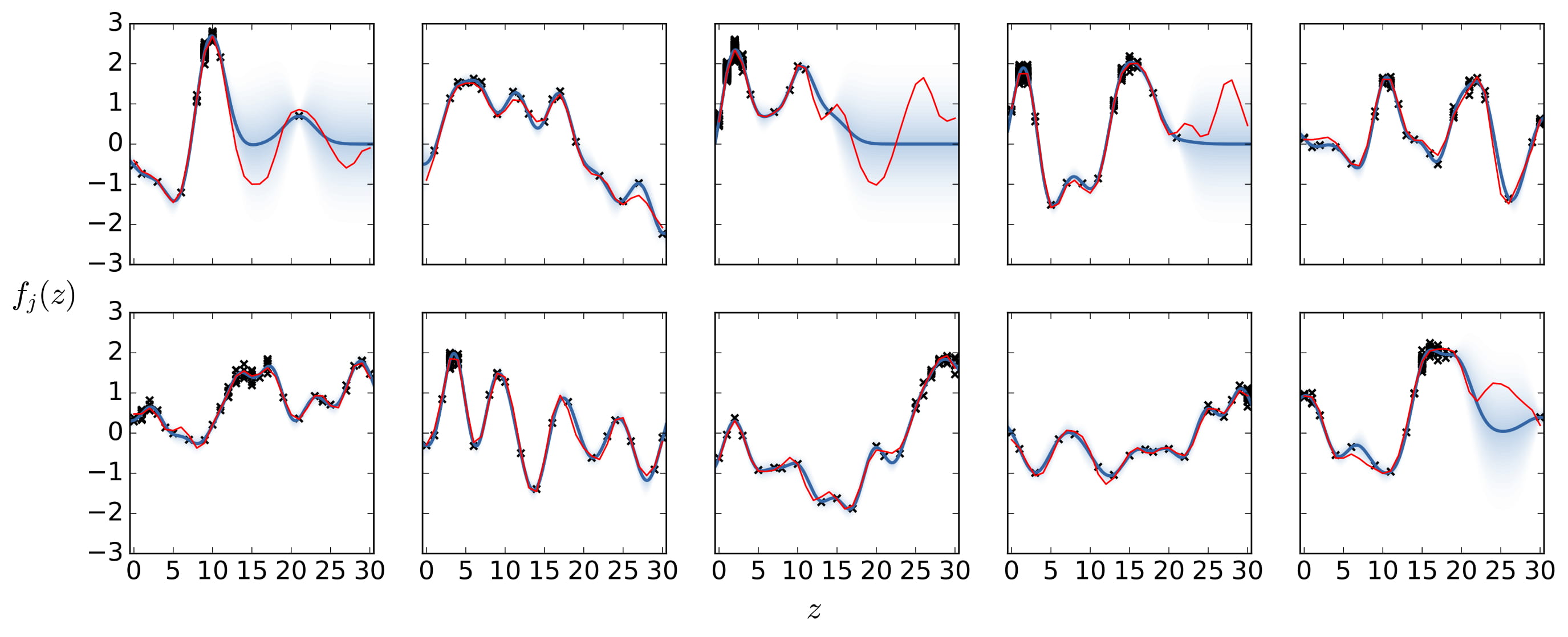}
		 \caption{$d=2$}
	\end{subfigure}
	\par\bigskip
	\begin{subfigure}{0.8\textwidth}
		  \centering
		 \includegraphics[width=0.9\textwidth]{figures/posteriors/posterior_gp_ucb_l2_d3_oct19}
		 \caption{$d=3$}
	\end{subfigure}
      \caption{$d$RGP-UCB with squared exponential kernel with $l=2$}
      \label{fig:gpucbpost_l2}
\end{figure}

\begin{figure}[h]
	\centering
	\begin{subfigure}{0.8\textwidth}
		  \centering
		 \includegraphics[width=0.9\textwidth]{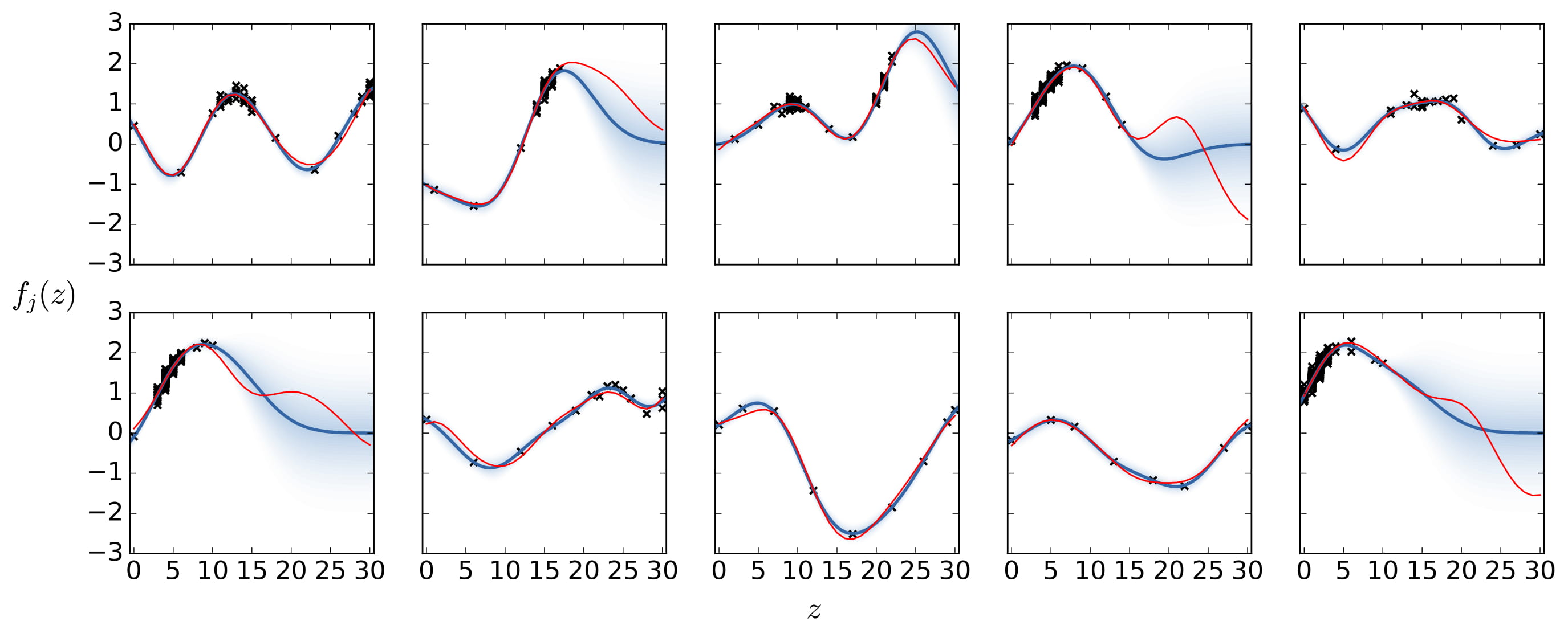}
		 \caption{$d=1$}
	\end{subfigure}
	\par\bigskip
	\begin{subfigure}{0.8\textwidth}
		  \centering
		 \includegraphics[width=0.9\textwidth]{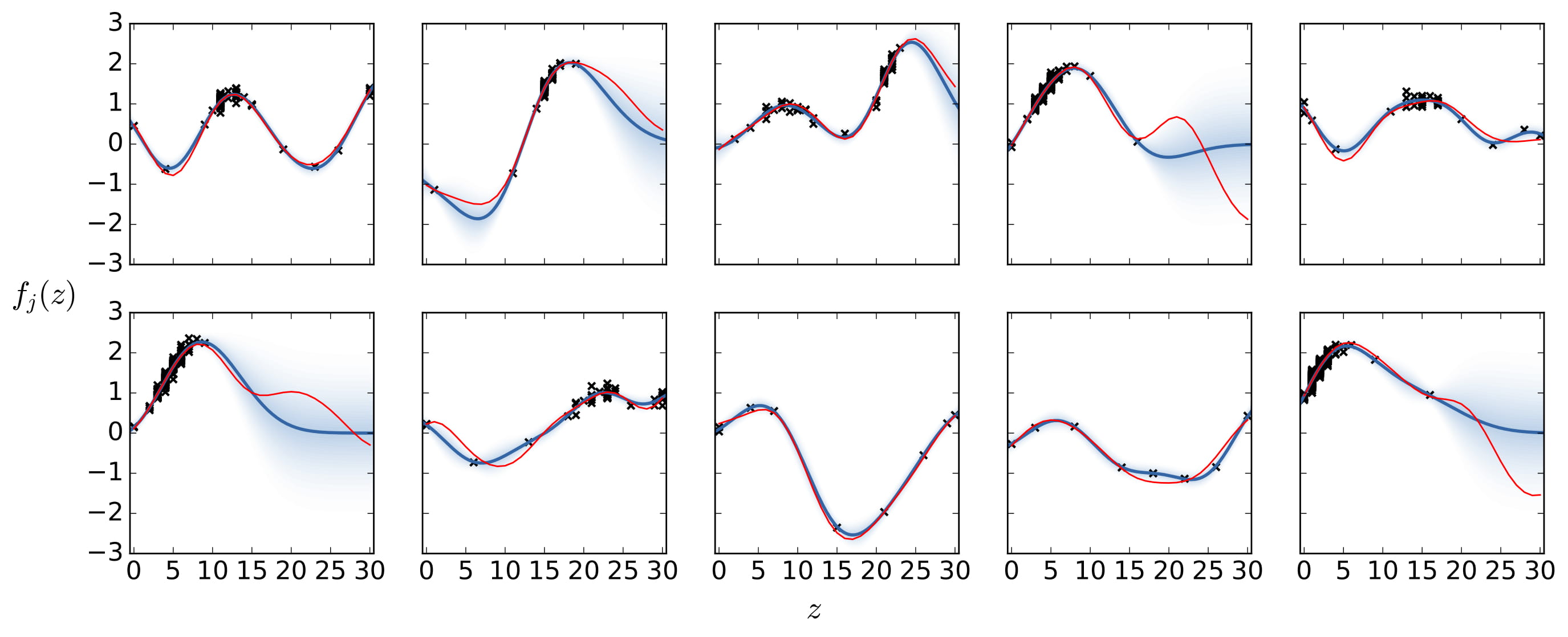}
		 \caption{$d=2$}
	\end{subfigure}
	\par\bigskip
	\begin{subfigure}{0.8\textwidth}
		  \centering
		 \includegraphics[width=0.9\textwidth]{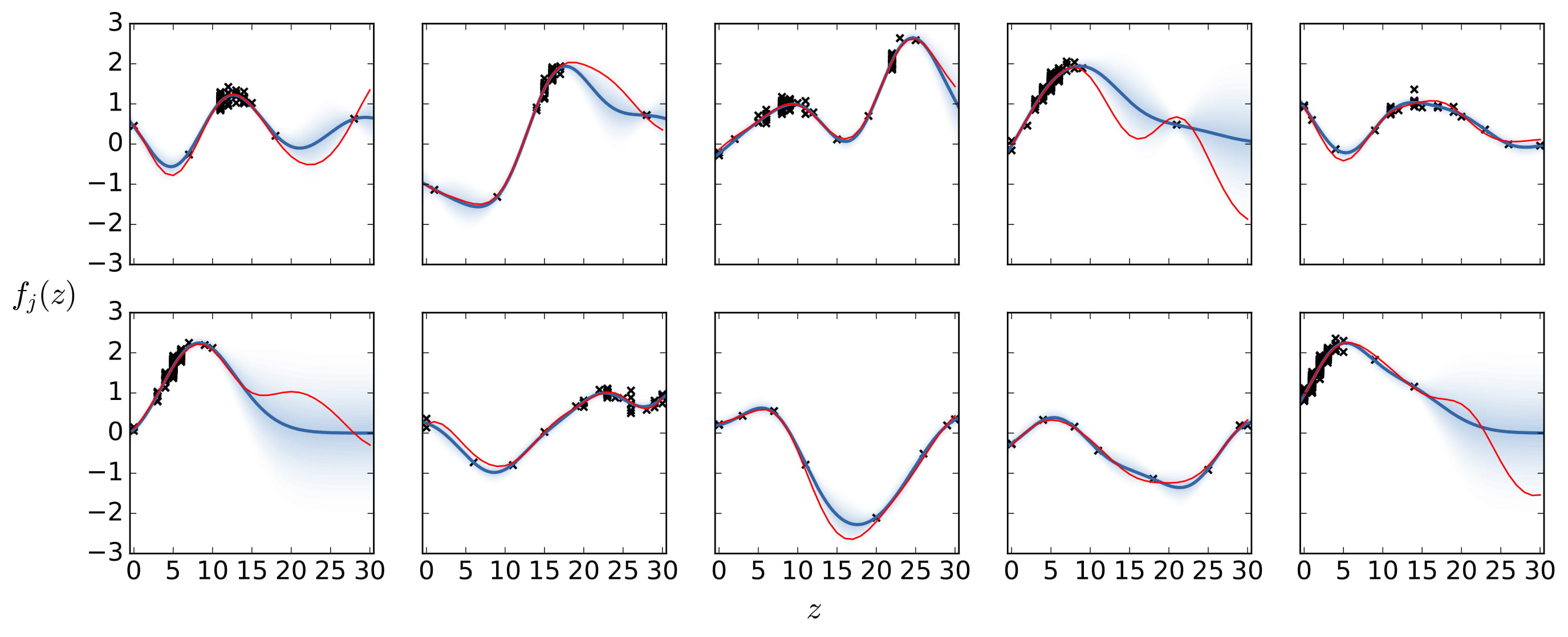}
		 \caption{$d=3$}
	\end{subfigure}
      \caption{$d$RGP-UCB with squared exponential kernel with $l=5$}
      \label{fig:gpucbpost_l5}
\end{figure}
\FloatBarrier

\subsubsection{$d$RGP-TS}
In this section, we plot the posterior (blue) of $d$RGP-TS. with density given by the blue region with different $l$'s and $d$'s. We see much the same pattern as for $d$RGP-UCB, although it does seem to demonstrate poorer estimation of the recovery curve in the single step case. This suggests that the Thompson sampling approach is focusing on exploitation rather than exploration, as has been observed in other settings (eg. in linear bandits \citep{abeille2017linear} show that the variance of the posterior needs to be inflated to encourage more exploration in Thompson Sampling). However, it is worth noting that the algorithms have only been run once for these plots.
\begin{figure}[hb]
	\centering
	\begin{subfigure}{0.8\textwidth}
		  \centering
		 \includegraphics[width=0.9\textwidth]{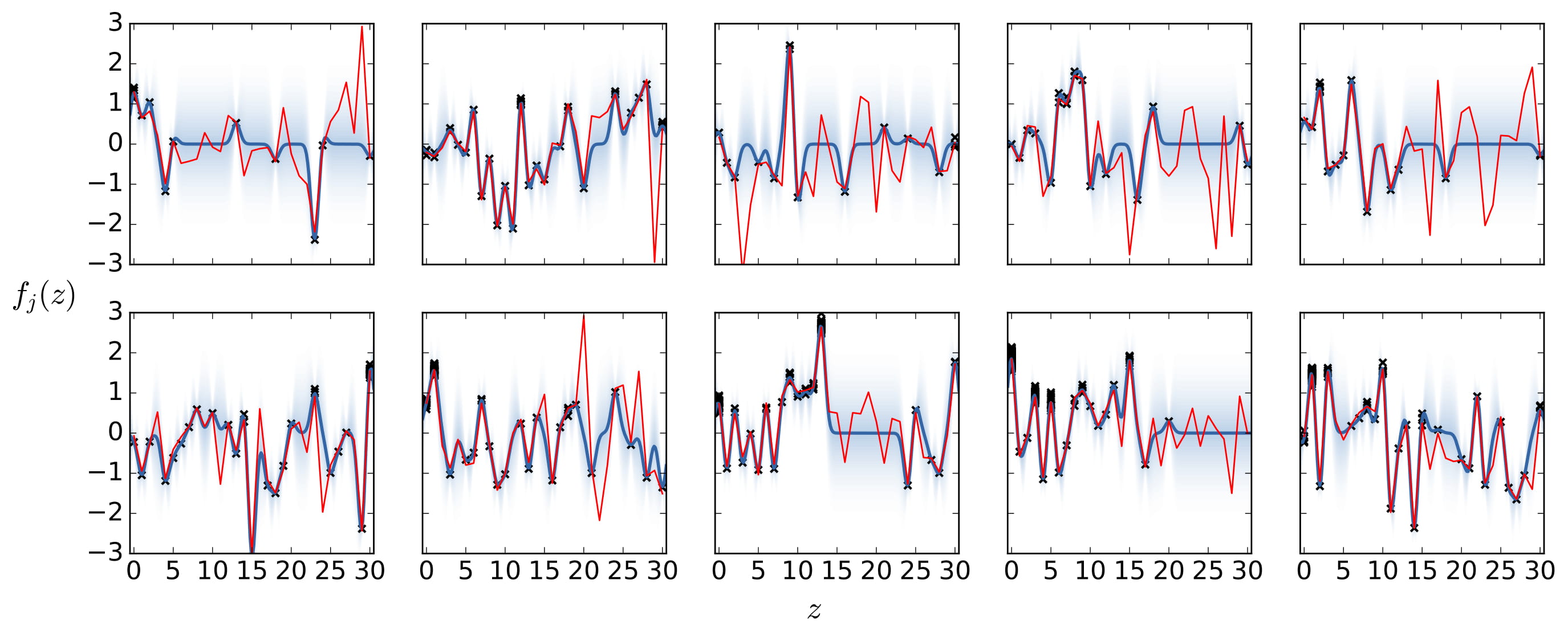}
		 \caption{$d=1$}
	\end{subfigure}
	\par\bigskip
	\begin{subfigure}{0.8\textwidth}
		  \centering
		 \includegraphics[width=0.9\textwidth]{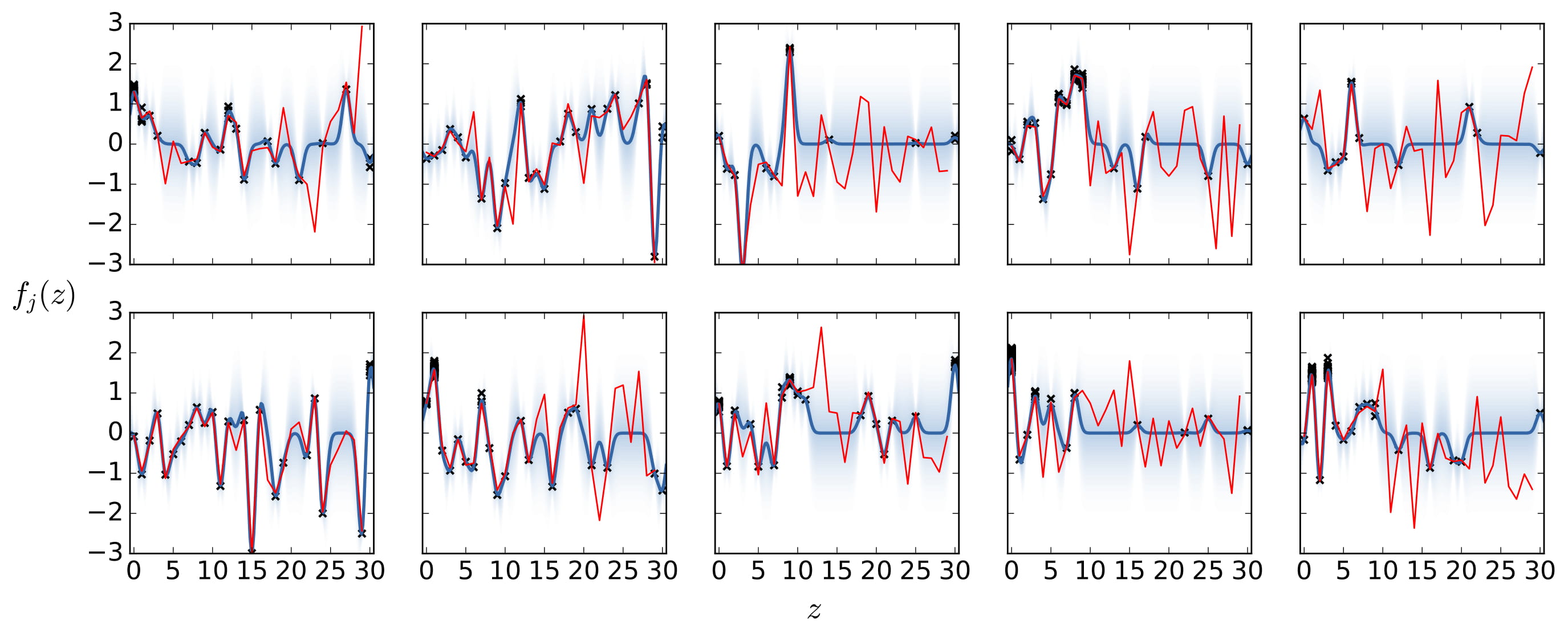}
		 \caption{$d=2$}
	\end{subfigure}
	\par\bigskip
	\begin{subfigure}{0.8\textwidth}
		  \centering
		 \includegraphics[width=0.9\textwidth]{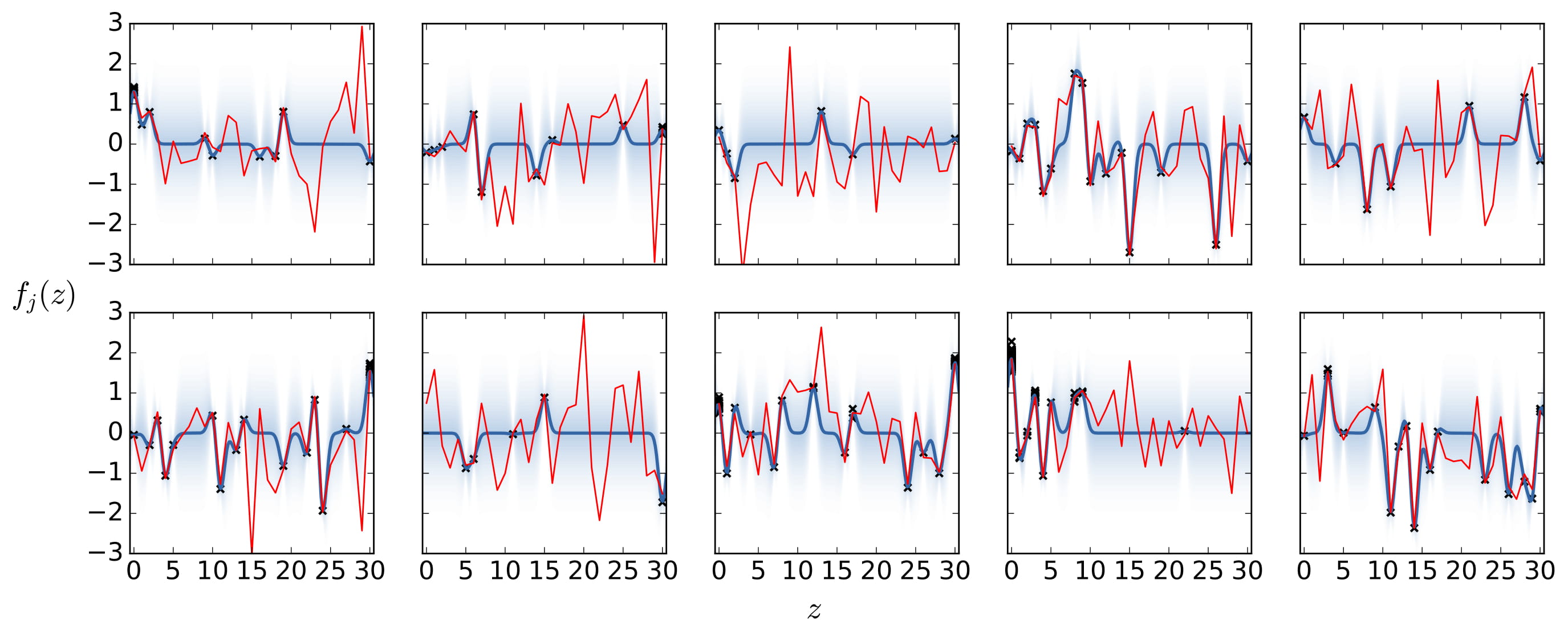}
		 \caption{$d=3$}
	\end{subfigure}
      \caption{$d$RGP-TS for squared exponential kernel with $l=0.5$}
      \label{fig:gptspost_l05}
\end{figure}

\begin{figure}[h]
	\centering
	\begin{subfigure}{0.8\textwidth}
		  \centering
		 \includegraphics[width=0.9\textwidth]{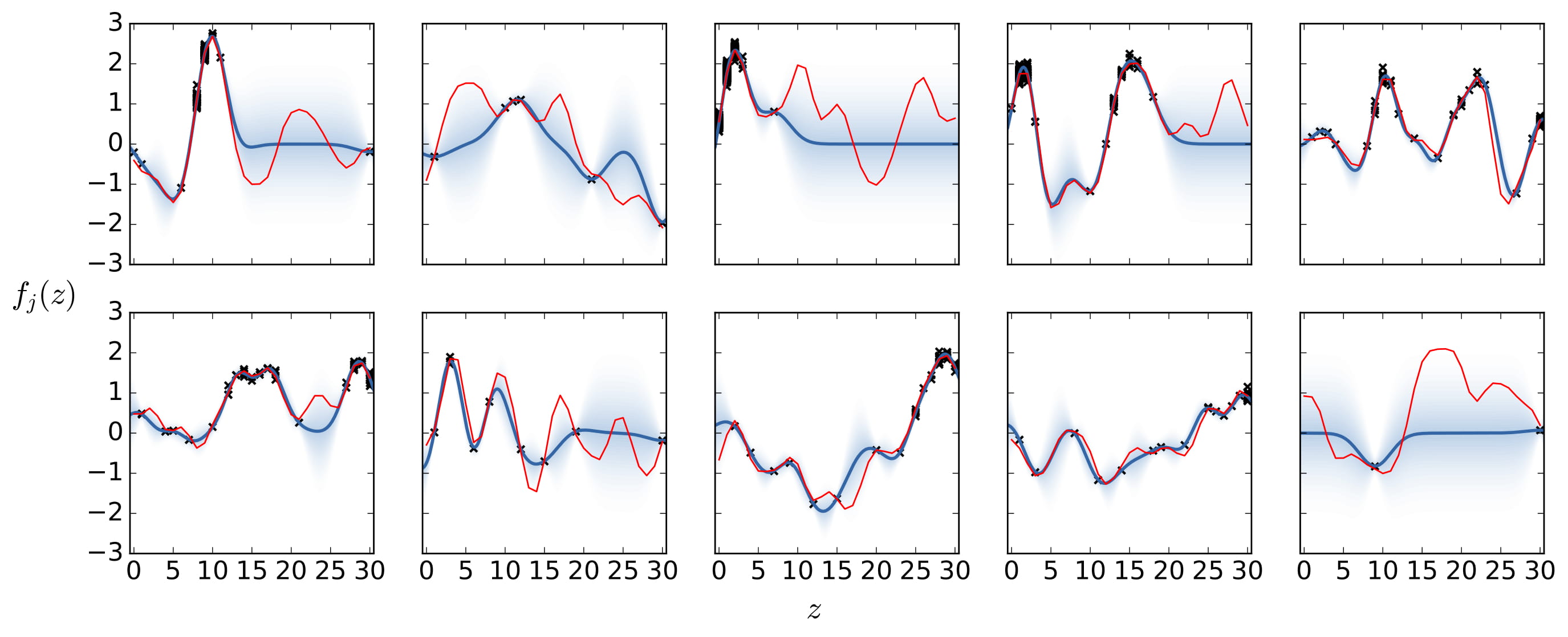}
		 \caption{$d=1$}
	\end{subfigure}
	\par\bigskip
	\begin{subfigure}{0.8\textwidth}
		  \centering
		 \includegraphics[width=0.9\textwidth]{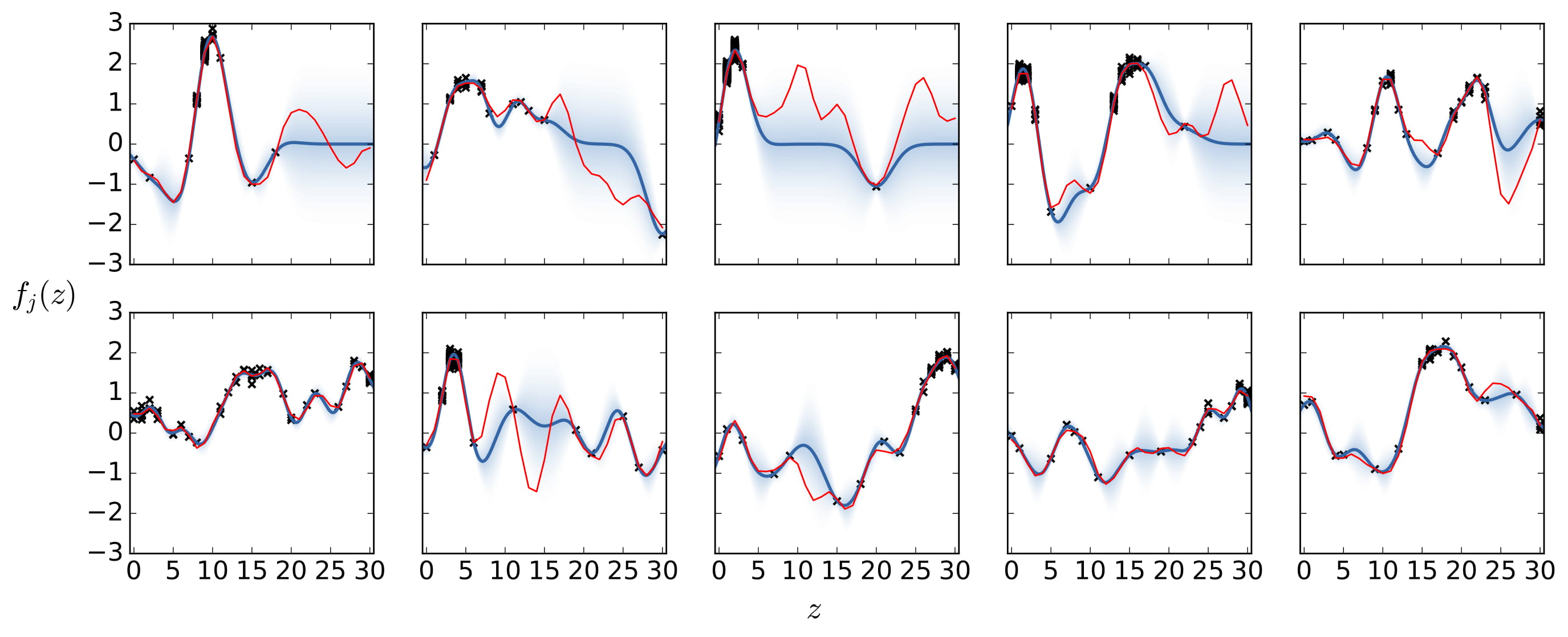}
		 \caption{$d=2$}
	\end{subfigure}
	\par\bigskip
	\begin{subfigure}{0.8\textwidth}
		  \centering
		 \includegraphics[width=0.9\textwidth]{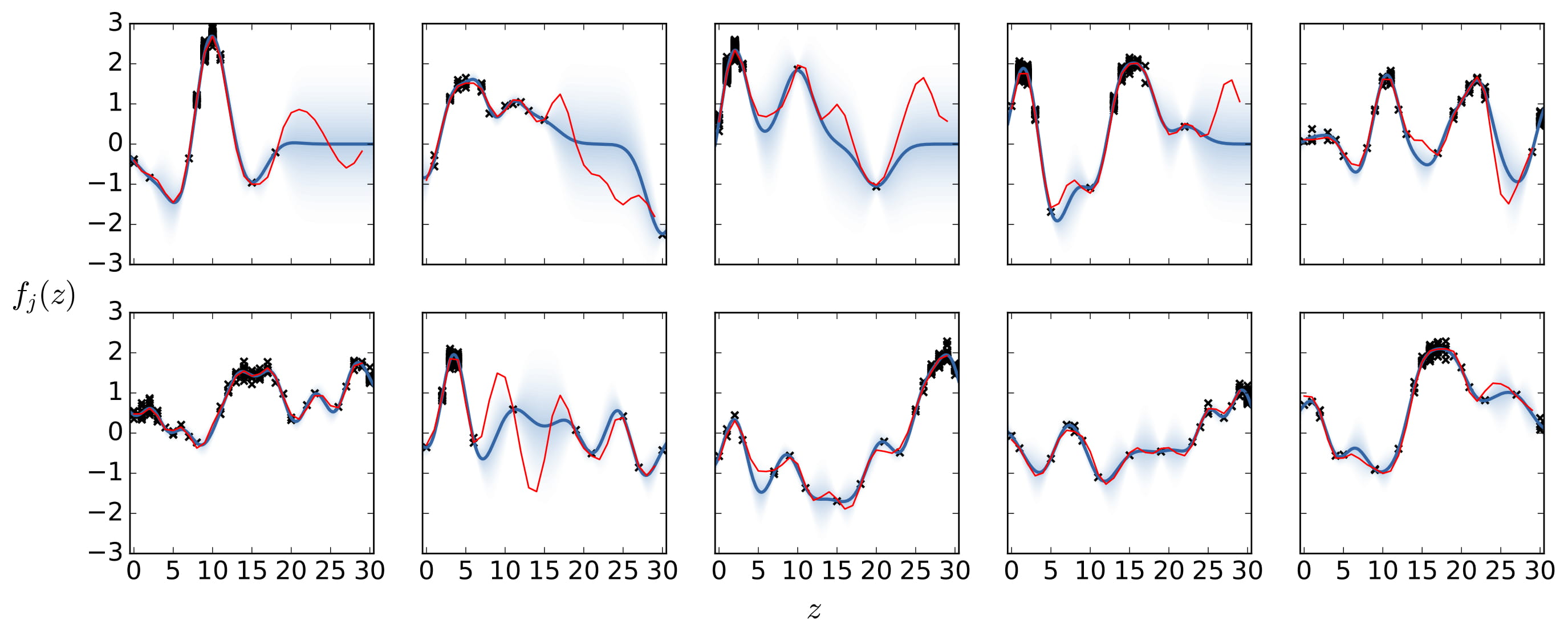}
		 \caption{$d=3$}
	\end{subfigure}
      \caption{$d$RGP-TS for squared exponential kernel with $l=2$}
      \label{fig:gptspost_l2}
\end{figure}

\begin{figure}[h]
	\centering
	\begin{subfigure}{0.8\textwidth}
		  \centering
		 \includegraphics[width=0.9\textwidth]{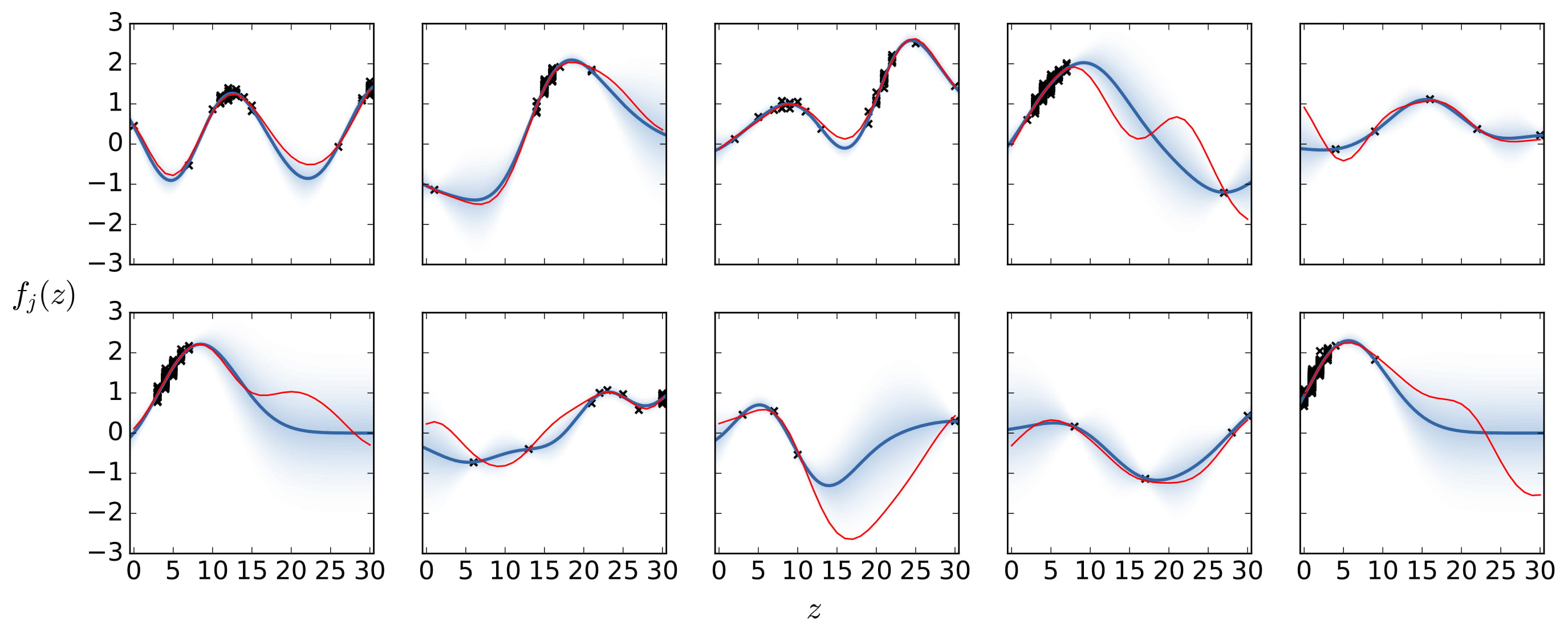}
		 \caption{$d=1$}
	\end{subfigure}
	\par\bigskip
	\begin{subfigure}{0.8\textwidth}
		  \centering
		 \includegraphics[width=0.9\textwidth]{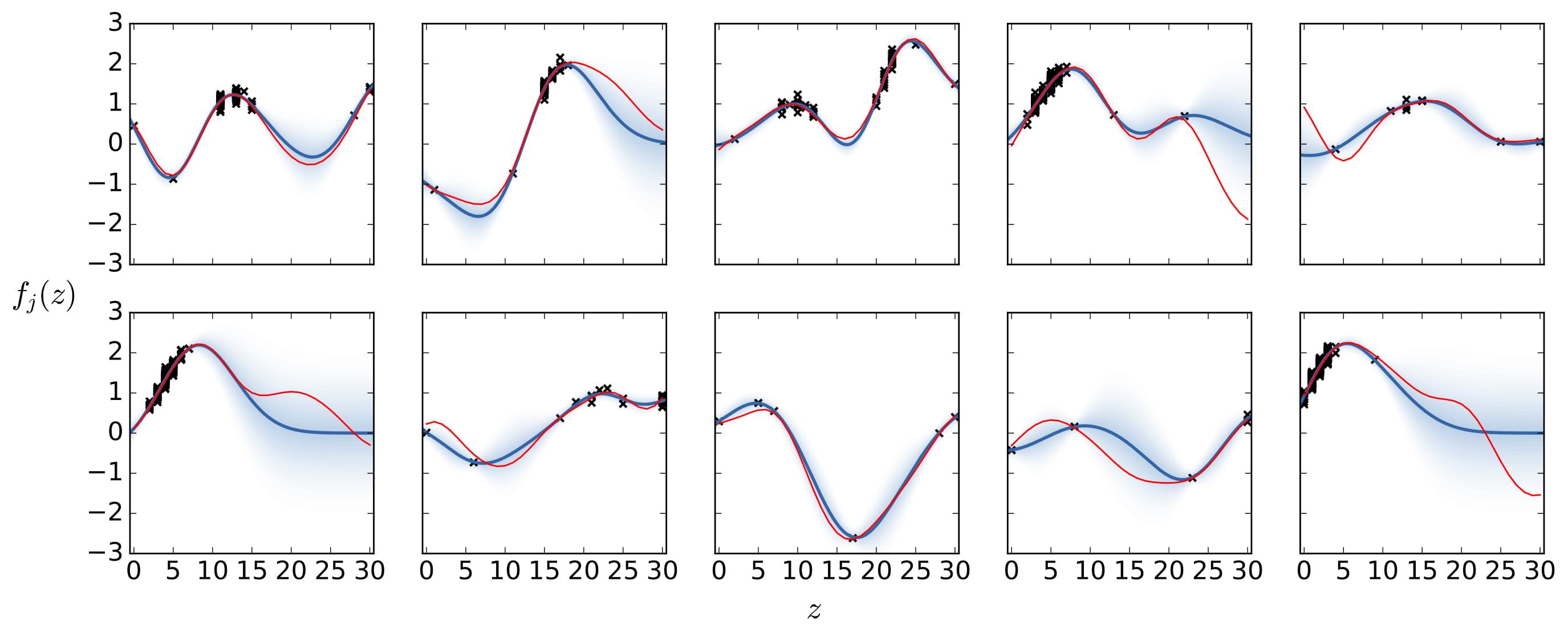}
		 \caption{$d=2$}
	\end{subfigure}
	\par\bigskip
	\begin{subfigure}{0.8\textwidth}
		  \centering
		 \includegraphics[width=0.9\textwidth]{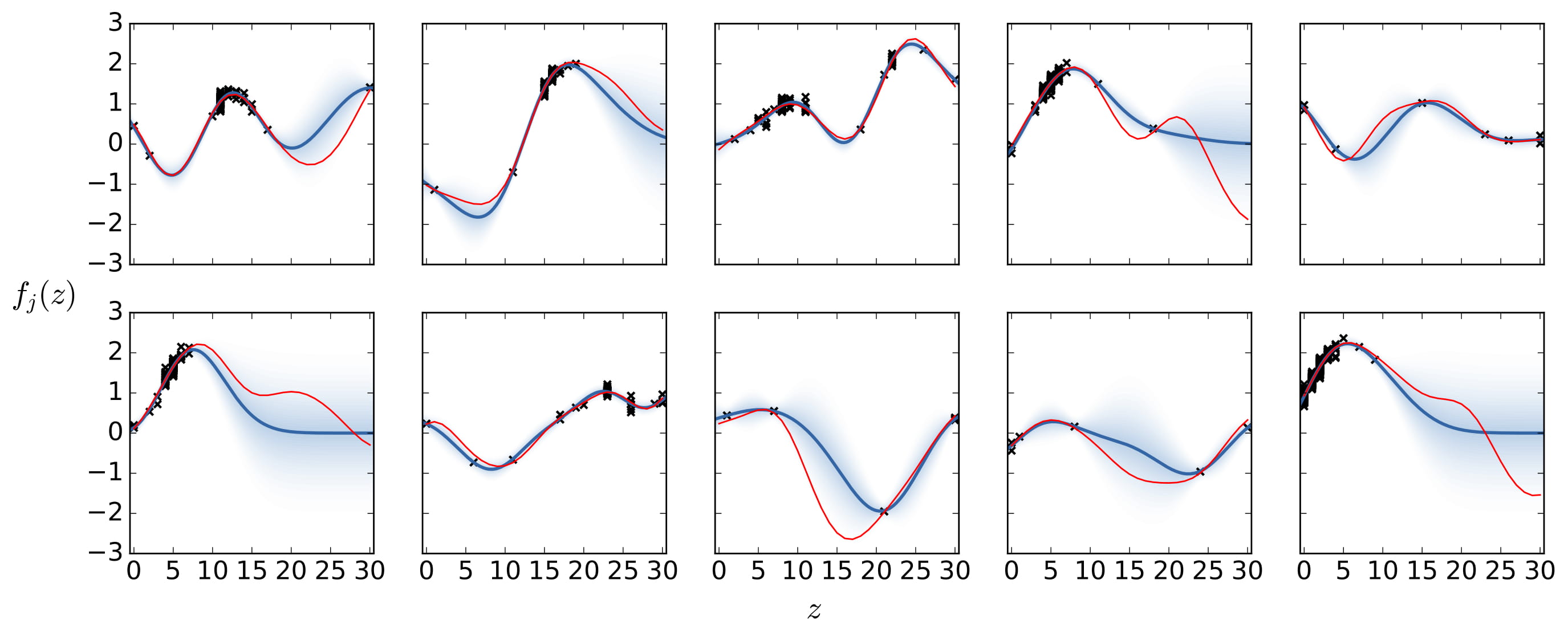}
		 \caption{$d=3$}
	\end{subfigure}
      \caption{$d$RGP-TS wit squared exponential kernel with $l=5$}
      \label{fig:gptspost_l5}
\end{figure}
\FloatBarrier

\subsection{Implementation of RogueUCB-Tuned} \label{app:rogueucb}
We briefly discuss the steps that were taken to map the recovering bandits problem into the setup of \cite{mintz2017nonstationary}. For this, we need to encode the recovery dynamics into a state dynamics function used by \cite{mintz2017nonstationary}. This can trivially be done by defining the functions $h: (\cZ, \cA) \to \cZ$ as $h(\bz, j) = \max\{\bz + {\bf 1}, z_{\max}\} - \max \{z_j+1,z_{\max}\}{\bf e}_j$, where ${\bf 1}$ is the vector of ones, ${\bf e}_j$ is the standard basis vector consisting of all zeros and a 1 in position $j$, and the maximum is taken component wise. As in \cite{mintz2017nonstationary}, we did not implement the RogueUCB algorithm, but rather the empirical version, RogueUCB-Tuned, for which there are no theoretical guarantees. When implementing this, we set the parameter $\eta$ to be the maximal value of the KL-divergence, as in \cite{mintz2017nonstationary}.

\subsection{Values of Theta used in Parametric Experiments} \label{app:thetas}
Here we give the values of $\theta$ (to 3dp) which were used in the logistic and gamma experiments in Section~\ref{sec:exp}. These were sampled uniformly. Note that this sampling had no influence over our choice of kernel.

\subsubsection{Logistic}
\begin{table}[h]
\caption{$\theta$ values used in experiments with logistic recovery functions}
\label{tab:thetalog}
\begin{center}
\begin{tabular}{c | c c c}
\toprule
 & \multicolumn{3}{c}{$\theta$} \\
\midrule
Arm 1 & 0.584 &  0.521 &  12.239 \\
Arm 2 & 0.971 &  0.357 & 10.460 \\
Arm 3 & 0.121 & 0.622 &  25.631 \\
Arm 4 & 0.240 & 0.943 & 18.870 \\
Arm 5 & 0.613 & 0.925 & 20.310 \\
Arm 6 & 0.480 & 0.914 & 1.452 \\
Arm 7 & 0.974 & 0.484 & 10.128 \\
Arm 8 & 0.780 & 0.422 & 0.396 \\
Arm 9 & 0.658 & 0.591 & 23.264 \\
Arm 10 & 0.687 & 0.753 & 7.908 \\
 \bottomrule
\end{tabular}
\end{center}
\end{table}
\FloatBarrier

\subsubsection{Gamma}
\begin{table}[h]
\caption{$\theta$ values used in experiments with gamma recovery functions}
\label{tab:thetagam}
\begin{center}
\begin{tabular}{c | c c c}
\toprule
 & \multicolumn{3}{c}{$\theta$} \\
\midrule
Arm 1 & 2.068 & 0.249 & 0.508 \\
Arm 2 & 5.023 & 0.375 & 0.551 \\
Arm 3 & 3.657 & 0.470 & 0.772 \\
Arm 4 & 0.560 & 0.176 & 0.569 \\
Arm 5 & 3.901 & 0.747 & 0.500 \\
Arm 6 & 0.600 & 0.145 & 0.266 \\
Arm 7 & 6.482 & 0.522 & 0.554 \\
Arm 8 & 13.645 & 0.748 & 0.678 \\
Arm 9 & 7.365 & 0.562 & 0.288 \\
Arm 10 & 2.705 & 0.593 & 0.381 \\
 \bottomrule
\end{tabular}
\end{center}
\end{table}
\FloatBarrier

\subsection{Results for Different Lengthscales}
In this section, we present results for the parametric setting where we have used different lenghtscales for the kernel of the Gaussian process in our methods. The parametric functions that we are considering are quite smooth so we choose a squared exponential kernel and used $l=5$ in the main text, and present results here for $l=2.5$ and $l=7.5$. Note that in this setting looking at the smoothness of the recovery functions to inform a decision about the lengthscale is reasonable since we are comparing our algorithms to RogueUCB-Tuned of \cite{mintz2017nonstationary} which requires knowledge of the parametric family and Lipschitz constant of the recovery function.

The results for $l=2.5$ are shown in Table~\ref{tab:rew_loggammal25} and Figure~\ref{fig:regloggam25}. The results for $l=7.5$ are in Table~\ref{tab:rew_loggammal75} and Figure~\ref{fig:regloggam75}. From these results, we can see that in the Gamma case, our algorithms are almost invariant to the choice of $l$, obtaining similar results for all choices of $l$. In particular, for all three choices of $l$ considered, our algorithms considerably outperform RogueUCB-Tuned of \cite{mintz2017nonstationary}. In the logistic setting, there is slightly more variation in the performance of our algorithms when the lengthscale changes, although the results are still fairly similar. In this case, we see that choosing $l=7.5$ leads to the best results for both of our algorithms. This is most likely due to the fact that logistic functions are quite smooth and $l=7.5$ represents the smoothest GPs we have considered.

\begin{table*}
\caption{Total reward at $T=1000$ for $l=2.5$}
\label{tab:rew_loggammal25}
\begin{center}
\begin{tabular}{c c c c c}
\toprule
Setting & 1RGP-UCB ($l=2.5$) & 1RGP-TS ($l=2.5$) & RogueUCB-Tuned & UCB-Z \\
\midrule
Logistic & 448.6  (441.1,456.6) &  452.5  (443.7,460.3)  & 446.2  (438.2,453.5) & 242.6 (229.6,256.0) \\ 
\rule{0pt}{2ex}    Gamma & 145.1  (138.5, 151.5)  & 155.8  (148.8,162.5) & 132.7 (111.0,144.5) & 116.8 (108.4,125.5)   \\ 
 \bottomrule
\end{tabular}
\end{center}
\end{table*}

\begin{table*}
\caption{Total reward at $T=1000$ for $l=7.5$}
\label{tab:rew_loggammal75}
\begin{center}
\begin{tabular}{c c c c c}
\toprule
Setting & 1RGP-UCB ($l=7.5$) & 1RGP-TS ($l=7.5$) & RogueUCB-Tuned & UCB-Z \\
\midrule
Logistic & 465.1  (457.3,472.9) &  465.1 (457.4,472.7)  & 446.2  (438.2,453.5) & 242.6 (229.6,256.0) \\ 
\rule{0pt}{2ex}    Gamma & 145.2  (139.8, 151.0)  & 155.8  (149.0,162.5) & 132.7 (111.0,144.5) & 116.8 (108.4,125.5)   \\ 
 \bottomrule
\end{tabular}
\end{center}
\end{table*}

\begin{figure}[H]
    \centering
    \begin{subfigure}{0.45\columnwidth}
    	\centering
        	\captionsetup{width=.9\linewidth}%
        \includegraphics[width=0.8\textwidth]{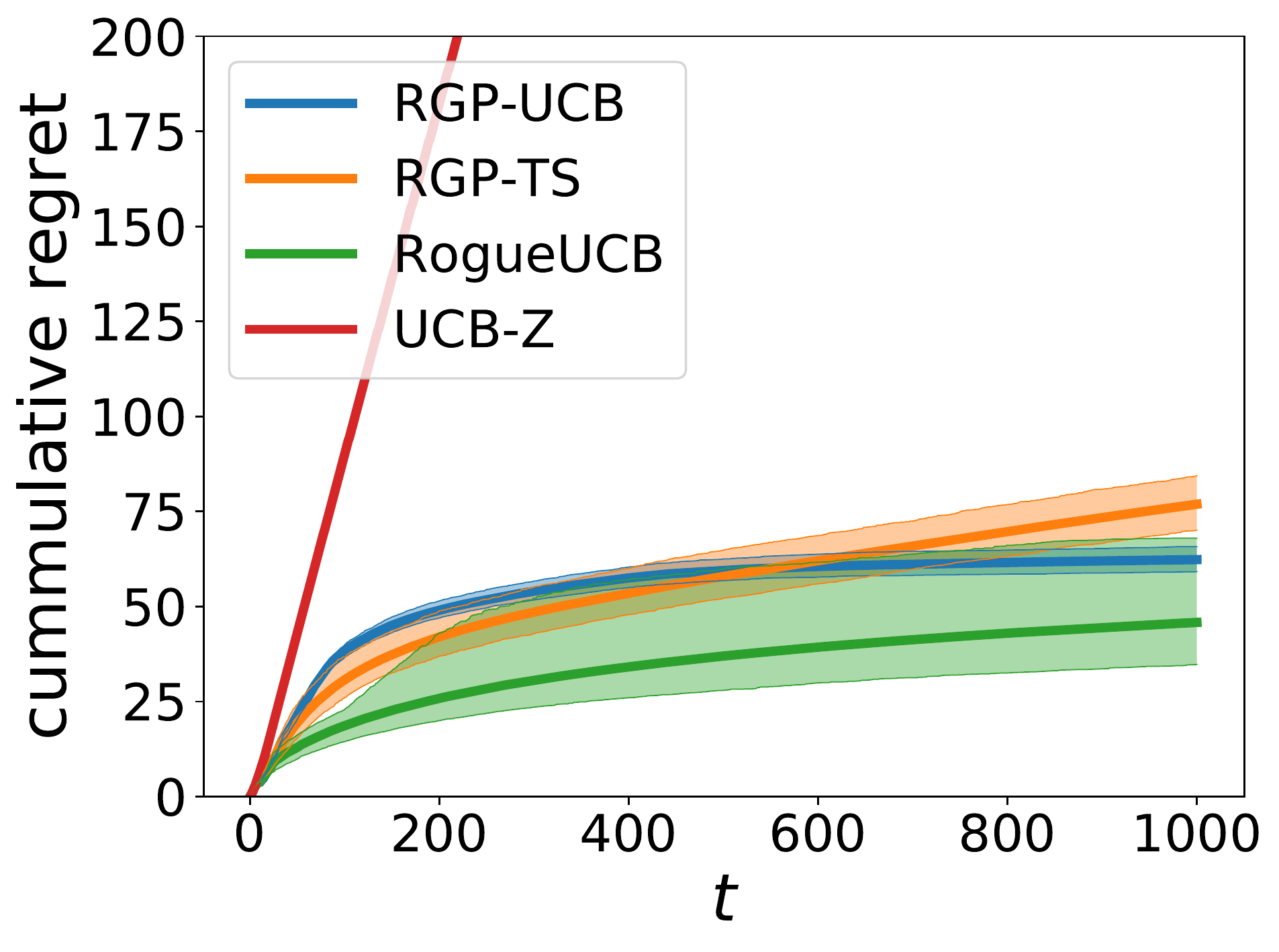}
        \caption{Logistic setup, $l=2.5$}
        \label{fig:reg_logistic}
    \end{subfigure}
    \begin{subfigure}{0.45\columnwidth}
    	\centering
        	\captionsetup{width=.9\linewidth}%
       	\includegraphics[width=0.8\textwidth]{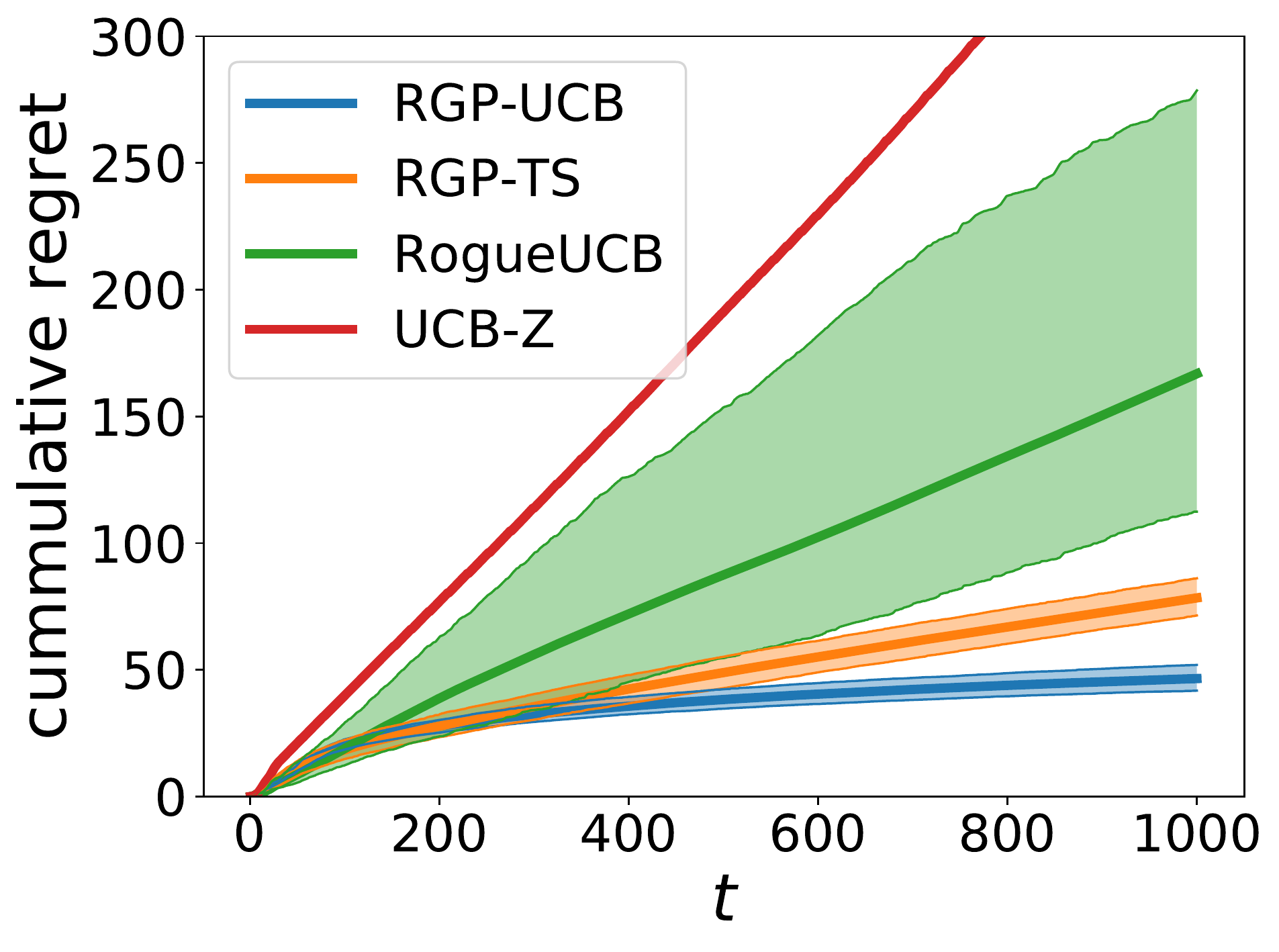}
        \caption{Gamma setup, $l=2.5$}
        \label{fig:reg_gamma}
    \end{subfigure}
    \caption{Cumulative instantaneous regret $l=2.5$} \label{fig:regloggam25}
\end{figure}

\begin{figure}[H]
    \centering
    \begin{subfigure}{0.45\columnwidth}
    	\centering
        	\captionsetup{width=.9\linewidth}%
        \includegraphics[width=0.8\textwidth]{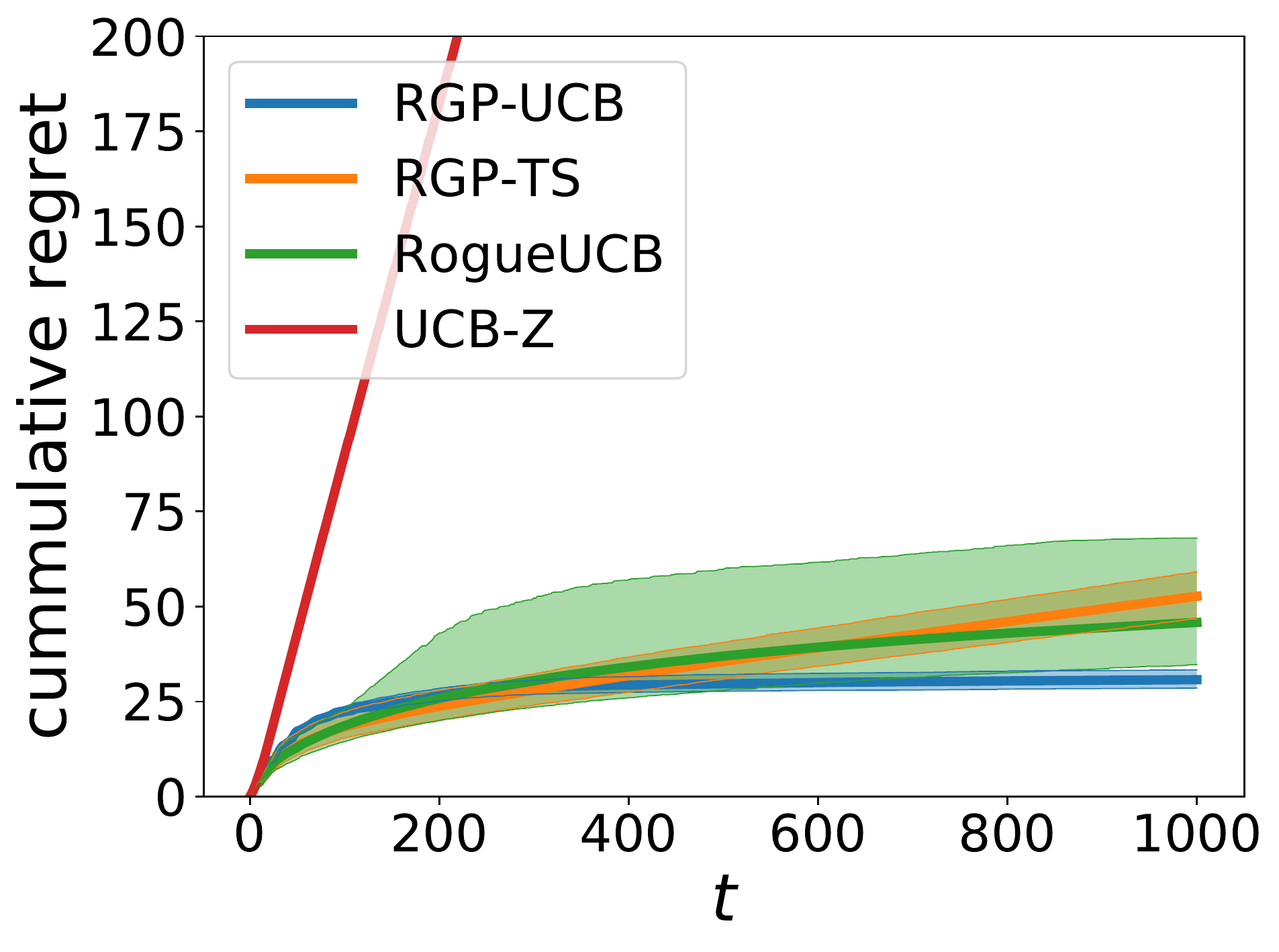}
        \caption{Logistic setup, $l=7.5$}
        \label{fig:reg_logistic}
    \end{subfigure}
    \begin{subfigure}{0.45\columnwidth}
    	\centering
        	\captionsetup{width=.9\linewidth}%
       	\includegraphics[width=0.8\textwidth]{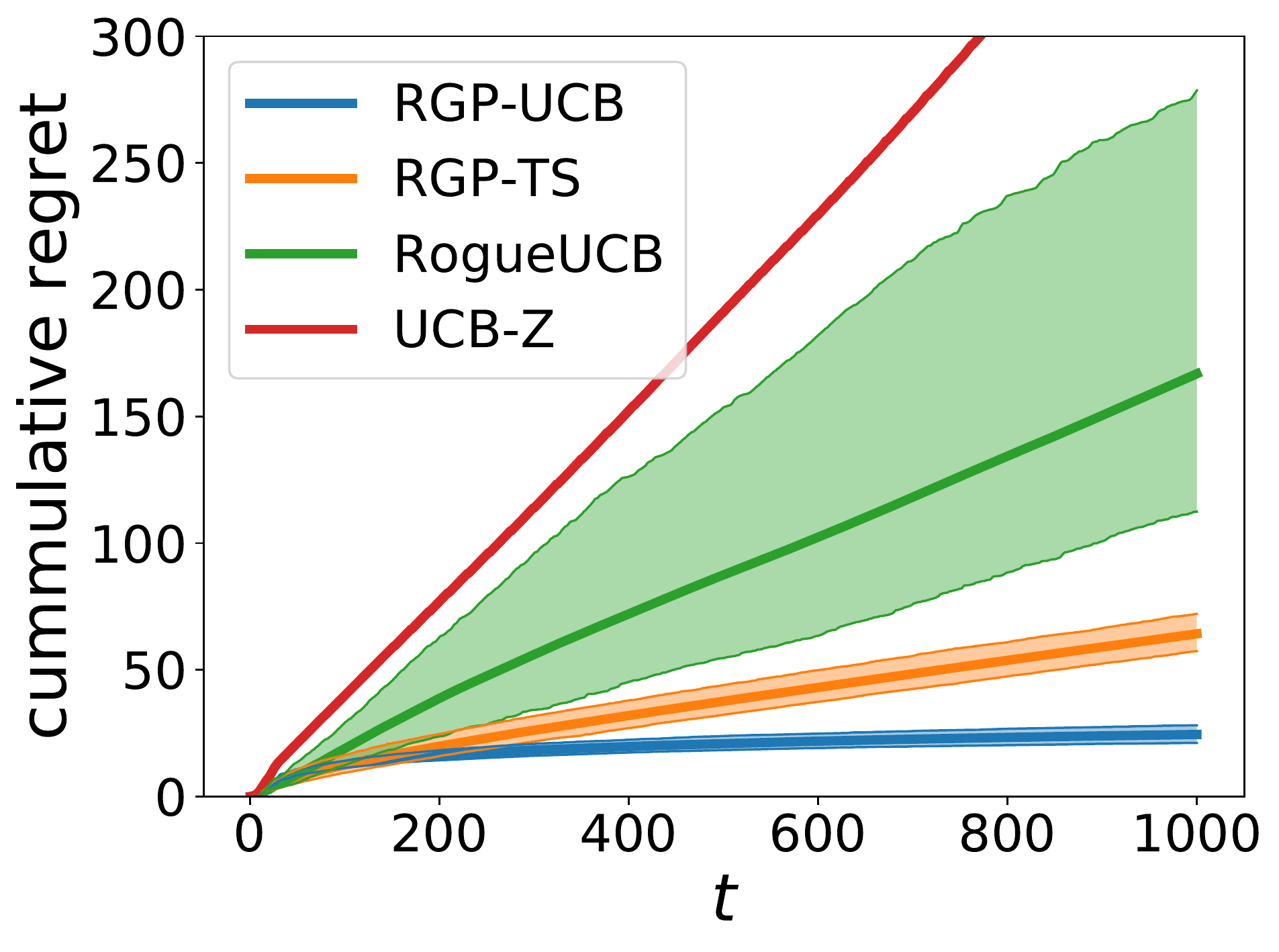}
        \caption{Gamma setup, $l=7.5$}
        \label{fig:reg_gamma}
    \end{subfigure}
    \caption{Cumulative instantaneous regret $l=7.5$} \label{fig:regloggam75}
\end{figure}
\FloatBarrier

\end{document}